%% file: paper.tex
\documentclass{article}
\usepackage[numbers]{natbib}
\pdfoutput=1
\usepackage{amsmath,amsfonts,amssymb,amsthm, color}
\usepackage{algorithm,algorithmic}
\usepackage{fullpage}

\include{header}

\begin{document}

\setlength{\parskip}{2mm}
\setlength{\parindent}{0pt}


\newcommand{\mbb}[1]{\mathbb{#1}}
\newcommand{\mbf}[1]{\mathbf{#1}}
\newcommand{\mc}[1]{\mathcal{#1}}
\newcommand{\mrm}[1]{\mathrm{#1}}
\newcommand{\trm}[1]{\textrm{#1}}

\newcommand{\sign}{\mrm{sign}}
\newcommand{\argmin}[1]{\underset{#1}{\mrm{argmin}} \ }
\newcommand{\argmax}[1]{\underset{#1}{\mrm{argmax}} \ }
\newcommand{\reals}{\mathbb{R}}
\newcommand{\E}[1]{\mathbb{E}\left[ #1 \right]} 
\newcommand{\Ebr}[1]{\mathbb{E}\left\{ #1 \right\}} 
\newcommand{\En}{\mathbb{E}}  
\newcommand{\Eu}[1]{\underset{#1}{\mathbb{E}}}  
\newcommand{\Ebar}{\Hat{\Hat{\mathbb{E}}}}  
\newcommand{\Esbar}[2]{\Hat{\Hat{\mathbb{E}}}_{#1}\left[ #2 \right]} 
\newcommand{\Es}[2]{\mathbb{E}_{#1}\left[ #2 \right]} 
\newcommand{\Ps}[2]{\mathbb{P}_{#1}\left[ #2 \right]}
\newcommand{\Prob}{\mathbb{P}}
\newcommand{\conv}{\operatorname{conv}}
\newcommand{\inner}[1]{\left\langle #1 \right\rangle}
\newcommand{\lv}{\left\|}
\newcommand{\rv}{\right\|}
\newcommand{\Phifunc}[1]{\Phi\left(#1\right)}
\newcommand{\ind}[1]{{\bf 1}\left\{#1\right\}}
\newcommand{\tr}{\ensuremath{{\scriptscriptstyle\mathsf{T}}}}
\newcommand{\eqdist}{\stackrel{\text{d}}{=}}
\newcommand{\alphT}{\widehat{\alpha}(T)}
\newcommand{\PD}{\mathcal P}
\newcommand{\QD}{\mathcal Q}
\newcommand{\jp}{\ensuremath{\mathbf{p}}}
\newcommand{\rh}{\boldsymbol{\rho}}
\newcommand{\proj}{\text{Proj}}
\newcommand{\Eunderone}[1]{\underset{#1}{\En}}
\newcommand{\Eunder}[2]{\underset{\underset{#1}{#2}}{\En}}
\newcommand{\bphi}{\boldsymbol\phi}
\newcommand\s{\mathbf{s}}
\newcommand\w{\mathbf{w}}
\newcommand\x{\mathbf{x}}
\newcommand\y{\mathbf{y}}
\newcommand\z{\mathbf{z}}
\newcommand\f{\mathbf{f}}

\renewcommand\v{\mathbf{v}}

\newcommand\cB{\mathcal{B}}
\newcommand\cC{\mathcal{C}}
\newcommand\cD{\mathcal{D}}
\newcommand\cL{\mathcal{L}}
\newcommand\cN{\mathcal{N}}
\newcommand\X{\mathcal{X}}
\newcommand\Y{\mathcal{Y}}
\newcommand\Z{\mathcal{Z}}
\newcommand\F{\mathcal{F}}
\newcommand\G{\mathcal{G}}
\newcommand\cH{\mathcal{H}}
\newcommand\N{\mathcal{N}}
\newcommand\M{\mathcal{M}}
\newcommand\W{\mathcal{W}}
\newcommand\Nhat{\mathcal{\widehat{N}}}
\newcommand\Diff{\mathcal{G}}
\newcommand\Compare{\boldsymbol{B}}
\newcommand\RH{\eta} 
\newcommand\metricent{\N_{\mathrm{metric}}} 

\newcommand\ldim{\mathrm{Sdim}}
\newcommand\fat{\mathrm{fat}}
\newcommand\Img{\mbox{Img}}
\newcommand\sparam{\sigma} 
\newcommand\Psimax{\ensuremath{\Psi_{\mathrm{max}}}}

\newcommand\Rad{\mathfrak{R}}
\newcommand\Val{\mathcal{V}}
\newcommand\Valdet{\mathcal{V}^{\mathrm{det}}}
\newcommand\Dudley{\mathfrak{D}}
\newcommand\Reg{\mbf{R}}
\newcommand\D{\mbf{D}}
\renewcommand\P{\mbf{P}}

\newcommand\Xcvx{\X_\mathrm{cvx}}
\newcommand\Xlin{\X_\mathrm{lin}}
\newcommand\loss{\mathrm{loss}}

\title{Online Learning: Beyond Regret }
\author{
Alexander Rakhlin \\
Department of Statistics \\
University of Pennsylvania
\and 
Karthik Sridharan \\
TTIC\\
Chicago, IL
\and 
Ambuj Tewari\\
Computer Science Department\\
University of Texas at Austin
}

\maketitle

\begin{abstract}
We study online learnability of a wide class of problems, extending the results of \cite{RakSriTew10} to general notions of performance measure well beyond external regret. Our framework simultaneously captures such well-known notions as internal and general $\Phi$-regret, learning with non-additive global cost functions, Blackwell's approachability, calibration of forecasters, adaptive regret, and more. We show that learnability in all these situations is due to control of the same three quantities: a martingale convergence term, a term describing the ability to perform well if future is known, and  a generalization of sequential Rademacher complexity, studied in \cite{RakSriTew10}. Since we directly study complexity of the problem instead of focusing on efficient algorithms, we are able to improve and extend many known results which have been previously derived via an algorithmic construction.

\end{abstract}

\input{intro}

\input{upper}

\input{lower}
\input{examples}

\input{highprob}

\input{highprob_calibration}

\section*{Acknowledgements}
We thank Dean Foster for many insightful discussions about calibration and Blackwell's approachability. A. Rakhlin gratefully acknowledges the support of NSF under grant CAREER DMS-0954737 and Dean's Research Fund. 

\bibliographystyle{plain}
\bibliography{paper}

\appendix

\input{appendix}

\input{pinelis}

\input{generaltriplex}

\end{document}

%% file: header.tex
\newcommand{\lemref}[1]{Lemma~\ref{#1}}

\newcommand{\assumpref}[1]{Assumption~\ref{#1}}

\newtheorem{theorem}{Theorem}
\newtheorem{lemma}[theorem]{Lemma}

\newtheorem{corollary}[theorem]{Corollary}
\newtheorem{assumption}{Assumption}
\newtheorem{proposition}[theorem]{Proposition}
\newtheorem{example}{Example}
\newtheorem{remark}{Remark}
\theoremstyle{definition}
\newtheorem{definition}{Definition}

%% file: intro.tex
\section{Introduction}
\label{sec:intro}

In the companion paper \cite{RakSriTew10}, we analyzed learnability in the {\tt Online Learning Model} when the value of the game is defined through minimax {\em regret}. However, regret (also known as {\em external regret}) is not the only way to measure performance of an online learning procedure. In the present paper, we extend the results of \cite{RakSriTew10} to other performance measures, encompassing a wide spectrum of notions which appear in the literature.  Our framework gives the same footing to external regret, internal and general $\Phi$-regret, learning with non-additive global cost functions, Blackwell's approachability, calibration of forecasters, adaptive regret, and more. We recover, extend, and improve some existing results, and (what is more important)  show that they all follow from control of the same quantities. In particular, sequential Rademacher complexity, introduced in \cite{RakSriTew10}, plays a key role in these derivations. 

A reflection on the past two decades of research in learning theory reveals (in our somewhat biased view) an interesting difference between Statistical Learning Theory and Online Learning. In the former, the focus has been primarily on understanding \emph{complexity measures} rather than \emph{algorithms}. There are good reasons for this: if a supervised problem with i.i.d. data is learnable, Empirical Risk Minimization is the algorithm that will perform well if one disregards computational aspects. In contrast, Online Learning has been mainly centered around algorithms. Given an algorithm, a non-trivial bound serves as a certificate that the problem is learnable. This algorithm-focused approach has dominated research in Online Learning for several decades. Many important tools (such as optimization-based algorithms for online convex optimization) have emerged, yet the results lacked a unified approach for determining learnability.

With the tools developed in \cite{RakSriTew10}, the question of learnability can now be addressed in a variety of situations in a unified manner. In fact, \cite{RakSriTew10} presents a number of examples of provably learnable problems for which computationally feasible online learning methods have not yet been developed. In the present paper, we show that the scope of problems whose learnability and precise rates can be characterized is much larger than those defined in \cite{RakSriTew10} through external regret. Within this circle of problems are such well-known results as Blackwell's approachability and calibration of forecasters. For instance, our complexity-based (rather than algorithm-based) approach yields a proof of Blackwell's approachability in Banach spaces without ever mentioning an algorithm. Let us remark that Blackwell's approachability has been a key tool for showing learnability \cite{PLG}; as our results imply  approachability, they can be utilized whenever Blackwell's approachability has been successful. The results can also be used in situations where phrasing a problem as an approachability question is not necessarily natural. In Section~\ref{sec:blackwell}, we discuss the relation of our results to approachability in greater detail.

Our contributions can be broken down into three parts. 
\begin{itemize}
	\item The first contribution lies in the formulation of the online learning problem, with a performance measure (a form of \emph{regret}), defined in terms of certain payoff transformation mappings. While this formulation might appear unusual, we show that it is general enough to encompass many seemingly different frameworks (games), yet specific enough that we can provide generic upper bounds. 
	\item The second contribution is in developing upper and lower bounds on the value of the game under various natural assumptions. These tools allow us to deal with performance measures well beyond the standard notion of external regret. Such performance measures include smooth non-additive functions of payoffs, generalizing the ``cumulative payoff'' notion often considered in the literature. The abstract definition in terms of payoff transformations lets us consider rich classes of mappings whose complexity can be studied through random averages, covering numbers, and combinatorial parameters.
	\item We apply our machinery to a number of well-known problems. (a) First, for the usual notion of external regret, the results boil down to those of \cite{RakSriTew10}. (b) For the more general $\Phi$-regret (see e.g.  \cite{StoLug07,GorGreMar08, HazKal07}), we recover and improve several known results. In particular, for convergence to $\Phi$-correlated equilibria, we improve upon the results of Stoltz and Lugosi \cite{StoLug07}.  (c)  We study the game of Blackwell's approachability \cite{Blackwell56} in (possibly infinite-dimensional) separable Banach spaces. Specifically, we show that martingale convergence in these spaces (along with Blackwell's one-shot approachability condition) is both necessary and sufficient for Blackwell's approachability to hold. (d) We also consider the game of calibrated forecasting. We improve upon the results of Mannor and Stoltz \cite{ManSto09} and prove (to the best of our knowledge) the first known $O(T^{-1/2})$ rates for calibration with more than $2$ outcomes. Our approach is markedly different from those found in the literature. (e) We use our framework to study games with global cost functions and as an example we extend the bounds recently obtained by Even-Dar et al \cite{EveKleManMan09}. (f) We provide techniques for bounding notions of regret where algorithm's performance is measured against a time-varying comparator (see e.g. \cite{HerWar98,BouWar02, Zinkevich03}). Such notions of regret are better suited for reactive environments. Using the general tools we developed, we not only recover the results in \cite{HerWar98,BouWar02} but also extend them to prove learnability and obtain rates for much more general settings. Our last example shows that adaptive regret notion of Hazan and Seshadhri \cite{HazanSe09} can be defined in greater generality while still preserving learnability.
\end{itemize}

The intent of this paper is to provide a framework and tools for studying problems that can be phrased as repeated games. However, unlike much of existing research in online learning, we are not solving the general problem by exhibiting an algorithm and studying its performance. Rather, we proceed by directly attacking the value of the game. Alas, the value is a complicated object, and the non-invitingly long sequence of infima and suprema can single-handedly extinguish any desire to study it. Our results attest to the power of \emph{symmetrization}, which emerges as a key tool for studying the value of the game. In the literature, symmetrization has been used  for i.i.d. data \cite{GinZin84}. In \cite{RakSriTew10, AbeAgaBarRak09}, it was shown that symmetrization can also be used in situations beyond the traditional setting. What is even more surprising, we are able to employ symmetrization ideas even when the objective function is not a summation of terms but rather a global function of many variables. We hope that these tools can have an impact not only on online learning but also on game theory. 

We believe that there are many more examples falling under the present framework. We only chose a few to demonstrate how upper and lower bounds arise from the complexity of the problem. Along with an upper bound, a (computationally inefficient) algorithm can always be recovered from the minimax analysis. Finding efficient algorithms is often a difficult enterprise, and it is important to be able to understand the inherent complexity even before focusing on computation.

Let us spend a minute describing the organization of this paper. Since our results are meant to serve as a unifying framework, we faced the question of whether to build up the level of generality as we progress through the paper, or whether to start with the most general results and then make them more specific. We decided to do the latter. While we find this flow of general-to-specific more natural, we risk losing potential readers on the first few pages. In hopes of avoiding this, after defining the online learning problem in full generality in Section~\ref{sec:setting}, we briefly state how various well-known frameworks appear as particular instances. Then, in Section~\ref{sec:upper}, learnability is established under various very general assumptions. Next, in Section~\ref{sec:lower}, techniques for proving lower bounds are shown. Various examples and frameworks are considered in more detail in Section~\ref{sec:examples}. In Section~\ref{sec:highprob}, the ``in-probability'' analogues are derived. Hannan consistency is established via almost sure convergence. For an overview of the results without the painful details, one may read Section~\ref{sec:setting} and then skip to Section~\ref{sec:examples}. For the sake of readability, most of the proofs are deferred to the appendix.  Let us remark that \cite{RakSriTew10} is not required for reading this paper. In a few places, however, if a proof is basically the same as in \cite{RakSriTew10} except for notation, we will omit the proof.

\section{The Setting}
\label{sec:setting}

At a very abstract level, the problem of online learning can be phrased as that of optimization of a given  function $\Reg_T(f_1,x_1,\ldots, f_T, x_T)$ with coordinates being chosen \emph{sequentially} by the player and the adversary. Of course, at this level of generality not much can be said. Hence, we make some minimal assumptions on the function $\Reg_T$ which lead to meaningful guarantees on the online optimization process.\footnote{The question of general conditions on the function under which such sequential minimization is possible was put forth by Peter Bartlett a few years ago in a coffee conversation. This paper paves way towards addressing this question.} These assumptions are satisfied by a number of natural performance measures, as illustrated by the examples below.

Let $\F$ and $\X$ be the sets of moves of the learner (player) and the adversary, respectively. Generalizing the {\tt Online Learning Model} considered in \cite{RakSriTew10}, we study the following $T$-round interaction between the learner and the adversary: 
\begin{itemize}
	\addtolength{\itemsep}{-0.6\baselineskip}
	\item[]\hspace{-9mm} On round $t = 1,\ldots, T$, 
	\item the learner chooses a mixed strategy $q_t$ (distribution on $\F$)
	\item the adversary picks $x_t \in \X$ 
	\item the learner draws $f_t\in\F$ from $q_t$ and receives payoff (loss) signal $\ell(f_t,x_t) \in \cH$ 
	\item[]\hspace{-9mm} End
\end{itemize}

We would like to specify that we are in the full information setting and that at the end of each round both the player and the adversary observe each other's moves $f_t,x_t$. The payoff space $\cH$ is a (not necessarily convex) subset of a separable Banach space $\mathcal{B}$. Both the player and the adversary can be randomized and adaptive.

The goal of the learner is to minimize the following general form of performance measure:
\begin{align}
	\label{eq:big_fat_definition}
	\Reg_T =  \Compare(\ell(f_1,x_1), \ldots, \ell(f_T, x_T)) - \inf_{\bphi\in \Phi_T} \Compare(\ell_{\phi_1}(f_1,x_1), \ldots, \ell_{\phi_T}(f_T, x_T)) \ ,
\end{align}
where
\begin{itemize}
	\item The function $\ell:\F\times\X \mapsto \cH$ is an $\cH$-valued payoff (or loss) function.
	\item The function $\Compare:\cH^T \mapsto \reals$ is a (not necessarily additive or convex) form of cumulative payoff.
	\item The set $\Phi_T$ consists of sequences $\bphi=(\phi_1,\ldots, \phi_T)$ of measurable payoff transformation mappings $\phi_t: \cH^{\F\times \X}\mapsto \cH^{\F\times \X}$ that transform the payoff function $\ell$ into a payoff function $\ell_{\phi_t}$. 
\end{itemize}
The goal of the adversary is to maximize the same quantity \eqref{eq:big_fat_definition}, making it a zero-sum game. 

This paper is concerned with learnability and with identifying \emph{complexity measures} that govern learnability. But complexity of what should we focus on? After all, the general online learning problem is defined by the choice of five components: $\Compare, \ell, \F, \X$, and $\Phi_T$. In \cite{RakSriTew10}, the choice was easy: it should be the complexity of the function class $\F$ that plays the key role. That was natural because the payoff was written as $\ell(f,x) = f(x)$, which suggested that the function class $\F$ is the object of study. The present formulation, however, is much more general. When this work commenced, it seemed likely that complexity of the problem will be some interaction between the complexity of $\Phi_T$ and complexity of $\F$. As we show below, one may just focus on the complexity of $\Phi_T$, while $\F$ and $\X$ are now on the same footing. For instance, even if it might seem unusual at first, we will introduce a notion of a cover of the set of sequences of payoff transformations $\Phi_T$. In summary, while all five components $\Compare, \ell, \F, \X$, and $\Phi_T$ play a role in determining learnability, we will mainly refer to the complexity of the payoff mapping $\ell$ and the payoff transformation $\Phi_T$ without an explicit reference to $\F$, $\X$, and $\Compare$. We emphasize that most flexibility comes from the payoff mapping $\ell$ and from the transformations $\Phi_T$ of the payoffs. 

In particular, important classes of payoff transformation mappings are the \emph{departure mappings} that transform the payoff function $\ell$ by acting only on the first argument of $\ell$, i.e. only modifying the row (player's action) choice.

\begin{definition}
	\label{def:departure}
A class of sequences of payoff transformations $\Phi_T$ is said to be a  \emph{departure mapping class} if there exists a class $\Phi'_T$ of sequences $\bphi' = (\phi'_1,\ldots,\phi'_T)$ with $\phi'_i: \F\mapsto\F$ such that for each $\bphi \in \Phi_T$ there exists a $\bphi' \in \Phi'_T$ with the property that, for all $t \in [T]$, $f \in \F$ and $x \in \X$, the payoff transformations can be written as
$
\ell_{\phi_t}(f, x):= \ell (\phi'_t(f), x).
$
\end{definition}

For payoff transformation classes that are departure mapping classes, the transformations $\Phi_T$ can be identified in terms of a corresponding class of departure mapping from $\F$ to itself, and we shall abuse notation and use $\Phi_T$ to represent both the class of payoff transformation and the class of departure mappings from $\F$ to itself. Another class of interest are payoff transformations that do not vary with time.

\begin{definition}
	\label{def:time_inv_and_product}
We say that $\Phi_T$ is \emph{time-invariant} if all sequences of payoff transformation are constant in time: $\Phi_T = \{(\phi,\ldots,\phi):\phi\in\Phi\}$, where $\Phi$ is a ``basis'' class of mappings $\cH^{\F\times \X}\mapsto \cH^{\F\times \X}$.
\end{definition}

In the following, we assume that $\F$ and $\X$ are subsets of a separable metric space. Let $\QD$ and $\PD$ be the sets of probability distributions on $\F$ and $\X$, respectively. Assume that $\QD$ and $\PD$ are weakly compact. From the outset, we assume that the adversary is non-oblivious (that is, adaptive).  Formally, define a learner's strategy $\pi$ as a sequence of mappings 
$
\pi_t : (\PD\times\F\times\X)^{t-1} \mapsto \QD
$
for each $t \in [T]$.  The form \eqref{eq:big_fat_definition} of the performance measure gives rise to the value of the game:
\begin{align}  
	\label{eq:def_val_game}
	\Val_T(\ell, \Phi_T) &= \inf_{q_1} \sup_{x_1} \Eunderone{f_1\sim q_1} \ldots \inf_{q_T} \sup_{x_T} \Eunderone{f_T\sim q_T}  
\sup_{\bphi\in \Phi_T}\left\{ \Compare(\ell(f_1,x_1), \ldots, \ell(f_T, x_T)) -  \Compare(\ell_{\phi_1}(  f_1,x_1), \ldots, \ell_{\phi_T}(f_T, x_T))\right\}
\end{align}
where $q_t$ and $x_t$ range over $\QD$ and $\X$, respectively. With this definition of a value, the (deterministic) strategy of the adversary is a sequence of mappings $(\QD\times\F\times\X)^{t-1}\times \QD \mapsto \X$ for each $t\in[T]$.

\begin{definition}
	\label{def:learnability}
	The problem is said to be {\em online learnable} if
	$$ \limsup_{T\to \infty} \Val_T(\ell, \Phi_T) = 0 \ .$$
\end{definition}

The value of the game is defined as an \emph{expected} performance measure. As such, it yields ``in probability'' statements. We define the value of the
game using a \emph{high probability} performance measure in Section~\ref{sec:highprob}.
We also discuss there how the high probability results lead to ``almost sure'' convergence.

\subsection{Examples}

A reader might wonder why we have defined the game in terms of abstract payoff transformation mappings. It turns out that with this definition, various seemingly different frameworks become nothing but special cases, as illustrated by the following examples. 

\begin{example}[External Regret Game]
	\label{eg:external}
	Let $\cH = \reals$ and 
	\begin{itemize}
		\item $\Compare(z_1,\ldots,z_T) = \frac{1}{T}\sum_{t=1}^T z_t$ 
		\item $\Phi_T = \{(\phi_f,\ldots,\phi_f): f\in\F ~~\mbox{ and }~~ \phi_f:\F\mapsto\F ~~\mbox{ is a constant mapping } \phi_f(g) = f ~\forall g\in\F  \}$
	\end{itemize}
	It is easy to see that Eq.~\eqref{eq:big_fat_definition} becomes
	$$ \Reg_T = \frac{1}{T}\sum_{t=1}^T \ell(f_t, x_t) - \inf_{f\in\F} \frac{1}{T}\sum_{t=1}^T \ell(f, x_t).$$
	External regret is discussed in Section~\ref{sec:external}.
\end{example}

\begin{example}[$\Phi$-Regret]
	\label{eg:phi}
	Let $\cH = \reals$ and 
	\begin{itemize}
		\item $\Compare(z_1,\ldots,z_T) = \frac{1}{T}\sum_{t=1}^T z_t$ 
		\item $\Phi_T = \{(\phi,\ldots,\phi): \phi\in\Phi \}$ for some fixed family $\Phi$ of $\F\mapsto\F$ mappings.
	\end{itemize}
	It is easy to see that Eq.~\eqref{eq:big_fat_definition} becomes
	$$ \Reg_T = \frac{1}{T}\sum_{t=1}^T \ell(f_t, x_t) - \inf_{\phi\in\Phi} \frac{1}{T}\sum_{t=1}^T \ell(\phi(f_t), x_t).$$
	This example covers a variety of notions such as external, internal, and swap regrets (see  Section~\ref{sec:phiregret}).
\end{example}

\begin{example}[Blackwell's Approachability]
	Let $\cH$ a subset of a Banach space $\mathcal{B}$, $S\subset \mathcal{B}$ be a closed convex set, and 
	\begin{itemize}
		\item $\Compare(z_1,\ldots,z_T) = \inf_{c \in S}\left\|\frac{1}{T}\sum_{t=1}^T z_t - c\right\|$ 
		\item $\Phi_T$ contains sequences $(\phi_1, \ldots,\phi_T)$ such that $\ell_{\phi_t}(f,x) = c_t\in S$ for all $f\in\F$, $x\in\X$, and $1\leq t\leq T$.
	\end{itemize}
	It is easy to see that Eq.~\eqref{eq:big_fat_definition} becomes
	$$ \Reg_T = \inf_{c\in S}\left\| \frac{1}{T}\sum_{t=1}^T \ell(f_t, x_t) - c \right\|,$$
	the distance to the set $S$. Indeed, our definition of $\Phi_T$ ensures that the comparator term is zero. Blackwell's approachability is discussed in Section~\ref{sec:blackwell}.
\end{example}

\begin{example}[Calibration of Forecasters]
	\label{eg:calibration2}
	Let $\cH=\reals^k$, $\F = \Delta(k)$ (the $k$-dimensional probability simplex) and $\X$ the set of standard unit vectors in $\reals^k$ (vertices of $\Delta(k)$). Define $\ell(f,x) = 0$.
	Further,
	\begin{itemize}
		\item $\Compare(z_1,\ldots,z_T) = -\left\|\frac{1}{T}\sum_{t=1}^T z_t \right\|$ for some norm $\|\cdot\|$ on $\reals^k$
		\item $\Phi_T=\{(\phi_{p,\lambda}, \ldots,\phi_{p,\lambda}): p\in\Delta(k), \lambda> 0\}$ contains time-invariant mappings defined by $$\ell_{\phi_{p,\lambda}}(f,x) = \ind{\|f-p\|\leq \lambda}\cdot (f-x).$$
	\end{itemize}
	It is easy to see that Eq.~\eqref{eq:big_fat_definition} becomes
	$$\Reg_T = \sup_{\lambda > 0}\sup_{p\in\Delta(k)} \left\|\frac{1}{T}\sum_{t=1}^T \ind{\|f_t-p\|\leq \lambda} \cdot (f_t-x_t) \right\|.$$
	Calibration is discussed in more detail in Section~\ref{sec:calibration}.
\end{example}

\begin{example}[Global Cost Online Learning Game \cite{EveKleManMan09}]
	\label{eg:global}
	Let $\cH = \reals^k$, $\X = [0,1]^k$, $\F=\Delta(k)$, $\ell(f, x) = f\odot x = (f^1\cdot x^1,\ldots, f^k \cdot x^k)$.
	\begin{itemize}
		\item $\Compare(z_1,\ldots,z_T) = \left\| \frac{1}{T}\sum_{t=1}^T z_t \right\|$ 
		\item $\Phi_T = \{(\phi_f,\ldots,\phi_f): f\in\F ~~\mbox{ and }~~ \phi_f:\F\mapsto\F ~~\mbox{ is a constant mapping } \phi_f(g) = f ~\forall g\in\F  \}$
	\end{itemize}
	It is easy to see that Eq.~\eqref{eq:big_fat_definition} becomes
	$$\Reg_T = \left\|\frac{1}{T}\sum_{t=1}^T f_t\odot x_t \right\| - \inf_{f\in\F} \left\|\frac{1}{T}\sum_{t=1}^T f\odot x_t \right\|.$$
	A generalization of this scenario is considered in Section~\ref{sec:global}.
\end{example}

\subsection{Notation}

Let $~\En_{x\sim p}~$ denote expectation with respect to a random variable $x$ with a distribution $p$. Note that we do not use capital letters for random variables in order to ease reading of already cumbersome equations. For a collection of random variables $x_1,\ldots, x_T$ with distributions $p_1,\ldots,p_T$, we will use the shorthand $\En_{x_{1:T}\sim p_{1:T}}$ to denote expectation with respect to all these variables. Let $q$ and $p$ be distributions on $\F$ and $\X$, respectively. We define a shorthand $\ell(q, p) = \En_{f\sim q, x\sim p} \ell(f, x)$ and $\ell_\phi(q, p)=\En_{f\sim q, x\sim p} \ell_\phi(f, x)$. The Dirac delta distribution is denoted by $\delta_x$. A Rademacher random variable $Y$ is uniformly distributed on $\{\pm 1\}$. The notation $x_{a:b}$ denotes the sequence $x_a,\ldots, x_b$. The indicator of an event $A$ is denoted by $\ind{A}$. The set $\{1,\ldots,T\}$ is denoted by $[T]$, while the $k$-dimensional probability simplex is denoted by $\Delta(k)$. The set of all functions from $\X$ to $\Y$ is denoted by $\Y^\X$, and the $t$-fold product $\X\times\ldots\times\X$ is denoted by $\X^t$. Whenever a supremum (infimum) is written in the form $\sup_{a}$ without $a$ being quantified, it is assumed that $a$ ranges over the set of all possible values which will be understood from the context. Convex hulls will be denoted by $\conv(\cdot)$.

Following \cite{RakSriTew10}, we define binary trees as follows.

\begin{definition}
Given some set ${\mathcal Z}$, a \emph{${\mathcal Z}$-valued tree of depth $T$} is a sequence $(\z_1,\ldots,\z_T)$ of $T$ mappings  $\z_i : \{\pm 1\}^{i-1} \mapsto \mathcal{Z}$. The \emph{root} of the tree $\z$ is the constant function $\z_1\in {\mathcal Z}$. 
\end{definition}

Unless specified otherwise, $\epsilon =(\epsilon_1,\ldots,\epsilon_T) \in \{\pm 1\}^T$ will define a path. Slightly abusing the notation, we will write $\z_t(\epsilon)$ instead of $\z_t(\epsilon_{1:t-1})$.

Let $\phi_\text{id}$ denote the identity payoff transformation $\ell_{\phi_\text{id}} (f,x)= \ell(f,x)$ for all $f\in\F$, $x\in\X$. Let $\mc{I} = \{(\phi_\text{id},\ldots,\phi_\text{id})\}$ be the singleton set containing the time-invariant sequence of identity transformations.

For a separable Banach space $\mathcal{B}$ equipped with a norm $\|\cdot\|$, let $B_{\|\cdot\|}$ be the unit ball. Let $\mathcal{B}^*$ denote the dual space and $B_{\|\cdot\|_*}$ the corresponding dual ball. For $a\in\mathcal{B}^*$, $\|a\|_* = \sup_{b\in B_{\|\cdot\|}} |\inner{a,b}|$.  For $b\in\mathcal{B}$, we write $\inner{a,b} = a(b)$ for the continuous linear functional $a\in\mathcal{B}^*$ on $\mathcal{B}$. A Hilbert space is dual to itself.

%% file: upper.tex
\section{General Upper Bounds}
\label{sec:upper}

This section is devoted to upper bounds on the value of the game. We start by introducing the Triplex Inequality, which requires no assumptions beyond those described in Section~\ref{sec:setting}. Under the additional weak assumption of subadditivity of $\Compare$, we can perform symmetrization and further upper bound two of the three terms in Triplex Inequality by a non-additive version of sequential Rademacher complexity \cite{RakSriTew10}. As we progress through the section, we make additional assumptions and specialize and refine the upper bounds. 

The following definition generalizes the notion of sequential Rademacher complexity, introduced in \cite{RakSriTew10}, to ``global'' functions $\Compare$ of the payoff sequence.
\begin{definition} 
	\label{def:rademacher}
	The \emph{sequential complexity} with respect to the payoff function $\ell$ and  payoff transformation mappings $\Phi_T$ is defined as
$$
\Rad_T(\ell, \Phi_T, \Compare) = \sup_{\f, \x}\ \En_{\epsilon_{1:T}} \sup_{\bphi\in \Phi_T} \Compare\Big(\epsilon_1 \ell_{\phi_1}(  \f_1(\epsilon),\x_1(\epsilon)), \ldots,  \epsilon_T \ell_{\phi_T}( \f_T(\epsilon), \x_T(\epsilon)) \Big)
$$
where the outer supremum is taken over all $(\F\times\X)$-valued trees of depth $T$ and $\epsilon=(\epsilon_1,\ldots, \epsilon_T)$ is a sequence of i.i.d. Rademacher random variables. 
\end{definition}

Whenever $\Compare$ is clear from the context, it will be omitted from the notation: $\Rad_T(\ell, \Phi_T)$. If $\Phi_T$ is a set of sequences of time-invariant transformations obtained from the base class $\Phi$, we will simply write $\Rad_T(\ell, \Phi)$.

Let us remark that the moves of the player and the adversary appear ``on the same footing'' in $\Reg_T$ and in the above definition of sequential complexity. The ``asymmetry'' of sequential Rademacher complexity \cite{RakSriTew10} (where the supremum is taken over the {\em player's} best choice) arises precisely from the asymmetry of the notion of external regret, which, in turn, is due to $\Phi_T$ acting on the player choice only. In Section~\ref{sec:external}, we show that the notion studied in \cite{RakSriTew10} is indeed  recovered for the case of external regret. 

An equivalent way to write sequential complexity is through the expanded version 
\begin{align}
	\label{eq:def_rademacher_expanded}
	\Rad_T(\ell, \Phi_T, \Compare) = \sup_{f_1,x_1}\ \En_{\epsilon_{1}} \ \sup_{f_2,x_2} \ \En_{\epsilon_2} \ldots \sup_{f_T,x_T} \ \En_{\epsilon_T} \sup_{\bphi\in \Phi_T} \Compare\Big(\epsilon_1 \ell_{\phi_1}(  f_1,x_1), \ldots,  \epsilon_T \ell_{\phi_T}( f_T, x_T) \Big)
\end{align}
where the supremum on $t$-th step is over $f_t\in\F$, $x_t\in\X$. We shall use Eq.~\eqref{eq:def_rademacher_expanded} and the more succinct Definition~\ref{def:rademacher} interchangeably.

\subsection{Triplex Inequality}

The following theorem is the main starting point for all further analysis. Because of its importance, we shall refer to it as the \emph{Triplex Inequality}. The three terms in the upper bound of the theorem can be thought of as the three key players in the process of online learning: martingale convergence, the ability to perform well if the future is known, and complexity of the class in terms of sequential complexity.
\begin{theorem}[\textbf{Triplex Inequality}]\label{thm:main}
	The following $3$-term upper bound on the value of the game holds:
		\begin{align}
			\label{eq:three_term_decomposition}
		 	& \Val_T(\ell,\Phi_T) \nonumber \\
			& ~~~\leq \sup_{p_1, q_1} \Eunder{x_1 \sim p_1}{f_1 \sim q_1} \ldots  \sup_{p_T,q_T} \Eunder{x_T \sim p_T}{f_T \sim q_T} \Big\{ \Compare(\ell(f_1,x_1), \ldots, \ell(f_T, x_T)) - \Eunder{x'_{1:T} \sim p_{1:T} }{ f'_{1:T} \sim q_{1:T}} \Compare(\ell(f'_1,x'_1), \ldots, \ell(f'_T, x'_T))  \Big\} \\
			&~~~ +\sup_{p_1} \inf_{q_1}  \ldots  \sup_{p_T} \inf_{q_T} \sup_{\bphi\in \Phi_T} \Eunder{x_{1:T} \sim p_{1:T}}{f_{1:T} \sim q_{1:T}} \Big\{ \Compare(\ell(f_1,x_1), \ldots, \ell(f_T, x_T)) - \Compare(\ell_{\phi_1}(f_1,x_1), \ldots, \ell_{\phi_T}( f_T, x_T))\Big\} \nonumber\\
			&~~~+ \sup_{p_1,q_1} \Eunder{x_1 \sim p_1}{f_1 \sim q_1} \ldots  \sup_{p_T,q_T} \Eunder{x_T \sim p_T}{f_T \sim q_T} \sup_{\bphi\in \Phi_T} 
			\left\{ \Eunder{x'_{1:T} \sim p_{1:T} }{ f'_{1:T} \sim q_{1:T}} \Compare\Big(\ell_{\phi_1}(f'_1,x'_1), \ldots, \ell_{\phi_T}(f'_T, x'_T)\Big) - \Compare\Big(\ell_{\phi_1}(f_1,x_1), \ldots, \ell_{\phi_T}(f_T,x_T) \Big) 
			\right\} \nonumber
		\end{align}
\end{theorem}

First, we remark that convexity of $\Compare$ is {\em not required} for the Triplex Inequality to hold. Under a weak subadditivity condition, the following Theorem gives upper bounds on the first and the third term.

\begin{theorem}\label{thm:rad}
	If $\Compare$ is subadditive, then the last term in the Triplex Inequality is upper bounded by twice the sequential complexity, $2\Rad_T(\ell, \Phi_T, \Compare)$, and the first term is bounded by $2\Rad_T(\ell, \mc{I}, \Compare)$ where $\mc{I}$ is the singleton set consisting of the identity mapping.  
	Similarly, if $-\Compare$ is subadditive, then the last term is upper bounded by $2\Rad_T(\ell, \Phi_T, -\Compare)$ and the first term is bounded by $2\Rad_T(\ell, \mc{I}, -\Compare)$.
\end{theorem}

\paragraph{Discussion of Theorem~\ref{thm:main} and Theorem~\ref{thm:rad}}

\begin{itemize}

	\item First, let us mention that Triplex Inequality is not the only way to decompose the value of the game into useful and interpretable terms. In fact, slightly different decompositions yield better constants for some of the examples in this paper. Nonetheless, the Triplex Inequality seems to capture the essence of all the problems we considered and allows us to give a unified treatment to all of them.
	
	\item We note that the first and the third terms are similar in their form. In fact, the first term can be equivalently written as 
$$\sup_{p_1,q_1} \Eunder{x_1 \sim p_1}{f_1 \sim q_1} \ldots  \sup_{p_T,q_T} \Eunder{x_T \sim p_T}{f_T \sim q_T} \sup_{\bphi\in \mc{I}} 
\left\{ \Compare\Big(\ell_{\phi_1}(f_1,x_1), \ldots, \ell_{\phi_T}(f_T,x_T) \Big) - \Eunder{x'_{1:T} \sim p_{1:T} }{ f'_{1:T} \sim q_{1:T}} \Compare\Big(\ell_{\phi_1}(f'_1,x'_1), \ldots, \ell_{\phi_T}(f'_T, x'_T)\Big) 
\right\}$$
where $\mc{I}$ only contains the identity mapping. If $\mc{I}\subseteq \Phi_T$, then, trivially, $\Rad_T(\ell, \mc{I}, \Compare) \leq \Rad_T(\ell, \Phi_T, \Compare)$ and, therefore, an upper bound on the third term yields and upper bound on the first. However, in some situations $\Phi_T$ is ``simpler'' or incomparable to $\mc{I}$ and, hence, the first and the third term in the Triplex Inequality are distinct.

\item What exactly is achieved by Theorem~\ref{thm:rad}? Let us compare the third term in the Triplex Inequality to its sequential complexity upper bound given by Eq.~\eqref{eq:def_rademacher_expanded}. Both quantities involve interleaved suprema and expected values. However, in the former, the suprema are over the choice of distributions $p_t,q_t$ and the expected values are draws of $x_t,f_t$ from these mixed strategies. In contrast, sequential complexity, as written in Eq.~\eqref{eq:def_rademacher_expanded}, contains suprema over the choices $x_t,f_t$ followed by a random draw of the next sign $\epsilon_t$. Crucially, it is easier to work with the sequential complexity as opposed to the third term in the Triplex Inequality since in the former the only randomness comes from the random signs. In mathematical terms, the $\sigma$-algebra is generated by $\{\epsilon_t\}$ rather than a complicated stochastic process arising from the Triplex Inequality. This is one of the key observations of the paper. 

\item Depending on a particular problem, some of the terms in the Triplex Inequality might be easier to control than others. However, it is often the case that the first term is the easiest, as it naturally leads to the question of martingale convergence. The second term is typically bounded by providing a specific response strategy for the player if the mixed strategy of the adversary is known. This response strategy is similar to the so-called Blackwell's condition for approachability (see Section~\ref{sec:blackwell} for further comparison). The third term is arguably the most difficult as it captures complexity of the set of payoff transformations $\Phi_T$. Under the subadditivity assumption on $\Compare$, Theorem~\ref{thm:rad} upper bounds the first and third terms by the sequential complexity. 

\item We remark that the first and third terms in Triplex Inequality contain \emph{suprema} over the player's strategies $q_t$ instead of \emph{infima} as in the definition of the value of the game. The proof of Theorem~\ref{thm:main} points out the step where this over-bounding is done. While this might appear as a loose step, in all the examples we considered, this still yields the needed results. Nevertheless, as mentioned in the proof, one can substitute a particular strategy $q^*_t$ for the first and third terms instead of passing to the supremum. For instance, $q^*_t$ can be the strategy which makes the second term in the Triplex Inequality small. To simplify the presentation, we decided not to include such analysis.

\item The following observation gives us a simple condition under which we can replace $\Compare$ with some other $\Compare'$, and we shall find it useful in scenarios when it is difficult to directly deal with $\Compare$. If $\Compare : \cH^T \mapsto \reals$ and $\Compare' : \cH^T \mapsto \reals$ are such that $\forall z_1 , \ldots, z_T \in \cH$, $\Compare(z_1,\ldots,z_T) \le \Compare'(z_1,\ldots,z_T)$ then we have that for any class of transformations $\Phi_T$,
\begin{align}
	\label{eq:surrogate_B}
\Rad_T(\ell, \Phi_T, \Compare) \le \Rad_T(\ell, \Phi_T, \Compare') \ .
\end{align}

\item Finally, let us mention that we could have defined the performance measure in \eqref{eq:big_fat_definition} as
\begin{align}
	\label{eq:alternative_reg_def}
	\Reg_T =  \sup_{(\bphi', \bphi)\in (\Phi'_T\times\Phi_T)} \Compare(\ell_{\phi'_1}(f_1,x_1), \ldots, \ell_{\phi'_T}(f_T, x_T)) -  \Compare(\ell_{\phi_1}(f_1,x_1), \ldots, \ell_{\phi_T}(f_T, x_T)) \ .
\end{align}
Clearly, \eqref{eq:big_fat_definition} can be expressed as an instance of  \eqref{eq:alternative_reg_def} by setting $\Phi'_T = \mc{I}$. Conversely, if $\Compare$ is, for instance, an average of its coordinates, we can view definition \eqref{eq:alternative_reg_def} as a particular case of \eqref{eq:big_fat_definition}. Indeed, given a payoff $\ell$ and sets $\Phi'_T,\Phi_T$ of transformations, define a new payoff $\bar{\ell}(f,x) = 0$ and $\bar{\ell}_{(\phi'_t,\phi_t)}(f,x) = -(\ell_{\phi'_t}(f,x) - \ell_{\phi_t}(f,x))$. Then \eqref{eq:big_fat_definition} becomes exactly \eqref{eq:alternative_reg_def}. While the analysis presented in this paper can be extended for \eqref{eq:alternative_reg_def}, in the examples we consider, the definition \eqref{eq:big_fat_definition} of performance measure is expressive enough.

\end{itemize}

We now detail upper bounds on this complexity under the smoothness assumption on $\Compare$. The smoothness assumption covers many important cases, such as norms.

\subsection{General Bounds for Smooth $\Compare$}
\label{sec:smooth}

As shown by Pisier \cite{Pisier75} and Pinelis \cite{Pinelis94}, existence of a smooth norm in a Banach spaces is crucial in the study of exponential inequalities for martingales. Using similar techniques, we show that a smooth function $\Compare$ will admit upper bounds in terms of certain increments. This will yield general tools for studying sequential complexity for smooth functions $\Compare$. Informally, the smoothness assumption provides a link from a ``global'' function of coordinates to a sum of its parts. From the point of view of online learning, this is very promising, as it appears to be difficult to sequentially optimize a ``global'' function of many decisions.

Consider the following definition of smoothness.
\begin{definition}
Function $G : \cH \mapsto \reals$ is said to be $(\sigma,p)$-uniformly smooth on $\cH$ for some $p\in(1,2]$ and $\sigma \geq 0$ if, for all $z, z' \in \cH$, we have,
$$
G(z) \le G(z') + \inner{\nabla G(z'), z - z'} +  \frac{\sigma}{p}\|z - z'\|^p
$$
\end{definition}

We say that $G$ is \emph{uniformly smooth} if there exist finite $\sigma$ and $p$ such that $G$ is $(\sigma,p)$-uniformly smooth.
We say that the space $(\cB,\|\cdot\|)$ is $(\gamma,p)$-smooth when the function $\|\cdot\|^p/p$ is $(\gamma,p)$-uniformly smooth.

A function which is smooth in its arguments can be ``sequentially linearized'', with additional second-order terms as norms of the increments. We establish the following upper bound on the first term of the Triplex Inequality.

\begin{lemma} 
	\label{lem:smooth_bound_first_term}
	Suppose $\Compare$ is subadditive and for some $q\ge 1$, $\Compare^q$ is $(\sigma, p)$-uniformly smooth in each of its arguments. Suppose $\Compare(0,\ldots,0) = 0$ and that for any $x \in \X$ and $f \in \F$ it is true that $\|\ell(f,x)\| \le \RH$. Then the first term in the Triplex Inequality is bounded by $\left((2\RH)^p \sigma T/p \right)^{1/q}$.
\end{lemma}

Under the assumptions of Lemma~\ref{lem:smooth_bound_first_term}, we can also provide an upper bound on the third term. Lemma~\ref{lem:smoothrad} below says that the sequential complexity defined through a smooth function $\Compare$ can be upper bounded by the sequential   complexity involving a sum of first-order expansions of $\Compare$. 

\begin{lemma}
	\label{lem:smoothrad}
Assume that for some $q\ge 1$, $\Compare^q$ is $(\sigma, p)$-uniformly smooth in each of its arguments, $\Compare(0,\ldots,0) = 0$ and that for any $x \in \X$, $f \in \F$, $\bphi\in\Phi_T$ and $t\in[T]$, it is true that $\|\ell_{\phi_t}(f,x)\| \le \RH$, then we have that
\begin{align*}
 \Rad_T(\ell,\Phi_T) \le \left( \sup_{\f, \x}\ \En_{\epsilon_{1:T}} \sup_{\bphi\in \Phi_T}  \sum_{t=1}^T \epsilon_t g_t\big(\ell_{\phi_1}(\f_1(\epsilon),\x_1(\epsilon)),\ldots,\ell_{\phi_t}(\f_t(\epsilon),\x_t(\epsilon))\big) \right)^{1/q} + (\sigma\RH^p/p)^{1/q} T^{1/q} 
\end{align*}
where 
\begin{align*}
&g_t\big(\ell_{\phi_1}(\f_1(\epsilon),\x_1(\epsilon)),\ldots,\ell_{\phi_t}(\f_t(\epsilon),\x_t(\epsilon))\big) \\
& ~~~~~~~= \inner{\nabla_{t} \Compare^q\big(\epsilon_1 \ell_{\phi_1}(\f_1(\epsilon),\x_1(\epsilon)), \ldots,  \epsilon_{t-1} \ell_{\phi_{t-1}}(\f_{t-1}(\epsilon), \x_{t-1}(\epsilon)) ,0,\ldots,0\big) , \ell_{\phi_t}(\f_{t}(\epsilon), \x_{t}(\epsilon)) } \ .
\end{align*}
\end{lemma}

By taking gradients at successive time steps, we reduced the study of a global function $\Compare$ to the study of its gradients. A reader familiar with \cite{RakSriTew10} will notice that the first term of Lemma~\ref{lem:smoothrad} (under the power of $1/q$) resembles sequential Rademacher complexity. The first step in studying this term is to ask what can be done with a finite class $\Phi_T$. To approach this question, we state a lemma from \cite{RakSriTew10}.

\begin{lemma}\label{lem:fin}\cite{RakSriTew10}
For any finite set $V$ of $\reals$-valued trees of depth $T$ we have that
$$
\Es{\epsilon}{\max_{\v \in V} \sum_{t=1}^T \epsilon_t \v_t(\epsilon)} \le \sqrt{2 \log(|V|) \max_{\v \in V} \max_{\epsilon \in \{\pm1\}^T} \sum_{t=1}^T \v_t(\epsilon)^2} \ .
$$
\end{lemma}

The above Lemma can be used to show the following result for any finite set of transformations $\Phi_T$.

\begin{proposition}
	\label{prop:fin_phi}
For any finite set of payoff transformations $\Phi_T$, under the conditions of Lemma~\ref{lem:smoothrad} and assuming
$$\left\|\nabla_{t} \Compare^q\big(\epsilon_1 \ell_{\phi_1}(\f_1(\epsilon),\x_1(\epsilon)), \ldots,  \epsilon_{t-1} \ell_{\phi_{t-1}}(\f_{t-1}(\epsilon), \x_{t-1}(\epsilon)) ,0,\ldots,0\big)\right\| \le R$$
 then
\begin{align*}
\Rad_T(\ell,\Phi_T) \le \left(2 \RH^2 R^2 \log(|\Phi_T|) T\right)^{1/2q} + (\sigma\RH^p/p)^{1/q} T^{1/q} \ .
\end{align*}
\end{proposition}

Hence, if $\Phi_T$ is finite, sequential complexity is bounded whenever $\Compare$ is smooth and the gradients of $\Compare$ are bounded by $R$. Typically, $R$ is of the order $O(1/T)$ if $\Compare$ is appropriately normalized to account for $T$ (for instance, if $\Compare$ is an average of its coordinates). Similarly, $\sigma$ is either zero or $o(1)$ for the examples considered in this paper. With the appropriate behavior of the online covering number, the bound yields learnability according to Definition~\ref{def:learnability}.

\subsection{When $\Compare$ is a Function of the Average}

For the rest of this sub-section we consider $\Compare$ of a particular form. We assume that, $$\Compare(z_1,\ldots,z_T) = G\left(\frac{1}{T} \sum_{t=1}^T  z_t\right),$$ 
where some power of $G$ is $(\gamma,p)$-smooth function on the convex set $\conv(\cH)$ for some $1<p\leq 2$. This form of $\Compare$ occurs naturally in many games including Blackwell's approachability and calibration. Among the most basic smooth functions are powers of norms, as the next example shows.
\begin{example}
	Consider $\Compare$ of the form
	$$
	\Compare(z_1,\ldots,z_T) = \left\|\frac{1}{T} \sum_{t=1}^T z_t\right\|_q \ .
	$$
	The three cases $q\in(1,\infty)$, $q=1$, and $q=\infty$ are considered separately. Here $G = \|\cdot\|_q$ and we are interested in checking if $G^s$ is uniformly smooth for some power $s$.
	\begin{itemize}
		\item[$\blacktriangleright$~] $\bf q \in (1,\infty)$~~~~~
			For any $q \in (1,2]$, $G^q(z) = \|z\|_q^q$
is $(q,q)$-uniformly smooth and for any $q \in [2,\infty)$ the function $
G^2(z) = \|z\|_q^2$ is $(2(q-1),2)$-uniformly smooth. 
		\item[$\blacktriangleright$~] $\bf q=\infty$~~~~~
			Unfortunately, for no finite power $s$ is $G^s$ uniformly smooth. However, for any $z \in \cH$ and any $q' \in (1,\infty)$, $\|z\|_\infty \le \|z\|_{q'}$. Hence we can use \eqref{eq:surrogate_B} and upper bound the sequential complexity
$$
\Rad_T(\ell,\Phi_T,\Compare) \le \Rad_T(\ell,\Phi_T,\Compare')
$$
where $\Compare'(z_1,\ldots,z_T) = \left\|\frac{1}{T} \sum_{t=1}^T z_t\right\|_{q'}$. By choosing $q'$ appropriately and using the smoothness of the $L_{q'}$ norm (previous case) we can provide upper bounds for the value of the game.
		\item[$\blacktriangleright$~] $\bf q=1$~~~~~
			As in the previous example, for no finite power $s$ is $G^s$ uniformly smooth. However if $\cH \subseteq \reals^d$, then for any $z \in \cH$ and any $q' \in (1,\infty)$, $\|z\|_1 \le C_{q',d} \|z\|_{q'}$ where $C_{q',d}$ is a constant dependent on $q'$ and dimension of the space $d$. Again we can use \eqref{eq:surrogate_B} and upper bound 
$$
\Rad_T(\ell,\Phi_T,\Compare) \le \Rad_T(\ell,\Phi_T,\Compare')
$$
where $\Compare'(z_1,\ldots,z_T) = \left\|\frac{1}{T} \sum_{t=1}^T z_t\right\|_{q'}$. Choosing $q'$ appropriately and using the smoothness of the $L_{q'}$ norm we can provide upper bounds for the value of the game.
	\end{itemize}
\end{example}

For a concrete example of a smooth norm, we refer to the calibration example of Section~\ref{sec:calibration}. We now specialize the statement of Proposition~\ref{prop:fin_phi} to the specific assumption on $\Compare$.

\begin{corollary}
	\label{cor:fin_phi_smooth_sum}
Let $\Phi_T$ be a finite set of payoff transformations. Assume that for some $q\ge 1$, $G^q$ is $(\gamma,p)$-smooth function for some $1<p\leq 2$. Also assume that $\left\| \nabla G^q\left( z \right)\right\|_* \le \rho$ for any $z \in \conv(\cH)$.
Further, suppose that for any $x \in \X$, $f \in \F$, $\bphi\in\Phi_T$ and $t\in[T]$, it is true that $\|\ell_{\phi_t}(f,x)\| \le \RH$.
Then it holds that 
	$$\Rad_T(\ell,\Phi_T) \le \left(\frac{2 \RH^2 \log(|\Phi_T|)}{T}\right)^{1/2q} + (\gamma\RH^p/p)^{1/q} T^{(1-p)/q} \ .$$
\end{corollary}

The above result is a direct corollary of the more general Proposition~\ref{prop:fin_phi} in the case where $\Compare$ is a function of the average. It turns out
that we do not always get the best convergence rate in this manner. The following result shows that if $G$ is $1$-Lipschitz and $G^2$ is $2$-smooth, we should
obtain a $O(1/\sqrt{T})$ convergence rate.

\begin{lemma}
	\label{lem:fin_phi_2smooth_sum}
Let $\Phi_T$ be a finite set of payoff transformations. Assume that $\Compare(z_1,\ldots,z_T) = G\left(\frac{1}{T} \sum_{t=1}^T  z_t\right)$ where $G\ge0$ is $1$-Lipschitz with
respect to a norm $\|\cdot\|$, $G(0) = 0$ and $G^2$ is $(\gamma,2)$-smooth function. Further, suppose that for any $x \in \X$, $f \in \F$, $\bphi\in\Phi_T$ and $t\in[T]$, it is true that $\|\ell_{\phi_t}(f,x)\| \le \RH$.
Then, for $T \ge \log(2|\Phi_T|)/\gamma$, it holds that 
	$$\Rad_T(\ell,\Phi_T) \le 2\sqrt{\frac{\gamma\RH^2 \log(2|\Phi_T|)}{T}}$$
\end{lemma}

The next result generalizes the above lemma to the case when the exponent of smoothness is different from $2$. Because of a different proof strategy, there are
two differences between the next lemma and the previous one. First, instead of assuming smoothness of some power of $G$, we instead assume that the space
$(\cB,\|\cdot\|)$ is $(\gamma,p)$-smooth. Second, we get extra $\log(T)$ factors that are
probably an artifact of our analysis.

\begin{lemma}
	\label{lem:fin_phi_p_smooth_sum}
Let $\Phi_T$ be a finite set of payoff transformations with $|\Phi_T|>1$. Assume that $\Compare(z_1,\ldots,z_T) = G\left(\frac{1}{T} \sum_{t=1}^T  z_t\right)$ where $G\ge0$ is $1$-Lipschitz with
respect to a norm $\|\cdot\|$ and $G(0)=0$. Suppose that $(\cB,\|\cdot\|)$ is a $(\gamma,p)$-smooth space. Further, suppose that for any $x \in \X$, $f \in \F$,
$\bphi\in\Phi_T$ and $t\in[T]$, it is true that $\|\ell_{\phi_t}(f,x)\| \le \RH$.
Then, for any $T \ge 3$, it holds that 
	$$\Rad_T(\ell,\Phi_T) \le \frac{4\,c\,\gamma^{1/p}  \log^{3/2} T}{T^{1-1/p}} \sqrt{\RH^2 \log (2|\Phi_T|)}$$
	for some absolute constant $c$.
\end{lemma}

Having a bound on the complexity of a finite set of payoff transformations, we seek to extend the results to infinite sets. A natural approach is to pass to a finite cover of the set at an expense of losing an amount proportional to the resolution of the cover. Before proceeding, however, we need to define an appropriate notion of a cover. The following definition can be seen as a generalization of the corresponding notion introduced in \cite{RakSriTew10}. We remark that the object, for which we would like to provide a cover, is the set $\Phi_T$ of payoff transformations. Whenever payoff transformations are simply constant time-invariant departure mappings, complexity of $\Phi_T$ identical to that of $\F$, yielding the online cover of class $\F$ (see Section~\ref{sec:external} for more details). In general, however, the set of payoff transformations can be much more complex than (or not even comparable to) $\F$.

\begin{definition}
	\label{def:cover}
A set $V$ of $\cH$-valued trees of depth $T$ is \emph{an $\alpha$-cover} (with respect to $\ell_p$-norm) of $\Phi_T$ on an $(\F\times\X)$-valued tree $(\f,\x)$ of depth $T$ if
\begin{align}
	\label{eq:def_cover}
\forall \bphi \in \Phi_T,\ \forall \epsilon \in \{\pm1\}^T \ \exists \v \in V \  \mrm{s.t.}  ~~~~ \left( \frac{1}{T} \sum_{t=1}^T \left\|\v_t(\epsilon) - \ell_{\phi_t}(\f_t(\epsilon), \x_t(\epsilon)) \right\|^p \right)^{1/p} \le \alpha
\end{align}
The \emph{covering number} of the set of payoff transformations $\Phi_T$ on a given tree $(\f,\x)$ is defined as 
$$
\N_p(\alpha, \Phi_T, (\f, \x)) = \min\{|V| :  V \ \trm{is an }\alpha-\text{cover w.r.t. }\ell_p\trm{-norm of }\Phi_T \trm{ on } (\f, \x) \trm{ tree}\}.
$$
Further define $\N_p(\alpha, \Phi_T , T) = \sup_{(\f,\x)} \N_p(\alpha, \Phi_T, (\f, \x)) $, the maximal $\ell_p$ covering number of $\Phi_T$ over depth $T$ trees. 
\end{definition}

This definition of the cover is indeed the most general for the setting we consider in this paper. In sections that follow, we specialize this definition to fit particular assumptions on $\Phi_T$.

We now give generalizations Dudley's bound for the case when $\Compare$ is a function of the average.

\begin{theorem}
	\label{thm:2smooth_sum_dudley}
Assume that $\Compare(z_1,\ldots,z_T) = G\left(\frac{1}{T} \sum_{t=1}^T  z_t\right)$ where $G\ge0$ is sub-additive, $1$-Lipschitz with
respect to a norm $\|\cdot\|$, $G(0) = 0$ and $G^2$ is $(\gamma,2)$-smooth. Further, suppose that for any $x \in \X$, $f \in \F$, $\bphi\in\Phi_T$ and $t\in[T]$, it is true that $\|\ell_{\phi_t}(f,x)\| \le 1$.
Then it holds that 
	$$\Rad_T(\ell,\Phi_T) \le 4 \inf_{\alpha > 0} \left\{ \alpha + 6 \sqrt{\frac{\gamma}{T}} \int_{\alpha}^1 \sqrt{ \log \cN_\infty(\beta,\Phi_T,T) } d\beta \right\} $$
\end{theorem}

\subsection{General Bounds Under Linearity Assumptions on $\Compare$}

The general results of the previous section can be restated in simpler terms once more assumptions are made. In particular, some of the terms in the three-term decomposition in Theorem~\ref{thm:main} can be dropped as soon as $\Compare$ is linear. While some of the results below can be repeated for a more general form  
$\Compare(z_1,\ldots,z_T) = \sum_{t=1}^T \inner{c_t,z_t}$
(for some $c_1,\ldots, c_T \in \mathcal{B}^*$ and $\cH\subseteq \mathcal{B}$), for simplicity we assume that $\Compare$ is an average of its arguments and that $\cH \subseteq \reals$:
$$\Compare(z_1,\ldots,z_T) = \frac{1}{T} \sum_{t=1}^T  z_t \ .$$

Of course, such $\Compare$ is trivially smooth (with $\sigma=0$), so all the results of the previous section apply. 
\begin{corollary}
	\label{cor:simple_consequences}
	The following statements hold: 
	\begin{itemize}
		\item The first term in the Triplex Inequality is zero. 
		\item If $\Phi_T$ is a class of departure mappings, then the second term in the Triplex Inequality is non-positive. In this case,  
		$$\Val_T (\ell, \Phi_T)\leq 2\Rad_T(\ell,\Phi_T).$$
		\item Let $\cH\subseteq [-1,1]$. We have, 
		\begin{align*}
		 \Rad_T(\ell,\Phi_T) \le  4 \inf_{\alpha \ge 0} \left\{ \alpha + 6 \sqrt{2} \int_{\alpha}^{1} \sqrt{\frac{\log \N_\infty(\delta, \Phi_T , T)}{T}} d \delta \right\} \ .
		\end{align*}
	\end{itemize}
\end{corollary}

Note that the use of $\ell_\infty$ covering numbers in the above result is not essential. In the case $\cH \subseteq [-1,1]$, we can use $\ell_2$ covering numbers by adapting
the proof of Theorem 9 in \cite{RakSriTew10}.

When $\Compare$ is the average of its coordinates, the sequential complexity takes on a familiar form:
$$
\Rad_T(\ell, \Phi_T) = \sup_{\f, \x}\ \En_{\epsilon_{1:T}} \sup_{\bphi\in \Phi_T} \frac{1}{T}\sum_{t=1}^T \epsilon_t \ell_{\phi_t}(  \f_t(\epsilon),\x_t(\epsilon)). 
$$
Further, for $\cH\subseteq \reals$, Eq.~\eqref{eq:def_cover} in definition of the cover becomes
\begin{align*}
\forall \bphi \in \Phi_T,\ \forall \epsilon \in \{\pm1\}^T \ \exists \v \in V \  \mrm{s.t.}  ~~~~ \left( \frac{1}{T} \sum_{t=1}^T \left| \v_t(\epsilon) - \ell_{\phi_t}(\f_t(\epsilon), \x_t(\epsilon)) \right|^p \right)^{1/p} \le \alpha
\end{align*}
where $V$ is now a set of $\reals$-valued trees. 

A further simplification of various notions is obtained for time-invariant payoff transformations. Moreover, for time-invariant payoff transformations we can define combinatorial parameters, generalizing the Littlestone's \cite{Lit88,BenPalSha09} and fat-shattering dimensions \cite{RakSriTew10}. This is the subject of the next section.

\subsubsection{Combinatorial Parameters for Time-Invariant Payoff Transformations}
\label{sec:time_invariant_combinatorial}

Assume $\cH\subseteq \reals$. Consider time-invariant payoff transformations generated from some base class of payoff transformations $\Phi$ (see Definition~\ref{def:time_inv_and_product}). That is, $\Phi_T = \{(\phi,\ldots,\phi): \phi\in\Phi\}$. We have the following definition of a generalized shattering dimension. 

\begin{definition}
Let $\cH=\{\pm1\}$. An $(\F\times\X)$-valued tree $(\f,\x)$ of depth $d$ is \emph{shattered}\footnote{As a historical aside, the term ``shattered set'' was introduced by J. Michael Steele in his Ph.D. thesis in 1975.} by a payoff transformation class $\Phi$ if for all $\epsilon \in \{\pm1\}^{d}$, there exists $\phi\in\Phi$ such that $\ell_{\phi}(\f_t(\epsilon),\x_t(\epsilon)) = \epsilon_t$ for all $t \in [d]$. 
	The \emph{shattering dimension} $\ldim(\Phi)$ is the largest $d$ such that $\Phi$ shatters an $(\F\times\X)$-valued tree of depth $d$. 
\end{definition}

We can also define the scale-sensitive version of the shattering dimension, generalizing the fat-shattering dimension of \cite{RakSriTew10}.
\begin{definition}
An $(\F\times\X)$-valued tree $(\f,\x)$ of depth $d$ is \emph{$\alpha$-shattered} by a payoff transformation class $\Phi$, if there exists an $\reals$-valued tree $\mbf{s}$ of depth $d$ such that 
$$
\forall \epsilon \in \{\pm1\}^d , \ \exists \phi \in \Phi \ \ \ \trm{s.t. } \forall t \in [d], \  \epsilon_t \Big(\ell_{\phi}(\f_t(\epsilon),\x_t(\epsilon)) - \mbf{s}_t(\epsilon) \Big) \ge \alpha/2
$$
The tree $\mbf{s}$ is called the \emph{witness to shattering}. The \emph{fat-shattering dimension} $\fat_\alpha(\Phi)$ at scale $\alpha$ is the largest $d$ such that $\Phi$ $\alpha$-shatters an $(\F\times\X)$-valued tree of depth $d$. 
\end{definition}

Slightly abusing notation, we write $\N_p(\alpha, \Phi, (\f,\x))$ instead of  $\N_p(\alpha, \Phi_T, (\f,\x))$ whenever $\Phi_T$ consists of sequences of time-invariant payoff transformations with a base class $\Phi$. 

The combinatorial parameters are useful if they can be shown to control problem complexity through, for instance, covering numbers. We state the following three results without proofs, as the arguments are identical to the ones given in \cite{RakSriTew10}. To be precise, the $(\f,\x)$ tree here plays the role of the $\x$ tree in \cite{RakSriTew10}, $\ell_\phi$ for $\phi\in\Phi$ plays the role of $f\in\F$ in \cite{RakSriTew10}.

\begin{theorem}
	\label{thm:sauer_multiclass}
	Let $\cH \subseteq \{0,\ldots, k\}$ and $\fat_2(\Phi) = d$. Then 
	$$ \N_\infty(1/2, \Phi, T) \leq \sum_{i=0}^d {T\choose i} k^i \leq \left(ekT \right)^d.$$ 
	Furthermore, for $T\geq d$ 
	$$\sum_{i=0}^d {T\choose i} k^i \leq \left(\frac{ekT}{d}\right)^d.$$
\end{theorem}

We now show that the covering numbers are bounded in terms of the fat-shattering dimension.
\begin{corollary}
	\label{cor:l2_norm_bound}
	Suppose $\cH\subseteq [-1,1]$. Then for any $\alpha >0$, any $T>0$, and any $(\F\times\X)$-valued tree $(\f,\x)$ of depth $T$,
	$$ \N_1(\alpha, \Phi, (\f, \x)) \leq \N_2(\alpha, \Phi, (\f, \x)) \leq \N_\infty(\alpha, \Phi, (\f,\x)) \leq \left(\frac{2e T}{\alpha}\right)^{\fat_{\alpha} (\Phi) }$$
\end{corollary}

\begin{theorem}
	\label{thm:sauer_multiclass_0_cover}
	Let $\cH \subseteq \{0,\ldots, k\}$ and $\fat_1(\Phi) = d$. Then 
	$$ \N (0, \Phi, T) \leq \sum_{i=0}^d {T\choose i} k^i \leq \left(ekT \right)^d.$$ 
	Furthermore, for $T\geq d$ 
	$$\sum_{i=0}^d {T\choose i} k^i \leq \left(\frac{ekT}{d}\right)^d.$$
	In particular, the result holds for binary-valued function classes ($k=1$), in which case $\fat_1(\Phi)=\ldim(\Phi)$.
\end{theorem}

The generality of these results is evident, as both the combinatorial parameters and covering numbers are defined for any performance measure \eqref{eq:big_fat_definition} with time-invariant payoff transformations. In particular, this includes $\Phi$-regret (see Section~\ref{sec:phiregret}).

\subsection{General Bounds for Slowly-Varying Payoff Transformations}

In Section~\ref{sec:time_invariant_combinatorial}, we assumed that the set $\Phi_T$ of sequences of payoff transformations is time-invariant. This assumption naturally leads to a control on the complexity of $\Phi_T$. Lifting the assumption of time-invariance, we now go back to the level of generality of Proposition~\ref{prop:fin_phi}. We observe that size of $\Phi_T$ or an appropriately behaving covering number $\N_2(\alpha, \Phi_T, T)$ is key for bounding the sequential complexity. If payoff transformations change wildly in time, there is little hope of getting non-trivial bounds. The good news is that, under some assumptions on the variability of the sequences in $\Phi_T$, we can get a bound on the covering number of $\Phi_T$. 

It has been shown in \cite{HerWar98,BouWar02} that it is possible to have small external regret against comparators that change a limited number of times. This alleviates an obvious limitation of the classical notion of external regret, viz., comparison to the fixed best decision. Another result of this flavor appears in \cite{Zinkevich03}, where {\em dynamic regret} is defined with respect to a comparator whose path length is bounded. In general, one can consider situations where we would like to compete with a budgeted comparator. We now show that the assumptions of slowly-varying or budgeted comparators are naturally captured by our framework through the notion of slowly-changing payoff transformations $\Phi_T$. Furthermore, the control of covering numbers of $\Phi_T$ becomes transparent under such assumptions. Our goal here is not to provide a comprehensive list of possible results, but rather to show versatility of our framework.

\subsubsection{Tracking the Best Transformation}

Suppose $\Phi$ is a finite set of payoff transformations. Let $\Phi^k_T$ be obtained by considering all piecewise constant sequences with $k$ changes: 
$$\Phi^k_T=\{(\phi_1,\ldots,\phi_T): 1 = i_0 \leq i_1\leq \ldots \leq i_k \leq T \mbox{ and } \phi_{t} = \phi_{t'} \mbox{ if } i_s\leq t\leq t' < i_{s+1} \mbox{ for some } s \geq 0\}.$$
If cardinality $|\Phi| = N$, it is easy to check that $|\Phi^k_T| \leq {T \choose k} \cdot N^{k+1}$. 
Under the assumptions of Proposition~\ref{prop:fin_phi}, this immediately implies a bound of the order
\begin{align*}
\left(R^2 (k\log N + k\log T) T\right)^{1/2q} + \sigma^{1/q} T^{1/q} \ .
\end{align*}

It is natural to extend the above results by lifting the assumption that $\Phi$ is a finite set of payoff transformations. This can be done by considering an online cover $\N_p (\ell, \Phi, \alpha)$ of $\Phi$ in some $\ell_p$ norm along with the same definition of $\Phi^k_T$. Next we do this in an even more general setting.

\subsubsection{Slowly Changing Transformations}

To start, suppose $\Phi_T$ consists of payoff transformations $(\phi_1,\ldots,\phi_T)$ which are ``almost'' time-invariant within each of $k+1$ intervals. Consider the following definition:
\begin{align*}
	\Phi^{k,\alpha}_T&=\Big\{(\phi_1,\ldots,\phi_T): 1 = i_0 \leq i_1\leq \ldots \leq i_k \leq T \\
	& ~~~~~~~\mbox{ and } \sup_{f,x}\|\ell_{\phi_{t}}(f,x) - \ell_{\phi_{t'}}(f,x)\| \leq \alpha \mbox{ if } i_s\leq t\leq t' < i_{s+1} \mbox{ for some } s \geq 0\Big\}.
\end{align*}
One can think of the time-invariant segments as ``accumulation points'' where the payoff transformations do not vary much.

Suppose that we have a finite cover $V$ of $\Phi$ at scale $\alpha$, of cardinality $|V| = \N_\infty(\alpha, \Phi, T) $. The $L_\infty$ covering is chosen for the purposes of simplicity, though tighter (and more difficult) results are expected from directly studying $L_2$ covering numbers. 

\begin{lemma}
	\label{lem:accum_pts}
	If $\N_\infty(\alpha, \Phi, T)$ is finite, 
	$$\N_\infty(2\alpha, \Phi^{k,\alpha}_T, T) \leq {T \choose k} \cdot \N_\infty(\alpha, \Phi, T)^{k+1} \ .$$
\end{lemma}

Further extending the above results, we will now study the size of an online cover if $\Phi_T$ consists of  payoff transformations of bounded length. In general, ``length'' can be defined as some budget given by the setting at hand. Here, we present a straightforward approach without an attempt to give very general and tight bounds.

Suppose that $\Phi_T$ is a set of sequences $(\phi_1,\ldots,\phi_T)$ of payoff transformations which do not ``vary much'', according to the following definition. The length of a sequence $(\phi_1,\ldots, \phi_T)$ of payoff transformations (with respect to $L_\infty$ distance) is defined as
$$\text{len}(\phi_1,\ldots,\phi_T):= \sum_{t=1}^{T-1} \sup_{f,x} \left\|\ell_{\phi_{t}}(f,x) - \ell_{\phi_{t+1}}(f,x) \right\| .$$
Again, we consider the $L_\infty$ distance between payoffs (as functions over $\F\times\X$). Assume that for all sequences in $\Phi_T$, their length is bounded by some $L>0$. We will now claim that by choosing $k$ large enough, the set of covering trees $V^{k}$ defined in the proof of Lemma~\ref{lem:accum_pts} provides a cover for $\Phi_T$ at a given scale $\alpha>0$. Consider any $(\phi_1,\ldots,\phi_T)\in \Phi_T$. We construct the nondecreasing sequence $i_1,\ldots, i_j,\ldots \in \{1,\ldots,T\}$ of ``change-points'' as follows: increase $t$ until the next payoff transformation is farther than $\alpha$ from the payoff transformation at $i_j$:
$$i_{j+1} = \inf_{t>i_j} \left\{ \sup_{f,x} \left\|\ell_{\phi_{i_j}}(f,x) - \ell_{\phi_{t}}(f,x)\right\| \geq \alpha \right\}$$ 
Let $k$ be the length of the largest such sequence for all elements of $\Phi_T$. We have simply reduced the problem to the one studied in the previous section: within each block, all the payoff transformations are close. 

Clearly, $k=k(\alpha) \leq L/\alpha$, but can potentially be smaller under additional assumptions on $\Phi_T$. We then have a bound on the size of a $2\alpha$-cover of $\Phi_T$:
$$\N_\infty (2\alpha, \Phi_T, T) \leq {T \choose k(\alpha)} \cdot \N_\infty(\alpha, \Phi, T)^{k(\alpha)+1} \leq {T \choose L/\alpha} \cdot \N_\infty(\alpha, \Phi, T)^{L/\alpha+1},$$
and 
$$\log \N_\infty (2\alpha, \Phi_T, T) \leq O\left( \frac{L}{\alpha} \log T + \frac{L}{\alpha} \log \N_\infty (\alpha, \Phi, T) \right) \ .$$

The covering number can be now used, for example in Theorem~\ref{thm:2smooth_sum_dudley}, to control sequential complexity when $\Compare$ is a function of the average. We note that it is possible to derive analogous Dudley's integral type bound solely under smoothness assumptions on $\Compare$.

%% file: lower.tex
\section{Techniques for Lower Bounds}
\label{sec:lower}

It is well-known that an {\em equalizing strategy} (i.e. a strategy that makes the move of the other player ``irrelevant'') can often be shown to be minimax optimal. In this section, we define a notion of an equalizer for our repeated game and show that it can be used to prove \emph{lower bounds} on the value of the game. While existence of an equalizer has to be established for particular problems at hand, the lower bounds below hold whenever such an equalizer exists.

\begin{definition}\label{def:equalizer}
A strategy $\left\{p^*_t \right\}$ for the adversary is said to be an \emph{equalizer strategy} if 
\begin{align*}
\Eunder{f_1 \sim q^*_1}{x_1 \sim p^*_1} \ldots \Eunder{f_T \sim q^*_T}{x_T \sim p^*_T} \Reg_T\left((f_1,x_1), \ldots, (f_T,x_T)\right)  = \Eunder{f_1 \sim \overline{q^*_1}}{x_1 \sim p^*_1} \ldots \Eunder{f_T \sim \overline{q^*_T}}{x_T \sim p^*_T} \Reg_T\left((f_1,x_1), \ldots, (f_T,x_T)\right)
\end{align*}
for all strategies $\left\{q^*_t\right\}$ and $\left\{\overline{q^*_t}\right\}$ of the player. Here $\Reg_T$ is defined as in \eqref{eq:big_fat_definition}.
\end{definition}

Using the above definition of an equalizer we have the following proposition as an immediate consequence. 

\begin{proposition}\label{prop:equalizer}
For any Equalizer strategy $\left\{p^*_t \right\}$ we have that for any $f \in \F$,
$$
\Val_T(\ell,\Phi_T) \ge \Eu{x_1 \sim p_1} \ldots \Eu{x_T \sim p_T}\left[ \Compare\left(\ell(f,x_1), \ldots, \ell(f,x_T)\right) - \inf_{\phi \in \Phi_T} \Compare\left(\ell_{\phi_1}(f,x_1), \ldots, \ell_{\phi_T}(f,x_T)\right)  \right]
$$
where $p_t = p^*_t\left(\left\{f_s = f, x_s\right\}_{s=1}^{t-1} \right)$
\end{proposition}

\begin{remark}
\label{rem:equalizer_simplification}
For many interesting games we consider it is often the case that for any $x_1,\ldots,x_T$ and any $f_1,\ldots,f_T,f'_1,\ldots,f'_T$,
$$
\inf_{\bphi \in \Phi_T} \Compare\left(\ell_{\phi_1}(f_1,x_1), \ldots, \ell_{\phi_T}(f_T,x_T) \right) = \inf_{\bphi \in \Phi_T} \Compare\left(\ell_{\phi_1}(f'_1,x_1), \ldots, \ell_{\phi_T}(f'_T,x_T) \right)
$$
In these cases since the player's actions do not even affect the second term of the regret, to check if a strategy $\{p^*_t\}$ is an equalizer or not we only need to check if
\begin{align*}
\Eunder{f_1 \sim q^*_1}{x_1 \sim p^*_1} \ldots \Eunder{f_T \sim q^*_T}{x_T \sim p^*_T} \Compare\left(\ell(f_1,x_1), \ldots, \ell(f_T,x_T) \right)  = \Eunder{f_1 \sim \overline{q^*_1}}{x_1 \sim p^*_1} \ldots \Eunder{f_T \sim \overline{q^*_T}}{x_T \sim p^*_T} \Compare\left(\ell(f_1,x_1), \ldots, \ell(f_T,x_T) \right)
\end{align*}
for all strategies $\{q^*_t\}$ and $\{\overline{q^*_t}\}$ of the player.
\end{remark}

Interestingly enough, many of the existing lower bounds in online learning literature are, in fact, equalizers (see e.g. \cite[p. 252]{PLG}). In particular, in \cite{AbeAgaBarRak09}, a lower bound on the value of the game was derived by looking at a certain {\em face} of a convex hull of loss vectors. The face, supported by a probability distribution $p$, corresponds to the set of functions with the same expected loss under the distribution $p$. Hence, $p$ is an equalizing strategy for those functions. Since these functions are the ``best'' with respect to this distribution, a lower bound in terms of complexity of this set was derived in \cite{AbeAgaBarRak09}. Furthermore, \cite{LeeBarWil98importance} shows that a lower bound on the rate of convergence in the i.i.d. setting is achieved when there are two distinct minimizers of expected error for a given distribution. Again, this distribution can be viewed as an equalizer for the non-singleton set of minimizers of expected error.

%% file: examples.tex
\section{Examples and Comparison to Known Results}\label{sec:examples}

We now turn to several specific settings studied in the literature and look at them through the prism of our general results. While we believe that online learnability in many different scenarios can be established through our framework, we decided to focus on several major problems. On the surface, these problems are quite different; yet, through our unified approach we show that learnability can be seamlessly established for all of them. The unification not only leads to simpler proofs and sharper results, but also yields insight into the inherent complexity and ways of making more comprehensive statements.

\input{example_phi}

\input{example_blackwell}

\input{example_calibration}

\input{example_global}

%% file: example_phi.tex
\subsection{$\Phi$-Regret}
\label{sec:phiregret}

In this section, we consider a particular notion of performance measure, known as $\Phi$-regret \cite{StoLug07, GorGreMar08, HazKal07}. In our framework, this means that we restrict ourselves to only \emph{time-invariant departure mapping classes} $\Phi_T$ specified by a base class $\Phi$ of mappings from $\F$ to itself (see Definitions~\ref{def:departure} and \ref{def:time_inv_and_product}). The particular choices of $\Phi$ lead to various notions, such as external, internal, swap regret, and more. 

To define $\Phi$-regret (Example~\ref{eg:phi}), we fix a set $\Phi$ of departure mappings which map $\F$ to $\F$ and define the set of time-invariant departure mappings $\Phi_T := \{(\phi,\ldots,\phi):\phi\in\Phi\}$. Then the measure of performance becomes $\Phi$-regret:
$$ \Reg_T = \frac{1}{T}\sum_{t=1}^T \ell(f_t, x_t) - \inf_{\phi\in\Phi} \frac{1}{T}\sum_{t=1}^T \ell(\phi(f_t), x_t),$$
where $\cH \subseteq \reals$. Since $\Compare$ is the average of its arguments, Corollary~\ref{cor:simple_consequences} implies

\begin{corollary}
	\label{cor:phi_value_bd_by_rad}
	In the setting of $\Phi$-regret,
	\begin{align*}
	 	\Val_T(\ell, \Phi) &\leq 2\Rad (\ell, \Phi) \ .
	\end{align*}
\end{corollary}

Specializing the definition of sequential complexity to $\Phi$-regret, we obtain the following definition.
\begin{definition}
	\label{def:rad_for_phi}
	The sequential complexity for $\Phi$-regret is defined as
	\begin{align}
		\Rad_T(\ell, \Phi) = \sup_{(\f, \x)}\ \En_{\epsilon_{1:T}} \sup_{\phi\in \Phi} \frac{1}{T}\sum_{t=1}^T \epsilon_t \ell( \phi \circ \f_t(\epsilon),\x_t(\epsilon))
	\end{align}
	where, as before, the first supremum is over $\F\times\X$-valued trees $(\f,\x)$ of depth $T$.
\end{definition}

The following property allows us to immediately obtain bounds for convex hulls of finite sets $\Phi$.
\begin{proposition}
	Suppose $\ell$ is convex in the first argument and $\conv(\Phi)$ maps $\F$ into $\F$. Then
	$$ \Rad_T(\ell, \conv(\Phi)) = \Rad_T (\ell, \Phi) \ .$$
\end{proposition}

We also have the following version of the contraction lemma, whose proof is identical to that given in \cite{RakSriTew10}.
\begin{lemma}
	\label{lem:contraction}
	Fix a function $\psi:\reals \times \F\times\X \mapsto \reals$ such that for any $f\in\F, x\in\X$, $\psi(\cdot,f,x)$ is a Lipschitz function with a constant $L$. Then
	$$ \Rad (\psi\circ\ell, \Phi) \leq L \cdot \Rad(\ell, \Phi)$$
	where $\psi\circ \ell$ is defined by the mapping $(f, x) \mapsto \psi(\ell(f,x), f,x)$ for all $f\in \F, x\in\X$.
\end{lemma}

Next, we specialize Definition~\ref{def:cover} to the particular case of $\Phi$-regret.

\begin{definition}
	\label{def:cover_Phi_regret}
	A set $V$ of $\reals$-valued trees of depth $T$ is \emph{an $\alpha$-cover} (with respect to $\ell_p$-norm) of $\Phi_T$ on the $\F\times\X$-valued tree $(\f,\x)$ of depth $T$ if
$$
\forall \phi \in \Phi,\ \forall \epsilon \in \{\pm1\}^T \ \exists \v \in V \  \mrm{s.t.}  ~~~~ \left( \frac{1}{T} \sum_{t=1}^T \left|\v_t(\epsilon) - \ell(\phi\circ \f_t(\epsilon), \x_t(\epsilon)) \right|^p \right)^{1/p} \le \alpha
$$
The \emph{covering number} of $\Phi_T$ on a given tree $(\f,\x)$ is defined as the size of the minimum cover, as in Definition~\ref{def:cover}.
\end{definition}

We now turn to particular examples to utilize the results and definitions stated above.

\subsubsection{External Regret}
\label{sec:external}
External regret is the simplest example of $\Phi$-regret. We separate it from the general discussion in order to show that for external regret the various notions introduced in this paper reduce to the ones proposed in \cite{RakSriTew10}. 

Considering the definitions in Example~\ref{def:departure}, notice that the time-invariant departure mappings class $\Phi_T$ is chosen to be the class of sequences of \emph{constant} mappings $\{(\phi_f,\ldots,\phi_f): f\in\F ~\mbox{ and }~ \phi_f(g) = f ~\forall g\in\F  \}$. It is precisely because of this constancy of $\bphi$  that the dependence on the $\F$-valued tree $\f$ disappears from all the definitions and results. Further, because of the obvious bijection between elements of $\Phi_T$ and $\F$, minimization (maximization) over $\Phi_T$ can be written as minimization (maximization) over $\F$. Notice that the action of $\phi_f$ on the payoff is $\ell_{\phi_f}(f_t,x_t) = \ell(f, x_t)$.

Let us turn to Definition~\ref{def:rad_for_phi} of the sequential complexity for $\Phi$-regret. Because each $\phi_f\in\Phi$ is a constant mapping, we have
\begin{align}
	\label{eq:rademacher_external_regret}
\Rad_T(\ell, \Phi) &= \sup_{\f, \x}\ \En_{\epsilon_{1:T}} \sup_{f\in\F} \frac{1}{T}\sum_{t=1}^T \epsilon_t \ell ( f,\x_t(\epsilon)) \notag\\
&= \sup_{\x}\ \En_{\epsilon_{1:T}} \sup_{f\in\F} \frac{1}{T}\sum_{t=1}^T \epsilon_t \ell ( f,\x_t(\epsilon)).
\end{align}
If payoff is written as $\ell(f,x) = f(x)$, this is precisely the sequential Rademacher complexity  defined in \cite{RakSriTew10}. 

Next, we show that Definition~\ref{def:cover_Phi_regret} reduces to the definition of online covering given in \cite{RakSriTew10}. Indeed, $\ell_{\phi_f}(\f_t(\epsilon), \x_t(\epsilon)) = \ell(f, \x_t(\epsilon)) $ for the constant mappings $\bphi=(\phi_f,\ldots,\phi_f)$. Further, the payoff space $\cH \subseteq \reals$. With these simplifications, the closeness to a covering element in Definition~\ref{def:cover_Phi_regret} becomes
$$
\forall f \in \F,\ \forall \epsilon \in \{\pm1\}^T \ \exists \v \in V \  \mrm{s.t.}  ~~~~ \left( \frac{1}{T} \sum_{t=1}^T \left|\v_t(\epsilon) - \ell(f, \x_t(\epsilon)) \right|^p \right)^{1/p} \le \alpha
$$
where $V$ is a set of $\reals$-valued trees. It is then immediate that Corollary~\ref{cor:simple_consequences} recovers the corresponding result of \cite{RakSriTew10}. For a detailed study of external regret, we refer the reader to the companion paper \cite{RakSriTew10}.

\paragraph{Lower Bounds in the Supervised Setting}

We provide a lower bound for external regret in the supervised learning setting using the notion of an equalizer (see Section~\ref{sec:lower}). To this end, we assume that $\X = \Z \times \Y$ where $\Z$ is the space of predictors and $\Y$ is the space of responses (outcomes). The setting is called {\em supervised} because, in the machine learning terminology, the observed data is thought of as examples together with labels. Assume $\F$ is a class of bounded real-valued functions and the space of outcomes is a bounded interval; for simplicity let $\F \subseteq [-1,1]^\Z$ and $\Y = [-1,1]$. Suppose the loss is of the form $\ell(f, (z,y)) = |f(z)-y|$.

\begin{proposition}
	The value of the supervised game defined above is lower bounded by sequential Rademacher complexity:
	$$\Val^S_T(\ell,\Phi_T) \geq \Rad_T(\ell, \Phi) $$
\end{proposition}
\begin{proof}
Recall that we have a fixed set $\Phi$ of {\em constant} departure mappings. We will now exhibit an equalizer strategy. Following Remark~\ref{rem:equalizer_simplification}, observe that for any $(z_1,y_1),\ldots,(z_T,y_T)$ and any $f_1,\ldots,f_T,f'_1,\ldots,f'_T$,
$$
\inf_{\phi \in \Phi} \frac{1}{T}\sum_{t=1}^T |(\phi\circ f_t) (z_t)- y_t|  = \inf_{\phi\in\Phi} \frac{1}{T}\sum_{t=1}^T |(\phi\circ f'_t) (z_t)- y_t|
$$
because any $\phi\in\Phi$ is a constant mapping. Thus, for a strategy to be an equalizer, it only needs to ``equalize'' the cumulative loss of the player. Here is how we construct such a strategy. Let $p^y$ be defined as the distribution of a Rademacher $\pm 1$ random variable $Y$; this will define the labels $y_t$ as independent coin flips. Now, fix any $\Z$-valued tree $\z$ of depth $T$. Let $\{p^*_t\}$ be a strategy defined by $p^*_t(y_{1:t-1}) = \delta_{\z_t(y_{1:t-1})} \times p^y$, a delta distribution on $\z_t(y_{1:t-1})$ defined by the tree $\z$ and $p^y$ on $\Y$. In plain words, the strategy of the adversary for each $t$ is to choose a particular $z_t\in\Z$ given the labels $y_1,\ldots,y_{t-1}$, and let the label be an independent Rademacher random variable.

By Remark~\ref{rem:equalizer_simplification}, it is enough to check
\begin{align*}
\Eunder{f_1 \sim q^*_1}{(z_1,y_1) \sim p^*_1} \ldots \Eunder{f_T \sim q^*_T}{(z_T,y_T) \sim p^*_T} \frac{1}{T}\sum_{t=1}^T |f_t(z_t)-y_t|  = \Eunder{f_1 \sim \overline{q^*_1}}{(z_1,y_1) \sim p^*_1} \ldots \Eunder{f_T \sim \overline{q^*_T}}{(z_T,y_T) \sim p^*_T} \frac{1}{T}\sum_{t=1}^T |f_t(z_t)-y_t|
\end{align*}
for all strategies $\{q^*_t\}$ and $\{\overline{q^*_t}\}$ of the player. This equality is indeed true because $\Eu{y_t\sim p^y} |a-y_t| = 1$ independently of the constant $a\in[-1,1]$. By Proposition~\ref{prop:equalizer}, for any $g\in\F$
\begin{align*}
\Val^S_T(\ell,\Phi_T) &\ge \Eu{(z_1,y_1) \sim p^*_1} \ldots \Eu{(z_T,y_T) \sim p^*_T}\left[ \frac{1}{T}\sum_{t=1}^T |g(z_t)-y_t| - \inf_{f\in\F} \frac{1}{T}\sum_{t=1}^T |f(z_t)-y_t|  \right]\\
&= \Eu{y_1,\ldots,y_T}\left[ 1 - \inf_{f\in\F} \frac{1}{T}\sum_{t=1}^T |f(\z_t(y_{1:t-1}))-y_t|  \right] \\
&= \Eu{y_1,\ldots,y_T}\left[ \sup_{f\in\F} \frac{1}{T}\sum_{t=1}^T y_t f(\z_t(y_{1:t-1}))  \right] 
\end{align*}
where $y_1,\ldots,y_T$ are i.i.d. Rademacher random variables. Since the lower bound holds for any $\Z$-valued tree $\z$ of depth $T$, it also holds for the supremum:
\begin{align*}
\Val^S_T(\ell,\Phi_T) &\ge \sup_{\z} \Eu{y_1,\ldots,y_T}\left[ \sup_{f\in\F} \frac{1}{T}\sum_{t=1}^T y_t f(\z_t(y_{1:t-1}))  \right] = \Rad_T(\ell, \Phi) \ .
\end{align*}
Hence, the lower bound on the value of the supervised game is the sequential Rademacher complexity of $\F$.
\end{proof}

\paragraph{Lower Bounds for Online Convex Optimization}

We first provide a lower bound for a linear game. By Lemma~\ref{lem:lincvx}, this lower bound will also serve as a lower bound for a convex Lipschitz game. We remark that these lower bounds are not entirely new (see e.g. \cite{AbeAgaBarRak09,abernethy08optimal}), and we derive them here for the purposes of completeness, as well as to stress that they arise from an equalizing strategy. 

Suppose $\F$ is a unit ball in some norm $\|\cdot \|$ and $\X$ is a unit ball in the dual norm $\|\cdot\|_*$. The loss $\ell(f,x) = x(f) = \inner{f, x}$ and the set $\Phi$ is, again, a set of constant departure mappings. 
\begin{proposition}
	The value of the linear game defined above is lower bounded by sequential Rademacher complexity:
	$$\Val_T(\ell,\Phi_T) \geq \Rad_T(\ell, \Phi) .$$
	Hence, the value of the convex Lipschitz game (where $\X$ is the set of all $1$-Lipschitz convex functions on $\F$) is also lower bounded by the same quantity.
\end{proposition}
\begin{proof}
Similarly to the proof for the supervised game, 
observe that for any $x_1,\ldots,x_T$ and any $f_1,\ldots,f_T,f'_1,\ldots,f'_T$,
$$
\inf_{\phi \in \Phi} \frac{1}{T}\sum_{t=1}^T \inner{\phi(f_t), x_t}  = \inf_{\phi\in\Phi} \frac{1}{T}\sum_{t=1}^T  \inner{\phi(f'_t), x_t}
$$
because any $\phi\in\Phi$ is a constant mapping. Following Remark~\ref{rem:equalizer_simplification}, we only need to exhibit a strategy that equalizes the player's loss. To this end, fix an $\X$-valued tree $\x$ of depth $T$. Consider the adversary's strategy where at each step an $\epsilon_t$ is chosen uniformly at random from $\{\pm 1\}$ and $x_t = \epsilon_t \cdot \x(\epsilon_{1:t-1}) \in \X$.

By Remark~\ref{rem:equalizer_simplification}, it is enough to check
\begin{align*}
\Eu{f_1 \sim q^*_1} \Eu{\epsilon_1} \ldots \Eu{f_T \sim q^*_T}\Eu{\epsilon_T} \frac{1}{T}\sum_{t=1}^T \epsilon_t \inner{f_t, \x(\epsilon_{1:t-1})}  = \Eu{f_1 \sim \overline{q^*_1}} \Eu{\epsilon_1} \ldots \Eu{f_T \sim \overline{q^*_T}}\Eu{\epsilon_T} \frac{1}{T}\sum_{t=1}^T \epsilon_t \inner{f_t, \x(\epsilon_{1:t-1})} 
\end{align*}
for all strategies $\{q^*_t\}$ and $\{\overline{q^*_t}\}$ of the player. This equality is indeed true because both terms are identically zero. By Proposition~\ref{prop:equalizer}, for any $g\in\F$
\begin{align*}
\Val_T(\ell,\Phi_T) &\ge \Eu{\epsilon_1,\ldots,\epsilon_T} \left[ \frac{1}{T}\sum_{t=1}^T \epsilon_t \inner{g, \x(\epsilon_{1:t-1}) } - \inf_{f\in\F} \frac{1}{T}\sum_{t=1}^T \epsilon_t \inner{f, \x(\epsilon_{1:t-1})}  \right]\\
&= \Eu{\epsilon_1,\ldots,\epsilon_T} \sup_{f\in\F} \frac{1}{T}\sum_{t=1}^T \epsilon_t \inner{f, \x(\epsilon_{1:t-1})} \ .
\end{align*}
Since this holds for any $\X$-valued tree $\x$, we have proven the statement.
\end{proof}

\subsubsection{Internal and Swap Regret}

Assume the cardinality $N=|\F|$ is finite. For internal regret, $\Phi$ is the set of mappings $\{\phi_{f\to g}: \phi_{f\to g}(f)=g~~\mbox{and}~~\phi_{f\to g}(h)=h ~~\forall h\neq f, h\in\F \}$. For swap regret \cite{BluMan05, PLG}, $\Phi$ contains all $N^N$ functions from $\F$ to itself. It is easy to see that the finite class lemma (Lemma~\ref{lem:fin}) immediately recovers the $O(\sqrt{T\log N})$ bound for  internal and external regret and the $O(\sqrt{T N\log N})$ bound for the swap regret \cite{PLG}. 

Our general tools, however, allow us to go well beyond finite sets of departure mappings. In the following sections, we consider several examples of infinite classes of departure mappings which have been considered in the literature. In some of these cases, an explicit strategy requires computation of a fixed-point \cite{HazKal07,GorGreMar08}. Since we are not providing efficient algorithms in order to obtain bounds, we are able to get sharp results by directly focusing on the complexity of these infinite classes of departure mappings.

\subsubsection{Convergence to $\Phi$-correlated Equilibria}

A beautiful result of Foster and Vohra \cite{FosVoh97} shows that convergence to the set of correlated equilibria can be achieved if players follow {\em internal} regret minimization strategies. What is surprising, no coordination is required to achieve this goal. Stoltz and Lugosi \cite{StoLug07} extended this result to compact and convex sets of strategies in normed spaces. In this section we show that their results can be improved in certain situations.

Let us consider their setting in a bit more detail. Suppose there are $N$ players each playing in a strategy set $\F$. We could make the strategy set player dependent
but it only complicates notation. There is $N$ loss functions mapping a strategy profile $(f_1,\ldots,f_N)$ to $\{\ell_k(f_1,\ldots,f_N)\}_{k=1}^N$, the losses for each
of the $N$ players. Consider a set of departure mappings $\Phi \subseteq \{ \phi \::\: \F \to \F \}$.
A $\Phi$-correlated equilibrim is a distribution $\pi$ over strategy profiles such that if the player jointly play according to it, no player has an incentive to unilaterally 
transform its action using a mapping from $\Phi$. That is,
\[
	\forall k\in[N], \forall \phi \in \Phi,\quad\quad\quad
	\En_{(f_1,\ldots,f_N)\sim\pi} \left[ \ell_k(f_k,f_{-k}) \right]
	\le \En_{(f_1,\ldots,f_N)\sim\pi} \left[ \ell_k(\phi(f_k),f_{-k}) \right] \ .
\]
Theorem 18 in \cite{StoLug07} shows the following. If $\F$ is convex compact subset of a normed vector space, $\ell_k$'s are continuous and $\Phi$ is a separable subset of
$\mathcal{C}(\F)$\footnote{The set of continuous function on $\F$ equipped with the supremum norm}, then there exist regret minimizing algorithms such that, if every player
follows the algorithm then the sequence of empirical plays jointly converges to the set of $\Phi$-correlated equilibria.

Consider a particular player $k$. The regret minimizing algorithm for it is simply a $\tilde{\Phi}$-regret minimizing algorithm with $\ell(f,x) = x(f)$ where we
have identified the adversary set $\X$ with the class of functions $\{ f \mapsto \ell_k(f,g) \::\: g \in \F^{k-1} \}$,
where $g$ is a strategy profile over the remaining $k-1$ players. Examining Stoltz and Lugosi's proof reveals that
$\tilde{\Phi}$ is taken to be a dense countable subset of $\Phi$ and an explicit regret minimizing algorithm for countably infinite classes of departure mappings is used.
The regret w.r.t. each $\phi \in \Phi$ does go to zero but the rate is not uniform in $\phi$. In particular, it depends on the order in which the class $\tilde{\Phi}$ is
enumerated. Later, they also consider examples of uncountable classes $\Phi$ of departure mapping where non-asymptotic rates of convergence for $\Phi$-regret can be obtained.
Specifically, they use the metric entropy of $\Phi$. We show how to improve their bounds
using sequential complexity.

As an example, consider the case where $\F$ is some compact subset of the unit ball in some normed space with a norm $\|\cdot\|$, the loss function $\ell_k$ is a 1-Lipschitz
convex function, and the class $\Phi$ of departure functions has finite metric entropy $\metricent(\Phi,\alpha)$ for all $\alpha > 0$.
Metric entropy is simply the log covering number where covers of $\Phi$ are built for the supremum norm $\|\phi\|_\infty = \sup_{f \in \F} \|\phi(f)\|$.
Let us consider a typical situation where $\metricent(\Phi,\alpha) = \Theta(1/\alpha^p)$.
To upper bound the $\Phi$-regret we can always make the set of adversary's moves larger. In fact, we make set $\X = \cC_\F$, where
\[
	\cC_\F = \{ x:\F \to \reals \::\: x \text{ convex and $1$-Lipschitz} \} \ .
\]
Moreover, by Lemma~\ref{lem:lincvx}, we have $\Val_T(\cC_\F,\F,\Phi) = \Val_T(\cL_\F,\F,\Phi)$ where
\[
	\cL_\F = \{ x:\F \to \reals \::\: x \text{ linear and $1$-Lipschitz} \} \ .
\]
Then the sequential complexity bound is
\begin{align}	
	\sup_{(\f,\x)} \En_{\epsilon_{1:T}} \sup_{\phi\in \Phi} \frac{1}{T} \sum_{t=1}^T \epsilon_t \inner{\phi( \f_t(\epsilon) ), \x_t(\epsilon)} \ .
\end{align}
Note that the set $\X$ is now just the set of $1$-Lipschitz linear functions, i.e. elements in the unit ball of the dual space.
Since $\|\phi_1 - \phi_2\|_\infty \le \alpha$ implies
\[
	\left| \inner{ \phi_1(f),x } - \inner{ \phi_2(f), x } \right| \le \alpha
\]
for any $x\in\X$, we can use metric entropy inside Dudley's integral to upper bound the sequential complexity by
\[
	c \inf_{\alpha} \left( \alpha T + \sqrt{T} \int_{\alpha'=\alpha}^{1} \sqrt{ \frac{1}{\alpha'^p} }d\alpha' \right) \ .
\]
This bound behaves as $O(\sqrt{T})$, if $p < 2$, as $O(\sqrt{T \log(T)})$ if $p=2$, and as $O(T^{(p-1)/p})$ if $p > 2$. These are better than the general bound
of $O(T^{(p+1)/(p+2)})$ given in Example 23 of \cite{StoLug07}.

\subsubsection{Linear Transformations}

In this section we consider the following scenario, discussed in \cite{GorGreMar08}.
Suppose $\F$ is a subset of a Hilbert space $\mathcal{M}$. Let $\Phi$ be the set of Lipschitz linear transformations on $\F$, i.e. $\Phi = \{ M \in \F \to \F \::\: \| M \| \le R \}$
for some operator norm $\|\cdot\|$. Let $\|\cdot\|_*$ be dual to $\|\cdot\|$. 
We are assuming the Online Convex Optimization scenario, i.e. $\X$ is a set of $L$-Lipschitz real-valued convex functions on $\F$ and the loss is defined as $\ell(f,x) = x(f)$. Furthermore, $$\ell_{\phi_M}(f,x) = x(Mf).$$ 
Therefore, we are in the setting of the well-studied online convex optimization (possibly in an infinite-dimensional Hilbert space), yet instead of being compared to the value of the best fixed point $f^*$ in hindsight, the player is being evaluated according to the best linear transformation of his trajectory $f_1,\ldots,f_T$. Is this problem learnable?

By Lemma~\ref{lem:lincvx}, the value of the convex game is equal to the value of the associated linear game. Suppose functions $x\in\X$ have gradients bounded by $L$ in the $\ell_2$ norm. 
The value of the convex game is upper bounded by the sequential complexity of the class of linear payoffs $\ell^{\text{lin}}(f,\tilde{x}) = \inner{f,\tilde{x}}$. Then the sequential complexity bound is
\begin{align}	
	\sup_{(\f,\x)} \En_{\epsilon_{1:T}} \sup_{M\in \Phi} \frac{1}{T} \sum_{t=1}^T \epsilon_t \inner{M \f_t(\epsilon), \x_t(\epsilon)} \ ,
\end{align}
which can be upper bounded by $R\cdot L\cdot\text{diam}_2(\F)$. Note that these results hold in infinite-dimensional Hilbert spaces, where a metric entropy-type cover of $\F$ would not even be finite.

%% file: example_blackwell.tex
\subsection{Blackwell's Approachability}
\label{sec:blackwell}

Blackwell's Approachability Theorem \cite{Blackwell56, Lehrer03, PLG} is a fundamental result for repeated two-player zero-sum games. By means of this Theorem, learnability (Hannan consistency) can be established for a wide array of problems, as illustrated in \cite{PLG}. For instance, existence of calibrated forecasters can be deduced from Blackwell's Approachability Theorem \cite{ManSto09, FosVoh97}. 

Let us first discuss the relation of our results to Blackwell's Theorem. A proof of Blackwell's Theorem (see for instance \cite{PLG}) reveals that (a) martingale convergence has to take place in the payoff space, and (b) the so-called Blackwell's one-shot approachability condition has to be satisfied. The former is closely related to the first term in our Triplex Inequality, while the latter is related to the second term (ability to play well if the next move is known). What is interesting, in the literature, Blackwell's Theorem is applied by embedding the problem at hand into an often high-dimensional space. The dimensionality represents the complexity of the problem, but this embedding is often artificial. In contrast, the problem complexity is captured by the third term of our decomposition, the \emph{sequential complexity}, and it is explicitly written as a complexity measure rather than an embedding into some other space. The ability to upper bound problem complexity with tools similar to those developed in \cite{RakSriTew10} (e.g. covering numbers) means that learnability can be established for a wide class of problems.

In this section we show that Blackwell's approachability can be viewed as an online game with a particular performance measure (distance to the set). Using the techniques developed in this paper, we prove Blackwell's approachability in Banach spaces for which martingale convergence holds (Theorem~\ref{thm:blackwell_upper}). We also show that martingale convergence is necessary for the result to hold (Theorem~\ref{thm:blackwell_lower}). To the best of our knowledge, both of these results are novel.

To define the problem precisely, suppose $\cH$ a subset of a Banach space $\mathcal{B}$ and $S\subset \mathcal{B}$ is a closed convex set. For the moves $f\in\F$ of the player and $x\in\X$ of the adversary, $\ell(f,x)\in\cH$ is a Banach space valued signal. The goal of the player is to keep the average of the signals $\frac{1}{T}\sum_{t=1}^T \ell(f_t,x_t)$ close to the set $S$. To view this problem as an instance of our general framework, define 
	$$\Compare(z_1,\ldots,z_T) = \inf_{c \in S}\left\|\frac{1}{T}\sum_{t=1}^T z_t - c\right\|.$$
The comparator term is zero by our assumption that $\Phi_T$ contain sequences $(\phi_1, \ldots,\phi_T)$ of constant mappings which transform our actions to a point inside $S$:  $\ell_{\phi_t}(f,x) = c_t\in S$ for all $f\in\F$, $x\in\X$, and $1\leq t\leq T$. Thus, indeed, the performance measure is 
	$$ \Reg_T = \inf_{c\in S}\left\| \frac{1}{T}\sum_{t=1}^T \ell(f_t, x_t) - c \right\|,$$
	the distance to the set $S$. The next condition on the payoff $\ell$ says that it must that the player can choose a ``good'' mixed strategy $q$ in response to a given mixed strategy $p$ of the adversary. This strategy $q$ should, on average, put the payoff inside the set $S$. Recall that $\ell(q,p)$ is simply a short-hand for the expected payoff $\En_{f\sim q,x\sim p}\ell(f,x)$ (that is, we do not make any assumptions about linearity of $\ell$).
	
\begin{definition}
Given a set $S$, the Blackwell's approachability game is said to be one shot approachable if for every mixed strategy $p$ of the adversary, there exists a mixed strategy $q$ for a player such that $\ell(q,p) \in S$.
\end{definition}

Blackwell's one-shot approachability condition is akin the second term in the Triplex Inequality, where the order of who plays first is switched. If the one-shot condition is satisfied, it remains to check martingale convergence.

\begin{definition}
	We will say that {\em martingale convergence holds} if 
	$$
	\lim_{T\to \infty}\sup_{\mbf{M}} \E{\left\|\frac{1}{T} \sum_{t=1}^T d_t \right\|} = 0,
	$$
	where the supremum is over distributions $\mbf{M}$ of martingale difference sequences $\{d_t\}_{t \in \mbb{N}}$ such that each $d_t \in \conv(\cH\ \bigcup -\cH)$ .
\end{definition}

We now show that, under the one-shot approachability condition, the set is approachable whenever martingale convergence holds in the subset of the Banach space. 

\begin{theorem}
	\label{thm:blackwell_upper}
For any game that is one shot approachable, we have that
$$
\Val_T(\ell, \Phi_T) \le 4 \sup_{\mbf{M}} \E{\left\|\frac{1}{T} \sum_{t=1}^T d_t \right\|}
$$
where the supremum is over distributions $\mbf{M}$ of martingale difference sequences $\{d_t\}_{t \in \mbb{N}}$ such that each $d_t \in \conv(\cH\ \bigcup -\cH)$.
\end{theorem}
\begin{proof}

Now we apply Theorem \ref{thm:main} to the Blackwell Approachability game.
Note that for any sequence $(\phi_1,\ldots,\phi_T)$, $\phi_t$ maps the payoff to some element of $S$. Hence,
$\Compare(\ell_{\phi_1}(f_1,x_1),\ldots,\ell_{\phi_T}(f_T,x_T)) = 0$ for any $f_1,\ldots,f_T\in \F$, $x_1,\ldots,x_T\in \X$. We then conclude that
\begin{align}
	\label{eq:blackwell_triplex}
\Val_T(\ell,\Phi_T) & \leq \sup_{p_1, q_1} \Eunder{x_1 \sim p_1}{f_1 \sim q_1} \ldots  \sup_{p_T,q_T} \Eunder{x_T \sim p_T}{f_T \sim q_T} \Big\{ \Compare(\ell(f_1,x_1), \ldots, \ell(f_T, x_T)) - \Eunder{x'_{1:T} \sim p_{1:T} }{ f'_{1:T} \sim q_{1:T}} \Compare(\ell(f'_1,x'_1), \ldots, \ell(f'_T, x'_T))  \Big\} \\
&~~~ +\sup_{p_1} \inf_{q_1}  \ldots  \sup_{p_T} \inf_{q_T} \Eunder{x_{1:T} \sim p_{1:T}}{f_{1:T} \sim q_{1:T}} \Compare(\ell(f_1,x_1), \ldots, \ell(f_T, x_T)) \nonumber\ .
\end{align}
We remark for the upper bound to hold it is enough to assume that $\Phi_T$ contains \emph{some} sequence that maps the payoffs to some element of $S$.

Consider the two terms in the above bound separately. The first term can be written as
\begin{align*}
&\sup_{p_1, q_1} \Eunder{x_1 \sim p_1}{f_1 \sim q_1} \ldots  \sup_{p_T,q_T} \Eunder{x_T \sim p_T}{f_T \sim q_T} \Eunder{x'_{1:T} \sim p_{1:T} }{ f'_{1:T} \sim q_{1:T}} \left\{ \inf_{c\in S} \left\|c - \frac{1}{T}\sum_{t=1}^T \ell(f_t,x_t) \right\| -  \inf_{c'\in S} \left\|c' - \frac{1}{T}\sum_{t=1}^T \ell(f'_t,x'_t) \right\| \right\} \\
&\leq \sup_{p_1, q_1} \Eunder{x_1 \sim p_1}{f_1 \sim q_1} \ldots  \sup_{p_T,q_T} \Eunder{x_T \sim p_T}{f_T \sim q_T} \Eunder{x'_{1:T} \sim p_{1:T} }{ f'_{1:T} \sim q_{1:T}} \left\{  \left\|\frac{1}{T}\sum_{t=1}^T \ell(f_t,x_t) - \frac{1}{T}\sum_{t=1}^T \ell(f'_t,x'_t) \right\| \right\} \\
&\leq \sup_{p_1, q_1} \Eunder{x_1,x'_1 \sim p_1}{f_1,f'_1 \sim q_1} \ldots  \sup_{p_T,q_T} \Eunder{x_T,x'_T \sim p_T}{f_T,f'_T \sim q_T} \left\{  \left\|\frac{1}{T}\sum_{t=1}^T \ell(f_t,x_t) - \frac{1}{T}\sum_{t=1}^T \ell(f'_t,x'_t) \right\| \right\} 
\end{align*}
where in the first inequality we used $\inf_a [C_1(a)] - \inf_a[C_2(a)] \leq \sup_{a}  [C_1(a)-C_2(a)]$ along with a triangle inequality. This is now bounded by
$$2 \sup_{\mbf{M}} \E{\left\|\frac{1}{T} \sum_{t=1}^T d_t \right\|}$$
where the supremum is over distributions $\mbf{M}$ of martingale difference sequences $\{d_t\}_{t \in \mbb{N}}$ such that each $d_t \in \conv(\cH\ \bigcup -\cH)$.

The second term in Eq.~\eqref{eq:blackwell_triplex} is 
\begin{align}
	\label{eq:blackwell_second_term}
&\sup_{p_1} \inf_{q_1}  \ldots  \sup_{p_T} \inf_{q_T} \Eunder{x_{1:T} \sim p_{1:T}}{f_{1:T} \sim q_{1:T}} \Compare(\ell(f_1,x_1), \ldots, \ell(f_T, x_T)) \nonumber\\
&=\sup_{p_1} \inf_{q_1}  \ldots  \sup_{p_T} \inf_{q_T} \Eunder{x_{1:T} \sim p_{1:T}}{f_{1:T} \sim q_{1:T}} \inf_{c\in S} \left\|c - \frac{1}{T}\sum_{t=1}^T \ell(f_t,x_t) \right\| \nonumber\\
&\leq \sup_{p_1} \inf_{q_1}  \ldots  \sup_{p_T} \inf_{q_T} \Eunder{x_{1:T} \sim p_{1:T}}{f_{1:T} \sim q_{1:T}} \inf_{c\in S} \left\{ 
	\left\|c - \frac{1}{T}\sum_{t=1}^T \ell(q_t,p_t) \right\| + \left\|\frac{1}{T}\sum_{t=1}^T \ell(q_t,p_t)  - \frac{1}{T}\sum_{t=1}^T \ell(f_t,x_t) \right\| 
	\right\} \nonumber\\
&\leq \sup_{p_1} \inf_{q_1}  \ldots  \sup_{p_T} \inf_{q_T} \left\{ 
	\inf_{c\in S} \left\|c - \frac{1}{T}\sum_{t=1}^T \ell(q_t,p_t) \right\| + \Eunder{x_{1:T} \sim p_{1:T}}{f_{1:T} \sim q_{1:T}} \left\|\frac{1}{T}\sum_{t=1}^T \ell(q_t,p_t)  - \frac{1}{T}\sum_{t=1}^T \ell(f_t,x_t) \right\| 
	\right\} \nonumber\\	
&\leq 	\sup_{p_1} \inf_{q_1}  \ldots  \sup_{p_T} \inf_{q_T} \left\{ 
		\inf_{c\in S} \left\|c - \frac{1}{T}\sum_{t=1}^T \ell(q_t,p_t) \right\| \right\} \\
	&~~~~~+ \sup_{p_1,q_1}  \ldots  \sup_{p_T,q_T} \Eunder{x_{1:T} \sim p_{1:T}}{f_{1:T} \sim q_{1:T}}  \left\|\frac{1}{T}\sum_{t=1}^T \ell(q_t,p_t)  - \frac{1}{T}\sum_{t=1}^T \ell(f_t,x_t) \right\| \nonumber
\end{align}
where the last inequality uses the fact that supremum is convex and infimum satisfies the following property: $\inf_a \left[C_1(a)+C_2(a)\right] \leq \left[\inf_a C_1(a)\right] + \left[\sup_a C_2(a)\right]$. By one shot approachability assumption, we can choose a particular response $q_t$ (in the first term of Eq.~\eqref{eq:blackwell_second_term}) for a given $p_t$ to be the mixed strategy that satisfies $\ell(q_t,p_t) \in S$. Since $S$ is a convex set, we conclude that 
$$
\frac{1}{T}\sum_{t=1}^T \ell(q_t,p_t) \in S
$$
and the first term in Eq.~\eqref{eq:blackwell_second_term} is zero. The second term is trivially upper bounded as 
\begin{align*}
	&\sup_{p_1,q_1}  \ldots  \sup_{p_T,q_T} \Eunder{x_{1:T} \sim p_{1:T}}{f_{1:T} \sim q_{1:T}}  \left\|\frac{1}{T}\sum_{t=1}^T \ell(q_t,p_t)  - \frac{1}{T}\sum_{t=1}^T \ell(f_t,x_t) \right\| \\
	&\leq \sup_{p_1,q_1} \Eunder{x_1\sim p_1}{f_1\sim q_1} \ldots  \sup_{p_T,q_T} \Eunder{x_T\sim p_T}{f_T\sim q_T}  \left\|\frac{1}{T}\sum_{t=1}^T \ell(q_t,p_t)  - \frac{1}{T}\sum_{t=1}^T \ell(f_t,x_t) \right\|\\
	& \leq 2 \sup_{\mbf{M}} \E{\left\|\frac{1}{T} \sum_{t=1}^T d_t \right\|} \ .
\end{align*}
Combining the two upper bounds yields the desired result.
\end{proof}

We now discuss lower bounds on the value of Blackwell's approachability game. The first lower bound is straightforward. 

\begin{proposition}
	Suppose martingale convergence holds. For any Blackwell's approachability game to have vanishing regret, one shot approachability for the game is a necessary condition.
\end{proposition}

We now show that martingale convergence in the space of payoffs is necessary for Blackwell's approachability. To the best of our knowledge, this result has not appeared in the literature.

\begin{theorem}
	\label{thm:blackwell_lower}
For every symmetric convex set $\cH$ there exists a one shot approachable game with payoff's mapping to $\cH$ such that
$$
\Val_T(\ell,\Phi_T)  \ge \frac{1}{2} \sup_{\mbf{M}} \E{\left\|\frac{1}{T} \sum_{t=1}^T d_t \right\|}
$$
where the supremum is over distributions $\mbf{M}$ of martingale difference sequences $\{d_t\}_{t \in \mbb{N}}$ such that each $d_t \in \cH$.
\end{theorem}
\begin{proof}
Consider the game where adversary plays from set $\X = \cH$, the player plays from set $\F = \{\pm1\}$, and $S=\{0\}$. Suppose the payoff is given by $\ell(f,x) = f \cdot x$.
Now consider the adversary strategy where adversary fixes a $\cH$ valued tree $\x$ and at each time $t$ picks a random $\epsilon_t \in \{\pm 1\}$ and plays $x_t = \epsilon_t \x_t(f_1 \cdot \epsilon_1, \ldots, f_{t-1} \cdot \epsilon_{t-1})$ that is a random sign multiplied with the instance given by the path on the tree specified by $f_1 \cdot \epsilon_1,\ldots,f_{t-1} \cdot \epsilon_{t-1}$. Further note that since $\epsilon_t \in \{\pm 1\}$ are Rademacher random variables, we see that irrespective of choice of distribution from which $f_t$ is drawn, $f_t \cdot \epsilon_t$ is a Rademacher random variable conditioned on history. This shows that for the above prescribed adversary strategy, we have that for any $\X$ valued tree $\x$ and any two player strategies $\{q^*_t\}$ and $\{\overline{q^*_t}\}$ we have
\begin{align*}
& \Eunder{\epsilon_1 \sim \mathrm{Unif}\{\pm1\}}{f_1 \sim q^*_1} \ldots \Eunder{\epsilon_T \sim \mathrm{Unif}\{\pm1\}}{f_T \sim q^*_T} \left\| \frac{1}{T} \sum_{t=1}^T (f_t \cdot \epsilon_t) \x(f_1 \cdot \epsilon_1, \ldots, f_{t-1} \cdot \epsilon_{t-1})\right\| \\
& ~~~~~~~~~= \Eunder{\epsilon_1 \sim \mathrm{Unif}\{\pm1\}}{f_1 \sim q^*_1} \ldots \Eunder{\epsilon_T \sim \mathrm{Unif}\{\pm1\}}{f_T \sim \overline{q^*_T}} \left\| \frac{1}{T} \sum_{t=1}^T (f_t \cdot \epsilon_t) \x(f_1 \cdot \epsilon_1, \ldots, f_{t-1} \cdot \epsilon_{t-1})\right\|\\
& ~~~~~~~~~= \Eunder{\epsilon_1 \sim \mathrm{Unif}\{\pm1\}}{f_1 \sim q^*_1} \ldots \Eunder{\epsilon_{T-1} \sim \mathrm{Unif}\{\pm1\}}{f_{T-1} \sim \overline{q^*_{T-1}}}~ \Eunder{\epsilon_T \sim \mathrm{Unif}\{\pm1\}}{f_T \sim \overline{q^*_T}} \left\| \frac{1}{T} \sum_{t=1}^T (f_t \cdot \epsilon_t) \x(f_1 \cdot \epsilon_1, \ldots, f_{t-1} \cdot \epsilon_{t-1})\right\|\\
& ~~~~~~~~~\ldots~= \Eunder{\epsilon_1 \sim \mathrm{Unif}\{\pm1\}}{f_1 \sim \overline{q^*_1}} \ldots \Eunder{\epsilon_T \sim \mathrm{Unif}\{\pm1\}}{f_T \sim \overline{q^*_T}} \left\| \frac{1}{T} \sum_{t=1}^T (f_t \cdot \epsilon_t) \x(f_1 \cdot \epsilon_1, \ldots, f_{t-1} \cdot \epsilon_{t-1})\right\|
\end{align*}
The first equality above is due to the fact that $f_T \cdot \epsilon_T$ is a Rademacher random variable conditioned on $f_{1},\ldots,f_{T-1}$ and $\epsilon_1 ,\ldots,\epsilon_{T-1}$ which means we can replace $q^*_T$ with $\overline{q^*_T}$. The subsequent equalities are got similarly by replacing each $q^*_t$ by $\overline{q^*_t}$ one by one inside out by conditioning on $f_1,\ldots,f_{t-1}$ and $\epsilon_1,\ldots,\epsilon_{t-1}$; and replacing each $q^*_t$ by $\overline{q^*_t}$. Hence we see that the adversary strategy is an equalizer strategy. Hence
using Proposition \ref{prop:equalizer} and picking the fixed $f = 1$ we see that
\begin{align*}
\Val_T & \ge \sup_{\x} \Es{\epsilon \sim \mathrm{Unif}\{\pm1\}^{T}}{\left\| \frac{1}{T} \sum_{t=1}^T \epsilon_t \x(\epsilon)\right\|}  \ge \frac{1}{2} \sup_{\mbf{M}} \E{ \left\| \frac{1}{T} \sum_{t=1}^T d_t \right\|}
\end{align*}
where the last inequality is because the worst-case martingale difference sequence generated by random signs (Walsh Paley martingales)  are lower bounded by the worst case martingale difference sequences within a factor of at most two  \cite{Pisier75}.
\end{proof}

%% file: example_calibration.tex
\subsection{Calibration}
\label{sec:calibration}

Calibration, introduced by Brier~\cite{Brier50} and Dawid~\cite{Dawid82}, is an important notion for forecasting binary sequences. In the context of weather forecasting, calibration means that, for the days the forecaster announced ``30\% chance of rain'', the empirical frequency of rain should indeed be close to $30\%$ \cite[p. 85]{PLG}; moreover, this has to hold for any forecasted value. The existence of calibrated forecasters, a fact which is not obvious a priori, was shown by Foster and Vohra \cite{FosVoh98asymptotic}. Following \cite{PLG}, we consider the notion of $\lambda$-calibration. If a forecaster is $\lambda$-calibrated for all $\lambda>0$, we say that the forecaster is well calibrated. 

In what follows, we formulate the calibration problem of forecasting $\{1,\ldots,k\}$-valued sequences in our general framework. In particular, we are interested in sharp rates on the resulting value of the calibration game, and we will compare our results with the recent work of Mannor and Stoltz \cite{ManSto09}.

Fix a norm $\|\cdot\|$ on $\reals^k$. Let $\cH=\reals^k$, $\F = \Delta(k)$, and $\X$ the set of standard unit vectors in $\reals^k$ (vertices of $\Delta(k)$). Define $\ell(f,x) = 0$; that is, the forecaster is penalized only through the comparator term. We define $\Compare(z_1,\ldots,z_T) = -\left\|\frac{1}{T}\sum_{t=1}^T z_t \right\|$. Define $\Phi_T=\{(\phi_{p, \lambda},\ldots,\phi_{p,\lambda}): p\in\Delta(k), \lambda> 0\}$ to contain time-invariant mappings defined by $$\ell_{\phi_{p,\lambda}}(f,x) = \ind{\|f-p\|\leq \lambda}\cdot (f-x) \ .$$
This definition of the loss is indeed natural for the $\lambda$-calibration problem. It says that, for any $p$ chosen after the game, if we consider a round when the player predicted $f\in\Delta(k)$ close to $p$, the loss should be the difference between the actual outcome $x$ and $f$. Indeed, when we put all the definitions together, we obtain

$$\Reg_T = \sup_{\lambda> 0}\sup_{p\in\Delta(k)} \left\|\frac{1}{T}\sum_{t=1}^T \ind{\|f_t-p\|\leq \lambda} \cdot (f_t-x_t) \right\| \ .$$

Note that this notion of regret allows the worst scale $\lambda$ to be chosen at the end of the game. This makes it a stronger requirement than what is required for building a well calibrated forcaster. Nevertheless, we can bound the value of this game, improving on the results of Mannor and Stoltz \cite{ManSto09}. Theorem~\ref{thm:deltacalib} shows that the rate of calibration is $\tilde{O}(T^{-1/3})$ no matter what $k$ is. The rate of $\tilde{O}(T^{-1/3})$ has been established for $k=2$ previously. For $k>2$, however, the best rates known to us (due to \cite{ManSto09}) deteriorate with $k$. Let us remark that some looseness of the approach of \cite{ManSto09} comes from discretization in order to phrase the problem as Blackwell's approachability. A reader will note that we also pass to a discretization in the proof below. However, this is done late in the analysis in order to upper bound the sequential complexity. This seems to speak in favor of our approach, aimed at directly looking at the complexity of the problem through the notion of sequential complexity. 

\begin{theorem}\label{thm:deltacalib}
For the calibration game with $k$ outcomes and with $\ell_1$ norm, we have that for $T\geq 3$ and some absolute constant $c$
$$
\Val_T (\ell, \Phi_T)\le c k^{2} \left(\frac{\log T}{T}\right)^{1/2}  \ .
$$
\end{theorem}
\begin{proof}
Let $\delta>0$ to be determined later. Let $\|\cdot\|$ denote the $\ell_1$ norm. Let $C_{\delta}$ be the maximal $2 \delta$-packing of $\Delta(\X)$ in this norm. Consider the calibration game defined in Example~\ref{eg:calibration2}, augmented with the restriction that the player's choice belongs to $C_\delta$ instead of $\Delta(k)$. The corresponding minimax expression with this restriction is clearly an upper bound on the value of the game defined in Example~\ref{eg:calibration2}. 

Observe that the first term in the Triplex Inequality of Theorem~\ref{thm:main} is zero. The second term is upper bounded by a particular (sub)optimal response $q_t$ being the point mass on $p^\delta_t$, the element of $C_\delta$ closest to $p_t$. Note that any $2 \delta$ packing is also a $2 \delta$ cover. Thus, the second term becomes
\begin{align*}
	&\sup_{p_1}\inf_{q_1}\ldots  \sup_{p_T}\inf_{q_T} \sup_{\bphi\in\Phi_T} \left[ -\Eunder{f_{1:T} \sim q_{1:T}}{x_{1:T} \sim p_{1:T}}\Compare(\ell_{\phi_1}(f_1,x_1), \ldots, \ell_{\phi_T}(f_T, x_T)) \right]\\
	&= \sup_{p_1}\inf_{q_1}\ldots  \sup_{p_T}\inf_{q_T} \sup_{\lambda> 0}\sup_{p\in\Delta(k)} \Eunder{f_{1:T} \sim q_{1:T}}{x_{1:T} \sim p_{1:T}} \left\| \frac{1}{T}\sum_{t=1}^T \ell_{\phi_{p,\lambda}}(f_t,x_t) \right\| \\
	&\leq \sup_{p_1}\ldots  \sup_{p_T} \sup_{\lambda> 0}\sup_{p\in\Delta(k)} \Eunder{}{x_{1:T} \sim p_{1:T}} \left\| \frac{1}{T}\sum_{t=1}^T \ind{\|p^\delta_t-p\|\leq \lambda}\cdot (p^\delta_t-x_t) \right\| 
\end{align*}
which, in turn, is upper bounded via triangle inequality by
\begin{align*}
	&\sup_{p_1}\ldots  \sup_{p_T} \sup_{\lambda> 0}\sup_{p\in\Delta(k)} \Eunder{}{x_{1:T} \sim p_{1:T}} \left\| \frac{1}{T}\sum_{t=1}^T \ind{\|p^\delta_t-p\|\leq \lambda}\cdot (p^\delta_t-p_t) \right\| \\
	&~~~~~~~~~~~~~~~~~~~~~+ \sup_{p_1}\ldots  \sup_{p_T} \sup_{\lambda> 0}\sup_{p\in\Delta(k)} \Eunder{}{x_{1:T} \sim p_{1:T}} \left\| \frac{1}{T}\sum_{t=1}^T \ind{\|p^\delta_t-p\|\leq \lambda}\cdot (p_t-x_t) \right\| \\
	& \le 2 \delta + \sup_{p_1}\ldots  \sup_{p_T} \sup_{\lambda> 0}\sup_{p\in\Delta(k)} \Eunder{}{x_{1:T} \sim p_{1:T}} \left\| \frac{1}{T}\sum_{t=1}^T \ind{\|p^\delta_t-p\|\leq \lambda}\cdot (p_t-x_t) \right\| 
\end{align*}
Now note that for a given $\lambda>0$, $p_1,\ldots,p_T$ and $p \in \Delta(k)$, we have that  $ \{\ind{\|p^\delta_t-p\|\leq \lambda}\cdot (p_t-x_t) \}_{t \in \mbb{N}}$ is a martingale difference sequence and so the second term in the triplex inequality is bounded as :
\begin{align}\label{eq:2ndterm}
	&\sup_{p_1}\inf_{q_1}\ldots  \sup_{p_T}\inf_{q_T} \sup_{\bphi\in\Phi_T} \left[-\Eunder{f_{1:T} \sim q_{1:T}}{x_{1:T} \sim p_{1:T}}\Compare(\ell_{\phi_1}(f_1,x_1), \ldots, \ell_{\phi_T}(f_T, x_T)) \right] \le 2 \delta + 2\sqrt{\frac{k}{T}} \ .
\end{align}

We now proceed to upper bounded the third term in the Triplex Inequality. Since $-\Compare$ is a subadditive, by Theorem~\ref{thm:rad}, we have that the third term is bounded by twice the sequential complexity
\begin{align*}
2\Rad_T(\ell, \Phi_T, -\Compare) &= 2\sup_{\f, \x}\ \En_{\epsilon_{1:T}} \sup_{\bphi\in \Phi_T} -\Compare\Big(\epsilon_1 \ell_{\phi_1}(  \f_1(\epsilon),\x_1(\epsilon)), \ldots,  \epsilon_T \ell_{\phi_T}( \f_T(\epsilon), \x_T(\epsilon)) \Big) \\
	&= 2\sup_{\f, \x}\ \En_{\epsilon_{1:T}} \sup_{\lambda> 0}\sup_{p\in\Delta(k)} \left\| \frac{1}{T}\sum_{t=1}^T \epsilon_t \ind{\|\f_t(\epsilon)-p\|\leq \lambda}\cdot(\f_t(\epsilon)-\x_t(\epsilon)) \right\|
\end{align*}
where $\f$ is a $C_\delta$-valued tree. Using the fact that $\f$ is a discrete-valued tree, not a $\Delta(k)$-valued tree, we would like to pass from the supremum over $\lambda>0$ and $p\in\Delta(k)$ to a supremum over finite discrete set in order to appeal to Proposition~\ref{prop:fin_phi}. 

To this end, fix $\f,\x$ and $\epsilon_{1:T}$ and let us see how many genuinely different functions can we get by varying $\lambda>0$ and $p\in\Delta(k)$. This question boils down to looking at the size of the class $$\G := \left\{ g_{p,\lambda}(f) = \ind{\|f-p\|\leq \lambda}: p\in\Delta(k), \lambda > 0 \right\}$$
over the possible values of $f \in C_{\delta}$. Indeed, if
$g_{p,\lambda}(f) = g_{p',\lambda'}(f)$ for all $f\in C_\delta$, then 

$$\frac{1}{T}\sum_{t=1}^T\ind{\|\f_t(\epsilon)-p\|\leq \lambda}\cdot(\f_t(\epsilon)-\x_t(\epsilon)) = \frac{1}{T}\sum_{t=1}^T \ind{\|\f_t(\epsilon)-p'\|\leq \lambda'}\cdot(\f_t(\epsilon)-\x_t(\epsilon)).$$

We appeal to VC theory for bounding the size of $\G$ over $C_\delta$. First, we claim that the VC dimension of $\G$ is $O(k^2)$. Note that $\G$ is the class of indicators over $\ell_1$ balls of radius $\lambda$ centered at $p$ for various values of $p,\lambda$. A result of Goldberg and Jerrum \cite{GolJer95vc} states that for a class $\G$ of functions parametrized by a vector of length $d$, if for $g\in\G$ and $f\in\F$, $\ind{g(f)=1}$ can be computed using $m$ arithmetic operations, the VC dimension of $\G$ is $O(md)$. In our case, the functions in $\G$ are parametrized by $k$ values and membership $\|f-p\|_1\leq \lambda$ can be established in $O(k)$ operations. This yields $O(k^2)$ bound on the VC dimension of $\G$. By Sauer-Shelah Lemma, the number of different labelings of the set $C_\delta$ by $\G$ is bounded by $|C_\delta|^{c\cdot k^2}$ for some absolute constant $c$. We conclude that the effective number of different $(p,\lambda)$ is finite. Let us remark that the VC upper bound is \emph{not} used in place of the sequential Littlestone's dimension. It is only used to show that the set $\Phi_T$  is finite, and such technique can be useful when the set of player's actions is finite.

Hence, there exists a finite set $S$ of pairs $(\lambda,p)$ with cardinality $|S|\leq |C_\delta|^{c\cdot k^2}$ such that
\begin{align*}
2\Rad_T(\ell, \Phi_T, -\Compare) &\le 2\sup_{\f, \x}\ \En_{\epsilon_{1:T}} \sup_{\lambda>0}\sup_{p\in \Delta(k)} \left\| \frac{1}{T}\sum_{t=1}^T \epsilon_t \ind{\|\f_t(\epsilon)-p\|_1\leq \lambda}\cdot(\f_t(\epsilon)-\x_t(\epsilon)) \right\|_1\\
&=2\sup_{\f, \x}\ \En_{\epsilon_{1:T}} \max_{(p,\lambda) \in S} \left\| \frac{1}{T}\sum_{t=1}^T \epsilon_t \ind{\|\f_t(\epsilon)-p\|_1\leq \lambda}\cdot(\f_t(\epsilon)-\x_t(\epsilon)) \right\|_1\\
& \le  2\ k^{1/2} \sup_{\f, \x} \mathbb{E}_{\epsilon} \max_{(p,\lambda) \in S} \left\|\frac{1}{T} \sum_{t=1}^T \epsilon_t \ind{\|\f_t(\epsilon) - p \|_1 \le \lambda}  \cdot (\f_t(\epsilon)-\x_t(\epsilon)) \right\|_2  
\end{align*}
Now note that $\|\cdot\|_2^2$ is $(2,2)$-smooth and so applying Lemma~\ref{lem:fin_phi_2smooth_sum} with $G = \|\cdot\|_2$, $\gamma=2$, $\eta = 2$, we see that
\begin{align*}
 2\Rad_T(\ell, \Phi_T, -\Compare) & \le 2 k^{1/2} \left(\frac{ 8\log(2|S|)}{T}\right)^{1/2} \\
 & \le 2 k^{1/2} \left(\frac{ 16 c k^2 \log(|C_\delta|)}{T}\right)^{1/2} \\
  & = c' k^{3/2} \left(\frac{\log(|C_\delta|)}{T}\right)^{1/2} \ 
\end{align*}
for some small absolute constant $c'$.

Now note that the size of set $C_\delta$ the $2 \delta$ packing of $\Delta(k)$ is upper bounded by the size of the minimal $\delta$ cover of $\Delta(k)$
which can be bounded as $|C_\delta| \le \left(\frac{1}{\delta}\right)^{k-1}$ and so we see that
\begin{align*}
2\Rad_T(\ell, \Phi_T, -\Compare)  \le  c' k^{2} \left(\frac{\log(1/\delta)}{T}\right)^{1/2} \ . 
\end{align*}

Combining the above upper bound on the third term of triplex inequality and Equation \ref{eq:2ndterm} that bounds the second term of the triplex inequality (and since first term is anyway $0$) we see that,
\begin{align*}
\Val_T \le 2 \delta + 2 \sqrt{\frac{k}{T}} +  c' k^{2} \left(\frac{\log(1/\delta)}{T}\right)^{1/2}  \ .
\end{align*}
Choosing $\delta = 1/T$ concludes the proof.
\end{proof}

%% file: example_global.tex
\subsection{Other Examples}
\label{sec:global}

\newcommand\ave[1]{\underline{#1}}

\subsubsection{External Regret with Global Costs}

Let us consider a more general setting where the (vector) loss is $\ell(f,x)$ rather than the specific choice $f \odot x$ in Example \ref{eg:global}. The Triplex Inequality and Theorem \ref{thm:rad} then gives
\begin{align*}
	\Val_T &\leq \sup_{p_1, q_1}\Eunder{x_1\sim p_1}{f_1 \sim q_1} \ldots  \sup_{p_T, q_T} \Eunder{x_T\sim p_T}{f_T \sim q_T} \Eunder{x'_{1:T}\sim p_{1:T}}{f'_{1:T}\sim q_{1:T}}
	\hspace{0.4cm}  \left\| \frac{1}{T} \sum_{t=1}^T (\ell(f_t, x_t) - \ell(f'_t,x'_t)) \right\|  \\
	&+\sup_{p_1} \inf_{q_1} \ldots  \sup_{p_T} \inf_{q_T} 
	\sup_{f \in \F}  \Eunder{x_{1:T}\sim p_{1:T}}{f_{1:T}\sim q_{1:T}}\left\{ \left\| \frac{1}{T} \sum_{t=1}^T \ell(f_t,x_t) \right\| -  \left\| \frac{1}{T} \sum_{t=1}^T \ell(f,x_t) \right\| \right\} \\
	&+ 2 \sup_{\x}\ \En_{\epsilon_{1:T}} \sup_{f \in \F} \left\| \frac{1}{T} \sum_{t=1}^T \epsilon_t \ell(f,\x_t(\epsilon)) \right\| \ .
\end{align*}

Consider the first term in the Triplex Inequality. Observe that  $(\ell(f_t,x_t) - \ell(f'_t,x'_t))_{t=1}^T$ is a (vector valued) martingale difference sequence and so  
\begin{align*}
	\sup_{p_1, q_1}\Eunder{x_1,x'_1\sim p_1}{f_1,f'_1 \sim q_1} \ldots  \sup_{p_T, q_T} \Eunder{x_T,x'_T\sim p_T}{f_T,f'_T \sim q_T} 
	\hspace{0.4cm}  \left\| \frac{1}{T} \sum_{t=1}^T (\ell(f_t, x_t) - \ell(f'_t,x'_t)) \right\|
	\le 2 \sup_{\mbf{M}} \E{ \left\| \frac{1}{T} \sum_{t=1}^T d_t \right\|} \ .
\end{align*}
where the supremum is over distributions $\mbf{M}$ of martingale difference sequences $\{d_t\}_{t \in \mbb{N}}$ such that each $d_t \in \conv(\cH\ \bigcup -\cH)$.

Now, consider the second summand above:

\begin{align*}
	&\sup_{p_1} \inf_{q_1} \ldots  \sup_{p_T} \inf_{q_T} 
	\sup_{f \in \F}  \Eunder{x_{1:T}\sim p_{1:T}}{f_{1:T}\sim q_{1:T}}\left\{ \left\| \frac{1}{T} \sum_{t=1}^T \ell(f_t,x_t) \right\| -  \left\| \frac{1}{T} \sum_{t=1}^T \ell(f,x_t) \right\| \right\} \\
	&= \sup_{p_1} \inf_{q_1} \ldots  \sup_{p_T} \inf_{q_T} 
	\left\{ \Eunder{x_{1:T}\sim p_{1:T}}{f_{1:T}\sim q_{1:T}} \left\| \frac{1}{T} \sum_{t=1}^T \ell(f_t,x_t) \right\| -  \inf_{f \in \F}  \Eunder{}{x_{1:T}\sim p_{1:T}}\left\| \frac{1}{T} \sum_{t=1}^T \ell(f,x_t) \right\| \right\}\\
	&\leq \sup_{p_1} \ldots  \sup_{p_T} 
	\left\{ \Eunder{}{x_{1:T}\sim p_{1:T}} \left\| \frac{1}{T} \sum_{t=1}^T \ell(f_t,x_t) \right\| -  \inf_{f \in \F}  \Eunder{}{x_{1:T}\sim p_{1:T}}\left\| \frac{1}{T} \sum_{t=1}^T \ell(f,x_t) \right\| \right\}
\end{align*}
where in the last step a (sub)optimal choice was made for $q_t$: the distribution $q_t = \delta_{f_t}$ puts all the mass on $f_t$ such that
$$\|\ell(f_t, p_t)\| = \inf_{f\in\F} \|\ell(f, p_t)\|.$$
Observe that by several applications of triangle and Jensen's inequalities,
\begin{align}
	\label{eq:global_cost_2nd_term}
&\Eunder{}{x_{1:T}\sim p_{1:T}} \left\| \frac{1}{T} \sum_{t=1}^T \ell(f_t,x_t) \right\| -  \inf_{f \in \F}  \Eunder{}{x_{1:T}\sim p_{1:T}}\left\| \frac{1}{T} \sum_{t=1}^T \ell(f,x_t) \right\| \nonumber\\
&~~~~~\leq \left\{\left\| \frac{1}{T} \sum_{t=1}^T \ell(f_t,p_t) \right\| -  \inf_{f \in \F}  \left\| \frac{1}{T} \sum_{t=1}^T \ell(f,p_t) \right\|\right\} +\Eunder{}{x_{1:T}\sim p_{1:T}} \left\| \frac{1}{T} \sum_{t=1}^T (\ell(f_t,x_t)-  \ell(f_t,p_t)) \right\| 
\end{align}

Now we make an important assumption.
\begin{assumption}
\label{asmp:global}
Suppose that, for any $p_1, p_2$,
\[
	\inf_{f} \left\| \ell(f,p_1) + \ell(f,p_2) \right\|
	\ge \inf_{f} \left\| \ell(f,p_1) \right\| + \inf_{f} \left\| \ell(f,p_2) \right\| \ .
\]
\end{assumption}
Under~\assumpref{asmp:global}, along with the way we chose $f_t$, the first term in \eqref{eq:global_cost_2nd_term} becomes
\begin{align*}
	&\left\| \frac{1}{T} \sum_{t=1}^T \ell(f_t,p_t) \right\| -  \inf_{f \in \F}  \left\| \frac{1}{T} \sum_{t=1}^T \ell(f,p_t) \right\| 
	\leq 
	\frac{1}{T} \sum_{t=1}^T \left\| \ell(f_t,p_t) \right\| -  \frac{1}{T} \sum_{t=1}^T \inf_{f \in \F}  \left\| \ell(f,p_t) \right\| = 0 \ .
\end{align*}
We conclude that the second term in the Triplex Inequality can be upper bounded by 
$$\sup_{p_1} \ldots  \sup_{p_T} \Eunder{}{x_{1:T}\sim p_{1:T}} \left\| \frac{1}{T} \sum_{t=1}^T (\ell(f_t,x_t)-  \ell(f_t,p_t)) \right\|, $$
which, in turn, is no worse than the supremum over distributions $\mbf{M}$ of martingale difference sequences used to bound the first term.

This gives us the general upper bound on the value of the game:
\begin{align}
\label{eq:genglobal}
	\Val_T \ \leq\ 4 \sup_{\mbf{M}} \E{ \left\| \frac{1}{T} \sum_{t=1}^T d_t \right\|} \ +
	2\ \sup_{\x}\ \En_{\epsilon_{1:T}} \sup_{f \in \F} \left\| \frac{1}{T} \sum_{t=1}^T \epsilon_t \ell(f,\x_t(\epsilon)) \right\| \ .
\end{align}
Let us see what this implies in a specific case of interest. 

\paragraph{Global Cost Learning on the Simplex}

Here we consider Example~\ref{eg:global}, the setting studied in Even-Dar et al \cite{EveKleManMan09}. Let $\F = \Delta(k)$, $\X = [0,1]^k$ and $\ell(f,x) = f \odot x$.
Let us first verify if~\assumpref{asmp:global} holds here.
By linearity of the vector loss, we just have to verify whether, for arbitrary $p_1, p_2$, we have
\[
	\inf_{q \in \Delta(k)} \left\| q \odot \ave{p_1} + q \odot \ave{p_2} \right\|
	\geq
	\inf_{q \in \Delta(k)} \left\| q \odot \ave{p_1} \right\| +
	\inf_{q \in \Delta(k)} \left\| q \odot \ave{p_2} \right\| \ .
\]
where the notation $\ave{p_i}$ stands for the mean of the distribution $p_i$. This is equivalent to asking whether the function
\[
	x \mapsto \inf_{f \in \F} \left\| f \odot x \right\|
\]
is {\em concave}. \lemref{lem:concave} in the appendix proves that it is.
Note that in~\cite{EveKleManMan09}, it is shown that the above function is concave for
the $\ell_p$ norms (including $p=\infty$). It turns out that it remains concave no matter what norm is chosen.
Thus, the general upper bound~\eqref{eq:genglobal} holds.
In the case we are considering, we can further massage the second term in that upper bound.
Note that for any $f$ and $y$, $\| f \odot y\| \le \|f\|_\infty \|y\| \le \|y\|$. Hence, we have
\begin{align*}
\sup_{f \in \F} \left\| \frac{1}{T} \sum_{t=1}^T \epsilon_t (f \odot \x_t(\epsilon)) \right\| & = \sup_{f \in \F} \left\|f \odot \left( \frac{1}{T} \sum_{t=1}^T \epsilon_t  \x_t(\epsilon) \right) \right\|\\
& \le  \left\|\frac{1}{T} \sum_{t=1}^T \epsilon_t  \x_t(\epsilon) \right\|
\end{align*}
Hence using the above in \eqref{eq:genglobal} we see that
\begin{align*}
	\Val_T \ & \leq\ 4 \sup_{\mbf{M}} \E{ \left\| \frac{1}{T} \sum_{t=1}^T d_t \right\|} + 2\ \sup_{\x}\ \En_{\epsilon_{1:T}} \left\|\frac{1}{T} \sum_{t=1}^T \epsilon_t  \x_t(\epsilon) \right\| \\
	& \le\ 6\ \sup_{\mbf{M}} \E{ \left\| \frac{1}{T} \sum_{t=1}^T d_t \right\|}
\end{align*}
where the last inequality is because $( \epsilon_t \x_t(\epsilon))_{t=1}^T$ is a martingale difference sequence. In the last inequality the supremum is over distributions $\mbf{M}$ of martingale difference sequences $\{d_t\}_{t \in \mbb{N}}$ such that each $d_t \in [-1,1]^k$.

\subsubsection{Adaptive Regret}
To study online learning in changing environment Hazan and Seshadhri defined the notion of \emph{adaptive regret} in \cite{HazanSe09}. The notion of adaptive regret introduced in \cite{HazanSe09} was mainly one where cumulative loss for any time interval is compared to the best predictor at hindsight for that particular interval. We first extend the notion of adaptive regret in \cite{HazanSe09} to include departure mappings as,
\begin{align}\label{eq:adaptreg}
\Reg_T := \sup_{[r,s] \subseteq [T]} \left\{\frac{1}{T}\sum_{t=r}^s \loss(f_t,x_t) - \inf_{\psi \in \Psi} \frac{1}{T}\sum_{t=r}^s \loss(\psi \circ f_t,x_t)\right\}
\end{align}
 where $\loss : \F \times \X \mapsto [0,1]$ is some arbitrary loss function and $\Psi$ is some class of departure mappings. The key idea in the above definition of regret is that we consider the worst time interval and consider the regret for that time interval versus some fixed set of departure mappings.

We capture the above notion of regret in our framework by defining :
\begin{itemize}
\item $\ell(f,x) = 0$ for all $f \in \F$ and $x \in \X$
\item  Define the set of time-invariant payoff transformations $\Phi_T = \mc{I}_T \times \Psi_T$ where $\Psi_T = \{(\psi,\ldots,\psi) : \psi \in \Psi \}$ and $\Psi$ is some class of departure mappings and $\mc{I}_T = \{([r,s],\ldots,[r,s]) : [r,s] \subseteq [T]\}$, the set of all intervals in $[T]$ repeated $T$ times. 
\item For each $t \in [T]$ and $\phi_t = (I_t,\psi_t)$, define
$
\ell_{\phi_t}(f,x) =  \left(- \loss(f,x) + \loss(\psi_t \circ f,x) \right) \ind{t \in I_t}
$ 
\item $\Compare(z_1, \ldots,z_T) = \sum_{t=1}^T z_t / T$
\end{itemize}

Note that 
\begin{align*}
\Reg_T &= \sup_{[r,s] \subseteq [T]} \left\{\frac{1}{T}\sum_{t=r}^s \loss(f_t,x_t) - \inf_{\psi \in \Psi} \frac{1}{T} \sum_{t=r}^s \loss(\psi \circ f_t,x_t)\right\}\\
& = \sup_{I \in \mc{I}_T , \psi \in \Psi_T}  \left\{\frac{1}{T}\sum_{t=1}^T \loss(f_t,x_t) \ind{t \in I_t} - \frac{1}{T} \sum_{t=1}^T \loss(\psi_t \circ f_t,x_t) \ind{t \in I_t}\right\}\\
& = \Compare(\ell(f_1,x_1),\ldots,\ell(f_T,x_T)) - \inf_{\phi \in \Phi_T} \Compare(\ell_{\phi_1}(f_1,x_1),\ldots,\ell_{\phi_T}(f_T,x_T))
\end{align*}
and thus we see that the adaptive regret defined in Equation \eqref{eq:adaptreg} falls under our general framework. We would like to point out as an example that if we take $\Psi_T = \{(f,\ldots,f) : f \in \F\}$ the time invariant set of constant mappings then the regret defined in Equation \eqref{eq:adaptreg} is identical to the one in \cite{HazanSe09}. Below we show a bound on the value of the game with adaptive regret in terms of covering number of the departure mapping class.

\begin{theorem}
\label{thm:adaptive}
For the adaptive regret game we have that
\begin{align}
\Val_T \le 8 \inf_{\alpha > 0 }\left\{\alpha + 6\sqrt{2} \int_{\alpha}^2 \sqrt{\frac{\log\ \mc{N}_\infty(\delta,\Psi,T)}{T}} d\delta \right\} + 96 \sqrt{\frac{ \log\ T}{T}} \label{eq:covadaptreg}
\end{align}
\end{theorem}

%% file: highprob.tex
\section{High Probability Bounds}
\label{sec:highprob}

The definition of value of the game provided in Equation \eqref{eq:def_val_game} only guarantees existence of a randomized algorithm which in expectation over its randomization achieves regret bounded by the value. Even with Markov inequality this is not sufficient to prove almost sure convergence but only convergence in expectation (or probability). We now define for any $\theta > 0$ an alternative notion of a value of the game $\Val^\theta_T(\ell,\Phi_T)$. It guarantees existence of a randomized online learning algorithm which in $T$ rounds achieves regret smaller than $\theta$ with probability at least $1 - \Val^\theta_T(\ell,\Phi_T)$ over its randomization. Using this value we are able to prove almost sure convergence for many games.

\begin{definition}
	\label{def:theta-value}
For any $\theta > 0$ define the value of the game as 
{\small
	\begin{align}
	\Val^\theta_T(\ell,\Phi_T) &= \inf_{q_1} \sup_{x_1} \Eu{f_1\sim q_1} \ldots \inf_{q_T} \sup_{x_T} \Eu{f_T\sim q_T} \ind{\sup_{\bphi\in \Phi_T}\left\{ \Compare(\ell(f_1,x_1), \ldots, \ell(f_T, x_T)) -  \Compare(\ell_{\phi_1}(  f_1,x_1), \ldots, \ell_{\phi_T}(f_T, x_T))\right\} > \theta} 
	\end{align}
}
\end{definition}

It is natural to think of the sequence of infima, suprema, and expectations as a stochastic process which generates $f_t$'s and $x_t$'s. The ``in-expectation'' version of the value of the game, defined in \eqref{eq:def_val_game}, is the expected performance measure $\Reg_T$ under a draw from this stochastic process. The ``in probability'' Definition~\ref{def:theta-value} is the probability that the performance measure $\Reg_T$ exceeds a threshold $\theta$. 

The above value of the game is related to the expected version of the value of the game. To see this, note that whenever $\Reg_T$ is a non-negative random variable, by Markov inequality we can conclude that 
$$
\Val^\theta_T(\ell,\Phi_T) \le \frac{\Val_T(\ell,\Phi_T)}{\theta}
$$
for any $\theta > 0$. Similarly if $\Compare$ is bounded by $L$ then we can conclude that
$$
\Val_T(\ell,\Phi_T) \le \inf_{\theta >0}\left\{\theta + 2L\ \Val^\theta_T(\ell,\Phi_T) \right\} \ .
$$
Since it is possible to bound expectation by integrating tail probabilities, we will sometimes get better bounds on the expected version of the value by
integrating $\Val^\theta_T(\ell,\Phi_T)$ with respect to $\theta$.

Note that bounding $\Val^\theta_T(\ell,\Phi_T)$ will guarantee, for a fixed $T$ and $\theta$, the existence of a player strategy whose regret against any adversary
will not exceed $\theta$ with high probability. Such a guarantee may already suffice in many cases. However, sometimes we want to prove the existence of Hannan consistent
player strategies: player strategies for a game with infinitely many rounds $t=1,2,\ldots$ such that $\Reg_T \to 0$ almost surely against any adversary.
We will not pursue a formal development of such infinite round games here. Instead, we will show later (in Section~\ref{sec:almostsurecalibration})
how the tools developed below allow us to prove the existence of Hannan consistent strategies for the calibration game. Similar arguments can be used to
show the existence of Hannan consistent player strategies for other games provided some anaologue of the so-called ``doubling trick" is available.

The rest of the section is devoted to tools for bounding the value of the game as defined in Definition~\ref{def:theta-value}. First, we provide the probability version of the Triplex Inequality. 

\begin{theorem}[\textbf{Analogue of Theorem~\ref{thm:main}}]
	\label{thm:main-prob}
For any $\theta > 0$, we have a probabilistic version of the Triplex Inequality:
\begin{align*}
	\Val^\theta_T(\ell,\Phi_T) & \le \sup_{\mbf{D}} \Prob_\mbf{D}\left(  \Compare(\ell(f_1,x_1), \ldots, \ell(f_T, x_T)) - \Compare(\ell(q_1,p_1), \ldots, \ell(q_T, p_T))  > \theta/3\right)\\
			&+\sup_{p_1}\inf_{q_1} \ldots  \sup_{p_T}\inf_{q_T} 
		  \ind{\sup_{\bphi\in \Phi_T}   \left\{ \Compare(\ell(q_1,p_1), \ldots, \ell(q_T, p_T)) - \Compare(\ell_{\phi_1}(q_1,p_1), \ldots, \ell_{\phi_T}(q_T, p_T))\right\}> \theta/3} \\
			&+\sup_{\mbf{D}}\Prob_\mbf{D}\left( \sup_{\bphi\in \Phi_T} \left\{\Compare(\ell_{\phi_1}(q_1,p_1), \ldots, \ell_{\phi_T}(q_T, p_T) ) - \Compare(\ell_{\phi_1}(f_1,x_1),\ldots, \ell_{\phi_T}(f_T,x_T))\right\} > \theta/3 \right)
\end{align*}
where $\mbf{D}$ ranges over distributions over sequences $(x_1,f_1),\ldots, (x_T,f_T)$. 
\end{theorem}

Note that $\mbf{D}$ can be thought of as sequence of conditional distributions $\{(p_t, q_t)\}_{t=1}^T$, where $p_t:(\F,\X)^{t-1}\mapsto\PD$, $q_t:(\F,\X)^{t-1}\mapsto\QD$.

We remark that the second term in the bound of Theorem~\ref{thm:main-prob} is deterministically either one or zero for a given $\theta$.

After the decomposition of Theorem~\ref{thm:main-prob} has been established, we turn to upper bounds on the three terms. Recall that, roughly speaking, the first term is typically bounded via martingale convergence, the second term is bounded by the choice of the best response to the strategy of the adversary, and the third term is bounded by sequential complexity. For the third term, we again apply the sequential symmetrization technique, but now in probability instead of expectation. This requires a bit more work. In particular, for the probabilistic version of Theorem~\ref{thm:rad} we first need the following mild assumption. We require that there is some $T_0< \infty$ such that for all $T>T_0$, for any fixed $\phi \in \Phi_T$, 
	\begin{align}
		\label{eq:assumption_symmetrization}
	\sup_{\mbf{D}} \Prob_{\mbf{D}} \Big( \Compare(\ell_{\phi_1}(q_1,p_1) - \ell_{\phi_1}(f'_1,x'_1), \ldots, \ell_{\phi_T}(q_T, p_T) - \ell_{\phi_T}(f'_T,x'_T)) > \theta/6 ~\Big|~ (f_1,x_1),\ldots, (f_T,x_T) \Big) < 1/2
	\end{align}
	Here $(f'_1,x'_1),\ldots,(f'_T,x'_T)$ is a sequence tangent to the sequence $(f_1, x_1), \ldots, (f_T,x_T)$, drawn from the distributions $(q_1,p_1), \ldots, (q_T,p_T)$. We remark that the assumption of Eq.~\eqref{eq:assumption_symmetrization} is mild and will always be satisfied (for $T$ large enough) in the problems we consider. Indeed, the tangent sequence is independent, given the original sequence, and so \eqref{eq:assumption_symmetrization} is a statement about the behavior of $\Compare$ for zero-mean independent random variables.

\begin{theorem}\label{thm:symmetrization_probability}
Suppose $\Compare$ is sub-additive. Fix $\theta > 0$ and suppose $T$ is large enough so that \eqref{eq:assumption_symmetrization} is satisfied. Then the third term in the Triplex Inequality is bounded by 
\begin{align*}
4 \sup_{\x,\f} \Prob_\epsilon\left( \sup_{\phi \in \Phi_T} \Compare(\epsilon_1 \ell_{\phi_1}(\f_1(\epsilon),\x_1(\epsilon)) , \ldots, \epsilon_T \ell_{\phi_T}(\f_T(\epsilon), \x_T(\epsilon))) > \theta/12 \right).
\end{align*}

If, on the other hand, $-\Compare$ is subadditive, the third term in the Triplex Inequality is instead bounded by 
\begin{align*}
4 \sup_{\x,\f} \Prob_\epsilon\left( \sup_{\phi \in \Phi_T} -\Compare(\epsilon_1 \ell_{\phi_1}(\f_1(\epsilon),\x_1(\epsilon)) , \ldots, \epsilon_T \ell_{\phi_T}(\f_T(\epsilon), \x_T(\epsilon))) > \theta/12 \right)
\end{align*}
\end{theorem}

The following lemma is useful for bounding the first term of the Triplex Inequality in Theorem~\ref{thm:main-prob} when the function $\Compare$ is smooth in each of its arguments. 

\begin{lemma}\label{lem:gensmoothcon}
For any $\cH$-valued martingale difference sequence $\{z_t\}_{t=1}^T$ such that $\|z_t\| \le \RH$, if $\Compare : \cH^T \mapsto \mbb{R}^+$ is such that $\Compare^q$ is $(\sigma,p)$-smooth in each of its arguments and if for all $t \in [T]$,
$\left\|\nabla_{t} \Compare^q\big(z_{1},\ldots,z_{t-1},0,\ldots,0\big)\right\| \le R$, 
then 
$$
\Prob\left(\Compare(z_1,\ldots,z_T) > \theta \right) \le \exp\left( - \frac{\left(\theta ^q - \sigma T \RH^p/p\right)^2}{2\RH^2 R^2 T}\right) \ .
$$
\end{lemma}

In particular, using Lemma~\ref{lem:gensmoothcon} above we can upper bound the third term of the triplex inequality for finite sets of payoff transformations.
\begin{corollary}
For any finite set of payoff transformations $\Phi_T$, under the conditions of Lemma~\ref{lem:gensmoothcon} 
$$
\sup_{\mbf{D}}\Prob_\mbf{D}\left(\sup_{\bphi\in \Phi_T}    \Compare(\ell_{\phi_1}(q_1,p_1) - \ell_{\phi_1}(f_1,x_1), \ldots, \ell_{\phi_T}(q_T, p_T) - \ell_{\phi_T}(f_T,x_T)) > \theta \right) \le |\Phi_T|\ \exp\left( - \frac{\left(\theta^q - \sigma T (2\RH)^p/p\right)^2}{2\RH^2 R^2 T}\right)
$$
\end{corollary}

The above results hold under very general assumptions of smoothness of $\Compare$. Stronger results are attainable if we make an additional assumption that $\Compare$ is a function of the average of its coordinates. The next subsection is devoted to this assumption.

\subsection{When $\Compare$ is a Function of the Average}

Throughout this section, we assume that $\Compare$ is a function of the average of its coordinates:
$$\Compare(z_1, \ldots,z_T) = G\left(\frac{1}{T} \sum_{t=1}^T z_t\right) \ .$$
The following upper bound can be derived.
\begin{lemma}
	\label{lem:pollard}
	Suppose $G\ge0$ is sub-additive, $1$-Lipschitz in the norm $\|\cdot\|$, and $G(0)=0$. Then 
	\begin{align*}
	&\sup_{\f,\x} \Prob_\epsilon\left( \sup_{\phi \in \Phi_T} 
		G\left( 
			\frac{1}{T}\sum_{t=1}^T \epsilon_t \ell_{\phi_t}(\f_t(\epsilon),\x_t(\epsilon)) 
		\right) > \theta \right) \le \N_1(\theta/2,\Phi_T,T)\ \sup_{\z} \Prob_{\epsilon} \left( G\left(\frac{1}{T}\sum_{t=1}^T \epsilon_t \z_t(\epsilon)\right) > \theta/2 \right)  
	\end{align*}
where supremum on the right hand side is over $\cH$-valued trees.
\end{lemma}

Lemma~\ref{lem:pollard} upper bounds the probabilistic version of sequential complexity by the size of an $\ell_1$ cover times the probability that the norm of a martingale difference sequence generated by random signs is close to zero. When the norm in question is $2$-smooth, we can invoke results on concentration of martingales due to Pinelis \cite{Pinelis94}. The results have been re-proven for general $2$-smooth functions in the Appendix. 

\begin{corollary}
	\label{cor:2smooth}
	Under the assumptions of Lemma~\ref{lem:pollard}, if $G^2$ is $(\sigma,2)$-smooth with respect to $\|\cdot\|$ and $\|\ell_\phi(f,x)\|\leq \RH$ for all $\phi,f,x$, then
	for any $T > \theta/4\sigma$, we have 
	\begin{align*}
	&\sup_{\f,\x} \Prob_\epsilon\left( \sup_{\phi \in \Phi_T} 
	G\left( 
		\frac{1}{T}\sum_{t=1}^T \epsilon_t \ell_{\phi_t}(\f_t(\epsilon),\x_t(\epsilon)) 
	\right) > \theta 
	\right) \le 2 \N_1(\theta/2,\Phi_T,T) \exp\left\{-\frac{T\theta^2}{16\sigma \RH^2} \right\} .
	\end{align*}
\end{corollary}

When $\Compare$ is a function of the average of its arguments, Lemma~\ref{lem:pollard} and Corollary~\ref{cor:2smooth} allow us to control the third term in the Triplex
Inequality by applying Theorem~\ref{thm:symmetrization_probability}.
Now, we would like to generalize the above results in two directions. First, we would like to obtain the Dudley integral-type upper bounds instead of the $\ell_1$-cover at a fixed scale. Second, we wish to consider norms which are $p$-smooth for $1 < p \leq 2$. Both the extensions enlarge the scope of problems that can be addressed and also make the upper bounds sharp.

We start by considering the real-valued case with the goal of obtaining upper bounds using the chaining technique.
\begin{proposition}\label{prop:dudreal}
Suppose $\cH\subseteq[-1,1]$. We have that for any $\theta >  \sqrt{8/T}$,
\begin{align*}
\Prob_\epsilon\left(\sup_{\bphi \in \Phi_T} \frac{1}{T}\sum_{t=1}^T \epsilon_t \ell_{\phi_t}(\f_t(\epsilon),\x_t(\epsilon))  > \inf_{\alpha}\left\{ 4\alpha + 12 \theta \int_{\alpha}^1 \sqrt{\log \N_\infty(\delta,\Phi_T,T)}d \delta \right\} \right) \le L\ \exp\left\{- T \theta^2/2  \right\}
\end{align*}
where $L$ is a constant such $L > \sum_{j=1}^\infty \N_\infty(2^{-j},\Phi_T,T)^{-1}$. In particular, for time-invariant constant departure mappings,
\begin{align*}
\Prob_\epsilon\left(\sup_{f \in \F} \frac{1}{T}\sum_{t=1}^T \epsilon_t f(\x_t(\epsilon)) > \inf_{\alpha}\left\{ 4\alpha + 12 \theta \int_{\alpha}^1 \sqrt{\log \N_\infty(\delta,\F,T)}d \delta \right\} \right) \le L\ \exp\left\{- T \theta^2/2  \right\}
\end{align*}
Furthermore, we have, 
\begin{align*}
& \Prob_\epsilon\left(\sup_{f \in \F} \frac{1}{T}\sum_{t=1}^T \epsilon_t f(\x_t(\epsilon)) > 
128\ \Rad_T(\F)\left(1 + \theta \sqrt{T \log^3 (2T)}  \right)
 \right) \le L\ \exp\left\{-T \theta^2/2 \right\}
\end{align*}
where $\Rad_T(\F)$ is the sequential Rademacher complexity of $\F$ as defined in \eqref{eq:rademacher_external_regret}.
\end{proposition}

The next lemma generalizes Proposition~\ref{prop:dudreal} to $2$-smooth norms. Its proof is almost identical to that of  Proposition~\ref{prop:dudreal} and will be omitted. 

\begin{lemma}
Assume that $G\ge 0$ is $1$-Lipschitz w.r.t. norm $\|\cdot\|$, sub-additive, $G(0)=0$, and $G^2$ is $(\sigma,2)$-smooth.
Further, suppose that for any $x \in \X$, $f \in \F$, $\bphi\in\Phi_T$ and $t\in[T]$, it is true that $\|\ell_{\phi_t}(f,x)\| \le 1$.
Then for any $\theta >  \sqrt{8\sigma/T}$ :
{\small
\begin{align*}
& \Prob_\epsilon\left(\sup_{\phi \in \Phi_T} G\left(\frac{1}{T}\sum_{t=1}^T \epsilon_t \ell_{\phi_t}(\f_t(\epsilon),\x_t(\epsilon))\right) > \inf_{\alpha > 0}\left\{ 4 \alpha + 12 \theta \int_{\alpha}^{1} \sqrt{\log \N_\infty(\delta,\Phi_T,T)} d \delta \right\} \right) \le L \exp\left\{- \frac{ T \theta^2  }{4 \sigma }\right\}
\end{align*}}
where $L$ is a constant such $L > 2\ \sum_{j=1}^\infty \N_\infty(2^{-j},\Phi_T,T)^{-1}$.
\end{lemma}

We now turn to the goal of proving upper bounds for general $p$-smooth norms. The following lemma is the main building block for Lemma~\ref{thm:rad_upper_p_smooth}.
It provides a large deviation inequality for (Walsh-Paley) martingale difference sequences in a $(\sigma,p)$-smooth Banach space. As such, it may be of
independent interest.

\begin{lemma}\label{lem:smoothcon}
Let $(\mc{B},\|\cdot\|)$ be a $(\sigma,p)$-smooth space. Let $\x$ be any $\mc{B}$-valued tree of depth $T$ with $\|\x_t(\epsilon)\| \le R$ for any $t,\epsilon$.
For any $\nu > 8\sigma^{1/p}\log^{3/2}T/T^{1-1/p}$, we have that
\begin{equation*}
\Prob\left(\left\|\frac{1}{T} \sum_{t=1}^T \epsilon_t \x_t(\epsilon) \right\| >
	128\ \frac{\sigma^{1/p}R}{T^{1-1/p}} + 128\ \nu R \right)
	\le
	2 \exp\left( -\frac{\nu^2T^{2-2/p}}{2\sigma^{2/p}\log^3 T} \right) \ .
\end{equation*}
\end{lemma}

With the above concentration inequality in hand, we can now derive a Dudley integral type bound when $\cH$ is a subset of a $(\sigma,p)$-smooth space.

\begin{theorem}
	\label{thm:rad_upper_p_smooth}
Assume that $G\ge 0$ is $1$-Lipschitz w.r.t. norm $\|\cdot\|$ and that $(\mc{B}, \|\cdot\|)$ is a $(\sigma,p)$-smooth space.
Further, suppose that for any $x \in \X$, $f \in \F$, $\bphi\in\Phi_T$ and $t\in[T]$, it is true that $\|\ell_{\phi_t}(f,x)\| \le 1$.
Then for any $\theta > \frac{1024 \sigma^{1/p} \log^{3/2} T}{T^{1 - 1/p}} $ :
{\small
\begin{multline*}
\Prob_\epsilon\left(\sup_{\phi \in \Phi_T} G\left(\frac{1}{T}\sum_{t=1}^T \epsilon_t \ell_{\phi_t}(\f_t(\epsilon),\x_t(\epsilon))\right) > \frac{768 \sigma^{1/p}}{T^{1 - 1/p}} + \inf_{\alpha > 0}\left\{ 4 \alpha + 36 \theta \int_{\alpha}^{1} \sqrt{\log \N_\infty(\delta,\Phi_T,T)} d \delta \right\} \right) \\
\le L \exp\left\{- \frac{ \theta^2  T^{2 - 2/p} }{65536\ \sigma^{2/p} \log^3 T}\right\}
\end{multline*}}
where $L$ is a constant such $L > 2\ \sum_{j=1}^\infty \N_\infty(2^{-j},\Phi_T,T)^{-1}$.
\end{theorem}

%% file: highprob_calibration.tex
\subsection{An Almost-Sure Bound for Calibration}
\label{sec:almostsurecalibration}

For the calibration game, using the tools developed above,
we first show the existence of a player strategy guaranteeing small regret with arbitrarily high probability.

\begin{theorem}\label{thm:highprob-calibration}
For the calibration game with $k$ outcomes and with $\ell_1$ norm, we have that for any $\theta > \frac{3}{T}$,
\begin{align}
	\label{eq:highprob-calibration}
	\Val^\theta_T \le  8\exp\left( -\frac{T(\theta/12)^2}{16k} + c k^3 \log(T)\right) 
\end{align}
where $c$ is a fixed numerical constant.
The inequality~\eqref{eq:highprob-calibration} above can be restated as: For any $\eta \in (0,1)$, there is a
player strategy such that, with probability at least $1 - \eta$,
$$
\Reg_T \le 48 \sqrt{ \frac{k \log(8/\eta) +  c k^4 \log\ T}{T}}
$$ 
for $T\ge 3$.
\end{theorem}
\begin{proof}[\textbf{Proof of Theorem~\ref{thm:highprob-calibration}}]
	The proof is similar to that of Theorem~\ref{thm:deltacalib}, with the exception of controlling appropriate quantities in probability in stead of in expectation. We consider the value of the game $\Val^\theta_T(\ell,\Phi_T)$ as in Definition~\ref{def:theta-value} for some $\theta > 0$. Let $\delta>0$ to be determined later. Let $\|\cdot\|$ denote the $\ell_1$ norm. Let $C_{\delta}$ be the maximal $2 \delta$-packing of $\Delta(\X)$ in this norm. Consider the calibration game defined in Example~\ref{eg:calibration2}, augmented with the restriction that the player's choice belongs to $C_\delta$ instead of $\Delta(k)$. The corresponding minimax expression with this restriction is clearly an upper bound on the value of the game defined in Example~\ref{eg:calibration2}. 

We now use the probabilistic version of the Triplex Inequality defined (Theorem~\ref{thm:main-prob}). Observe that the first term in the Triplex Inequality is zero. The second term is upper bounded by a particular (sub)optimal response $q_t$ being the point mass on $p^\delta_t$, the element of $C_\delta$ closest to $p_t$. Note that any $2 \delta$ packing is also a $2 \delta$ cover. Thus, the second term becomes
\begin{align*}
	&\sup_{p_1}\inf_{q_1} \ldots  \sup_{p_T}\inf_{q_T} 
  \ind{\sup_{\bphi\in \Phi_T}   \left\{ - \Compare(\ell_{\phi_1}(q_1,p_1), \ldots, \ell_{\phi_T}(q_T, p_T))\right\}> \theta/3} \\
	&\leq \sup_{p_1}\ldots  \sup_{p_T} \ind{ \sup_{\lambda> 0}\sup_{p\in\Delta(k)} \left\| \frac{1}{T}\sum_{t=1}^T \En_{x_t \sim p_t} \ell_{\phi_{p,\lambda}}(p^\delta_t,x_t) \right\| \geq \theta/3}\\
	&= \sup_{p_1}\ldots  \sup_{p_T} \ind{ \sup_{\lambda> 0}\sup_{p\in\Delta(k)}\left\| \frac{1}{T}\sum_{t=1}^T \ind{\|p^\delta_t-p\|\leq \lambda}\cdot (p^\delta_t-p_t) \right\| \geq \theta/3} \\
	&\leq \ind{\delta\geq \theta/3}
\end{align*}

We now proceed to upper bounded the third term in the Triplex Inequality. If $T$ is large enough such that the conditions of Theorem~\ref{thm:symmetrization_probability} are satisfied, the third term in the Triplex Inequality is upper bounded by 
\begin{align*}
&4 \sup_{\x,\f} \Prob_\epsilon\left( \sup_{\lambda> 0}\sup_{p\in\Delta(k)} \left\| \frac{1}{T}\sum_{t=1}^T \epsilon_t \ind{\|\f_t(\epsilon)-p\|\leq \lambda}\cdot(\f_t(\epsilon)-\x_t(\epsilon)) \right\| > \theta/12 \right) 
\end{align*}
since $-\Compare$ is a subadditive. 

Note that $\f$ is a $C_\delta$-valued tree, not a $\Delta(k)$-valued tree. Using this fact, we would like to pass from the supremum over $\lambda>0$ and $p\in\Delta(k)$ to a supremum over finite discrete set. 

To this end, fix $\f,\x$ and $\epsilon_{1:T}$ and let us see how many genuinely different functions can we get by varying $\lambda>0$ and $p\in\Delta(k)$. This question boils down to looking at the size of the class $$\G := \left\{ g_{p,\lambda}(f) = \ind{\|f-p\|\leq \lambda}: p\in\Delta(k), \lambda > 0 \right\}$$
over the possible values of $f \in C_{\delta}$. Indeed, if
$g_{p,\lambda}(f) = g_{p',\lambda'}(f)$ for all $f\in C_\delta$, then 

$$\frac{1}{T}\sum_{t=1}^T\ind{\|\f_t(\epsilon)-p\|\leq \lambda}\cdot(\f_t(\epsilon)-\x_t(\epsilon)) = \frac{1}{T}\sum_{t=1}^T \ind{\|\f_t(\epsilon)-p'\|\leq \lambda'}\cdot(\f_t(\epsilon)-\x_t(\epsilon)).$$

We appeal to VC theory for bounding the size of $\G$ over $C_\delta$. First, we claim that the VC dimension of $\G$ is $O(k^2)$. Note that $\G$ is the class of indicators over $\ell_1$ balls of radius $\lambda$ centered at $p$ for various values of $p,\lambda$. A result of Goldberg and Jerrum \cite{GolJer95vc} states that for a class $\G$ of functions parametrized by a vector of length $d$, if for $g\in\G$ and $f\in\F$, $\ind{g(f)=1}$ can be computed using $m$ arithmetic operations, the VC dimension of $\G$ is $O(md)$. In our case, the functions in $\G$ are parametrized by $k$ values and membership $\|f-p\|_1\leq \lambda$ can be established in $O(k)$ operations. This yields $O(k^2)$ bound on the VC dimension of $\G$. By Sauer-Shelah Lemma, the number of different labelings of the set $C_\delta$ by $\G$ is bounded by $|C_\delta|^{c\cdot k^2}$ for some absolute constant $c$. We conclude that the effective number of different $(p,\lambda)$ is finite. Let us remark that the VC upper bound is \emph{not} used in place of the sequential Littlestone's dimension. It is only used to show that the set $\Phi_T$  is finite, and such technique can be useful when the set of player's actions is finite.

Hence, there exists a finite set $S$ of pairs $(\lambda,p)$ with cardinality $|S|\leq |C_\delta|^{c\cdot k^2}$ such that
\begin{align*}
&4 \sup_{\x,\f} \Prob_\epsilon\left( \sup_{\lambda> 0}\sup_{p\in\Delta(k)} \left\| \frac{1}{T}\sum_{t=1}^T \epsilon_t \ind{\|\f_t(\epsilon)-p\|\leq \lambda}\cdot(\f_t(\epsilon)-\x_t(\epsilon)) \right\| > \theta/12 \right) \\
&\leq 4 \sup_{\x,\f} \Prob_\epsilon\left( \max_{(p,\lambda) \in S} \left\| \frac{1}{T}\sum_{t=1}^T \epsilon_t \ind{\|\f_t(\epsilon)-p\|\leq \lambda}\cdot(\f_t(\epsilon)-\x_t(\epsilon)) \right\| > \theta/12 \right) \\
&\leq 4|S| \sup_{\z} \Prob_\epsilon\left( \left\| \frac{1}{T}\sum_{t=1}^T \epsilon_t \z_t(\epsilon) \right\|> \theta/12 \right)
\end{align*}

where the supremum is over all $2B_1^k$-valued binary trees of depth $T$, where $B_1^k$ is a unit $\ell_1$ ball in $\reals^k$. Note that the $\|\cdot\|_1\leq \sqrt{k}\|\cdot\|_2$. By Corollary~\ref{cor:pinelis_concentration}, 
$$\Prob_\epsilon\left( \left\| \frac{1}{T}\sum_{t=1}^T \epsilon_t \z_t(\epsilon) \right\|_1 > \theta/12 \right) \leq \Prob_\epsilon\left( \left\| \frac{1}{T}\sum_{t=1}^T \epsilon_t \z_t(\epsilon) \right\|_2 > \theta/(12\sqrt{k}) \right) \leq 2\exp\left( -\frac{T(\theta/12)^2}{16 k}\right)$$

Now note that the size of set $C_\delta$ the $2 \delta$ packing of $\Delta(k)$ is upper bounded by the size of the minimal $\delta$ cover of $\Delta(k)$
which can be bounded as $|C_\delta| \le \left(\frac{1}{\delta}\right)^{k-1}$ and so we see that
\begin{align*}
4|S| \sup_{\z} \Prob_\epsilon\left( \left\| \frac{1}{T}\sum_{t=1}^T \epsilon_t \z_t(\epsilon) \right\|> \theta/12 \right) &\leq 8 \left(\frac{1}{\delta}\right)^{c k^3} \exp\left( -\frac{T(\theta/12)^2}{16 k}\right) \\
&=8\exp\left( -\frac{T(\theta/12)^2}{16k} + ck^3 \log(1/\delta)\right) 
\end{align*}

Combining everything we see that,
\begin{align*}
\Val^\theta_T \le \ind{\delta\geq \theta/3} + 8\exp\left( -\frac{T(\theta/12)^2}{16k} + c k^3 \log(1/\delta)\right) 
\end{align*}
Choosing,  $\delta = 1/T$ gives
\begin{align*}
\Val^\theta_T \le \ind{1/T \geq \theta/3} + 8\exp\left( -\frac{T(\theta/12)^2}{16k} + c k^3 \log(T)\right) 
\end{align*}
which gives the first statement of the theorem.

We now rewrite the result in terms of a fixed probability of deviation. To this end, set
$$
\frac{\eta}{8} =  \exp\left( -\frac{T(\theta/12)^2}{16k} + c k^3 \log(T)\right)
$$
which gives
$$
\theta = 48 \sqrt{ \frac{k \log(8/\eta) +  c k^4 \log\ T}{T}}
$$
Note that for any $T\ge 3$ and $\eta \in (0,1)$, we have 
$$T > \frac{1}{16^2(k \log(8/\eta) +  c k^4 \log\ T)}\ .$$
Hence we conclude that for any $\eta \in (0,1)$,
we have with probability at least $1 - \eta$,
$$
\Reg_T \le 48 \sqrt{ \frac{k \log(8/\eta) +  c k^4 \log\ T}{T}} \ .
$$
\end{proof}

The above result almost suffices to get a result stating almost sure convergence. The only issue is that the player
strategy guaranteed above depends on the confidence level $\eta$. In the proof of the following result, we show
how to achieve small regret uniformly for all confidence levels $\eta$. Then, it is fairly easy to show the existence of a Hannan consistent
strategy for the calibration game.

\begin{theorem}
	\label{thm:calibration_almost_sure}
Suppose the calibration game is played for infinitely many rounds $T=1,2,\ldots$. Then there exists a player strategy such that against
any adversary we have,
\[
\limsup_{T\to\infty} \frac{\sqrt{T}}{\sqrt{ 3 k\log(2T) + \tfrac{ck^4}{2} \log(T) }} \cdot \Reg_T \le 60 \qquad\qquad \text{almost surely}\ .
\]
\end{theorem}

The proof of Theorem~\ref{thm:calibration_almost_sure} can be taken as a general recipe for proving almost sure bounds (and, therefore, Hannan consistency). The idea is to lift the dependence of the in-probability value $\Val_T^\theta$ (as well as player's strategy) on $\theta$ by instead considering a closely related value of the form $\En \exp\left\{K\Reg_T^2\right\}$ for some appropriate $T$-dependent factor $K$. Whenever this value is bounded, Markov's inequality gives tail bounds for a strategy that does not depend on $\theta$. Together with a doubling trick, this leads to an almost sure convergence guarantee.

%% file: appendix.tex
\subsection*{Appendix}

\begin{proof}[\textbf{Proof of Theorem~\ref{thm:main}}]
	The value of the game, defined in \eqref{eq:def_val_game}, is
	\begin{align*}
		\Val_T(\ell,\Phi_T) &= \inf_{q_1} \sup_{p_1} \Eunder{x_1\sim p_1}{f_1\sim q_1} \ldots \inf_{q_T} \sup_{p_T} \Eunder{x_T\sim p_T}{f_T\sim q_T}  
		\sup_{\bphi\in \Phi_T}\left\{ \Compare(\ell(f_1,x_1), \ldots, \ell(f_T, x_T)) -  \Compare(\ell_{\phi_1}(  f_1,x_1), \ldots, \ell_{\phi_T}(f_T, x_T))\right\}\\
		&= \sup_{p_1}\inf_{q_1} \Eunder{x_1\sim p_1}{f_1\sim q_1} \ldots  \sup_{p_T}\inf_{q_T} \Eunder{x_T\sim p_T}{f_T\sim q_T}  
		\sup_{\bphi\in \Phi_T} \left\{ \Compare(\ell(f_1,x_1), \ldots, \ell(f_T, x_T)) - \Compare(\ell_{\phi_1 }( f_1,x_1), \ldots, \ell_{\phi_T}(f_T, x_T))\right\}
	\end{align*}
	via an application of the minimax theorem. Adding and subtracting terms to the expression above leads to
	\begin{align*}
		\Val_T(\ell, \Phi_T) &= \sup_{p_1}\inf_{q_1}\Eunder{x_1\sim p_1}{f_1 \sim q_1} \ldots  \sup_{p_T}\inf_{q_T} \Eunder{x_T\sim p_T}{f_T \sim q_T}
		\hspace{0.4cm} \left[ \Compare(\ell(f_1,x_1), \ldots, \ell(f_T, x_T)) - \Eunder{x'_{1:T} \sim p_{1:T} }{ f'_{1:T} \sim q_{1:T}} \Compare(\ell(f'_1,x'_1), \ldots, \ell(f'_T, x'_T)) \right.\\
		&\left.\hspace{1in}+\sup_{\bphi\in \Phi_T} \left\{ \Eunder{x'_{1:T} \sim p_{1:T} }{ f'_{1:T} \sim q_{1:T}} \Compare(\ell(f'_1,x'_1), \ldots, \ell(f'_T, x'_T)) - \Compare(\ell_{\phi_1}(f_1,x_1), \ldots, \ell_{\phi_T}(f_T,x_T))\right\} \right]\\
		&\leq \sup_{p_1}\inf_{q_1}\Eunder{x_1\sim p_1}{f_1 \sim q_1} \ldots  \sup_{p_T}\inf_{q_T} \Eunder{x_T\sim p_T}{f_T \sim q_T}
		\hspace{0.4cm} \left[ \Compare(\ell(f_1,x_1), \ldots, \ell(f_T, x_T)) - \Eunder{x'_{1:T} \sim p_{1:T} }{ f'_{1:T} \sim q_{1:T}} \Compare(\ell(f'_1,x'_1), \ldots, \ell(f'_T, x'_T)) \right.\\
		&\left.\hspace{1in}+\sup_{\bphi\in \Phi_T}  \Eunder{x'_{1:T} \sim p_{1:T} }{ f'_{1:T} \sim q_{1:T}} \Big\{ \Compare(\ell(f'_1,x'_1), \ldots, \ell(f'_T, x'_T)) -   \Compare(\ell_{\phi_1}(f_1',x_1'), \ldots, \ell_{\phi_T}(f_T', x_T'))\Big\} \right.\\
		&\left.\hspace{1in} + \sup_{\bphi\in \Phi_T} \left\{ \Eunder{x'_{1:T} \sim p_{1:T} }{ f'_{1:T} \sim q_{1:T}} \Compare(\ell_{\phi_1}(f_1',x_1'), \ldots, \ell_{\phi_T}(f_T', x_T')) -  \Compare(\ell_{\phi_1}(f_1,x_1), \ldots, \ell_{\phi_T}(f_T,x_T))\right\} \right]
		\end{align*}
	At this point, we would like to break up the expression into three terms. To do so, notice that expectation is linear and $\sup$ is a convex function, while for the infimum,
	$$\inf_a \left[C_1(a)+C_2(a)+C_3(a)\right] \leq \left[\sup_a C_1(a)\right] + \left[\inf_a C_2(a)\right] + \left[\sup_a C_3(a)\right] $$
	for functions $C_1,C_2,C_3$. We use these properties of $\inf$, $\sup$, and expectation, starting from the inside of the nested expression and splitting the expression in three parts. We arrive at
	{\small
	\begin{align*}
		&\Val_T(\ell, \Phi_T)\\ 
		&\leq \sup_{p_1}\sup_{q_1}\Eunder{x_1\sim p_1}{f_1 \sim q_1} \ldots  \sup_{p_T}\sup_{q_T} \Eunder{x_T\sim p_T}{f_T \sim q_T}
		 \Big[ \Compare(\ell(f_1,x_1), \ldots, \ell(f_T, x_T)) -  \Eunder{x'_{1:T} \sim p_{1:T} }{ f'_{1:T} \sim q_{1:T}}  \Compare(\ell(f'_1,x'_1), \ldots, \ell(f'_T, x'_T)) \Big]\\
		&+\sup_{p_1}\inf_{q_1}\Eunder{x_1\sim p_1}{f_1 \sim q_1} \ldots  \sup_{p_T}\inf_{q_T} \Eunder{x_T\sim p_T}{f_T \sim q_T}
	 \left[ \sup_{\bphi\in \Phi_T} \Eunder{x'_{1:T} \sim p_{1:T} }{ f'_{1:T} \sim q_{1:T}}  \left\{ \Compare(\ell(f'_1,x'_1), \ldots, \ell(f'_T, x'_T)) - \Compare(\ell_{\phi_1}(f'_1,x'_1), \ldots, \ell_{\phi_T}(f'_T, x'_T))\right\} \right]\\
		&+\sup_{p_1}\sup_{q_1}\Eunder{x_1\sim p_1}{f_1 \sim q_1} \ldots  \sup_{p_T}\sup_{q_T} \Eunder{x_T\sim p_T}{f_T \sim q_T} \left[ \sup_{\bphi\in \Phi_T} \left\{ \Eunder{x'_{1:T} \sim p_{1:T} }{ f'_{1:T} \sim q_{1:T}}  \Compare(\ell_{\phi_1}(f'_1,x'_1), \ldots, \ell_{\phi_T}(f'_T, x'_T)) -  \Compare(\ell_{\phi_1}(f_1,x_1), \ldots, \ell_{\phi_T}(f_T,x_T))\right\} \right]
	\end{align*}
	}
	The replacement of infima by suprema in the first and third terms appears to be a loose step and, indeed, one can pick a particular response strategy $\{q^*_t\}$ instead of passing to the supremum. For instance, this can be the best-response strategy for the second term. However, in the examples we have considered so far, passing to the supremum still yields the results we need. This is due to the fact that the online learning setting is worst-case.

	Consider the second term in the above decomposition. We claim that
	\begin{align*}
		&\sup_{p_1}\inf_{q_1}\Eunder{x_1\sim p_1}{f_1 \sim q_1} \ldots  \sup_{p_T}\inf_{q_T} \Eunder{x_T\sim p_T}{f_T \sim q_T}
		 \left[ \sup_{\bphi\in \Phi_T} \Eunder{x'_{1:T} \sim p_{1:T} }{ f'_{1:T} \sim q_{1:T}} \left[ \Compare(\ell(f'_1,x'_1), \ldots, \ell(f'_T, x'_T)) - \Compare(\ell_{\phi_1}(f'_1,x'_1), \ldots, \ell_{\phi_T}(f'_T, x'_T))\right] \right]\\
		&=\sup_{p_1}\inf_{q_1} \ldots  \sup_{p_T}\inf_{q_T} 
		 \sup_{\bphi\in \Phi_T} \Eunder{x_{1:T} \sim p_{1:T}}{f_{1:T} \sim q_{1:T}}  \left[ \Compare(\ell(f_1,x_1), \ldots, \ell(f_T, x_T)) -  \Compare(\ell_{\phi_1}(f_1,x_1), \ldots, \ell_{\phi_T}(f_T, x_T))\right] 
	\end{align*}
	because the objective 
	$$\Eunder{x'_{1:T} \sim p_{1:T} }{ f'_{1:T} \sim q_{1:T}}\left[\Compare(\ell(f'_1,x'_1), \ldots, \ell(f'_T, x'_T)) -   \Compare(\ell_{\phi_1}(f'_1,x'_1), \ldots, \ell_{\phi_T}(f'_T, x'_T)) \right]$$ 
	does not depend on the random draws $f_1,x_1,\ldots, f_T, x_T$. We then rename $f_t',x_t'$ into $f_t,x_t$.
	This concludes the proof of the Triplex Inequality. 
\end{proof}

\begin{proof}[\textbf{Proof of Theorem~\ref{thm:rad}}]
	We turn to the third term in the Triplex Inequality. If $\Compare$ is subadditive,
	\begin{align*}
		\Eunder{x'_{1:T} \sim p_{1:T} }{ f'_{1:T} \sim q_{1:T}} \Compare(\ell_{\phi_1}(f'_1,x'_1), \ldots, \ell_{\phi_T}(f'_T, x'_T)) -  \Compare(\ell_{\phi_1}(f_1,x_1), \ldots, \ell_{\phi_T}(f_T,x_T)) \\
		\leq \Eunder{x'_{1:T} \sim p_{1:T} }{ f'_{1:T} \sim q_{1:T}} \Compare(\ell_{\phi_1}(f'_1,x'_1) - \ell_{\phi_1}(f_1,x_1), \ldots, \ell_{\phi_T}(f'_T, x'_T)- \ell_{\phi_T}(f_T,x_T)).
	\end{align*}
	If, on the other hand, $-\Compare$ is subadditive, 
	\begin{align}
		\label{eq:negative_of_cumloss_subadditive}
		\Eunder{x'_{1:T} \sim p_{1:T} }{ f'_{1:T} \sim q_{1:T}} \Compare(\ell_{\phi_1}(f'_1,x'_1), \ldots, \ell_{\phi_T}(f'_T, x'_T)) -  \Compare(\ell_{\phi_1}(f_1,x_1), \ldots, \ell_{\phi_T}(f_T,x_T)) \nonumber\\
		\leq -\Eunder{x'_{1:T} \sim p_{1:T} }{ f'_{1:T} \sim q_{1:T}} \Compare( \ell_{\phi_1}(f_1,x_1)- \ell_{\phi_1}(f'_1,x'_1), \ldots, \ell_{\phi_T}(f_T,x_T)-\ell_{\phi_T}(f'_T, x'_T)). 
	\end{align}
	Below assume that $\Compare$ is subadditive, and the proof of the other case is identical.

	To prove the bound on the third term in terms of twice sequential complexity, we proceed as in \cite{RakSriTew10}, applying the symmetrization technique from inside out. To this end, first note that, 
		\begin{align*}
			& \sup_{p_1, q_1}\Eunder{x_1\sim p_1}{f_1 \sim q_1} \ldots  \sup_{p_T, q_T} \Eunder{x_T\sim p_T}{f_T \sim q_T}\sup_{\bphi\in \Phi_T}
			\Eunder{x'_1\sim p_1, \ldots x'_T\sim p_T}{f'_1\sim q_1, \ldots, f'_T\sim q_T} 
			\Compare\Big(\ell_{\phi_1}(f'_1,x'_1)-\ell_{\phi_1}(f_1,x_1), \ldots, \ell_{\phi_T}(f'_T, x'_T)-\ell_{\phi_T}(f_T,x_T) \Big) \\		
			&\leq
			\sup_{p_1,q_1}\Eunder{x_1,x'_1\sim p_1}{f_1,f'_1 \sim q_1} \ldots  \sup_{p_T,q_T} \Eunder{x_T, x'_T\sim p_T}{f_T, f'_T \sim q_T}\sup_{\bphi\in \Phi_T}
			\Compare\Big(\ell_{\phi_1}(f'_1,x'_1)-\ell_{\phi_1}(f_1,x_1), \ldots, \ell_{\phi_T}(f'_T, x'_T)-\ell_{\phi_T}(f_T,x_T) \Big)
		\end{align*}
	the above is true because the expectations are pulled outside the suprema, thus resulting in an upper bound.
		Now notice that conditioned on history $f_T, f'_T$ are distributed identically and independently drawn from $q_T$. Similarly $x_T,x'_T$ are also identically distributed conditioned on history. Hence renaming them we see that
	\begin{align*}
	& \Eunder{x_T, x'_T\sim p_T}{f_T, f'_T \sim q_T}\sup_{\bphi\in \Phi_T}
			\Compare\Big(\ell_{\phi_1}(f'_1,x'_1)-\ell_{\phi_1}(f_1,x_1), \ldots, \ell_{\phi_T}(f'_T, x'_T)-\ell_{\phi_T}(f_T,x_T) \Big)   \\
			& ~~~~~~~~~~~~~~~ = \Eunder{x'_T, x_T\sim p_T}{f'_T, f_T \sim q_T}\sup_{\bphi\in \Phi_T}
			\Compare\Big(\ell_{\phi_1}(f'_1,x'_1)-\ell_{\phi_1}(f_1,x_1), \ldots, \ell_{\phi_T}(f_T,x_T)-\ell_{\phi_T}(f'_T, x'_T) \Big) \\
			& ~~~~~~~~~~~~~~~  = \Eunder{x_T, x'_T\sim p_T}{f_T, f'_T \sim q_T}\sup_{\bphi\in \Phi_T}
			\Compare\Big(\ell_{\phi_1}(f'_1,x'_1)-\ell_{\phi_1}(f_1,x_1), \ldots, -(\ell_{\phi_T}(f'_T, x'_T)-\ell_{\phi_T}(f_T,x_T)) \Big) 
	\end{align*}
	where only the last argument of $\Compare$ is changing sign.
	Thus,
	\begin{align*}
	& \Eunder{x_T, x'_T\sim p_T}{f_T, f'_T \sim q_T}\sup_{\bphi\in \Phi_T}
			\Compare\Big(\ell_{\phi_1}(f'_1,x'_1)-\ell_{\phi_1}(f_1,x_1), \ldots, \ell_{\phi_T}(f'_T, x'_T)-\ell_{\phi_T}(f_T,x_T) \Big)   \\
			& ~~~~~~~~~~~~~~~  = \En_{\epsilon_T}\Eunder{x_T, x'_T\sim p_T}{f_T, f'_T \sim q_T}\sup_{\bphi\in \Phi_T}
			\Compare\Big(\ell_{\phi_1}(f'_1,x'_1)-\ell_{\phi_1}(f_1,x_1), \ldots, \epsilon_T(\ell_{\phi_T}(f'_T, x'_T)-\ell_{\phi_T}(f_T,x_T)) \Big) 
	\end{align*}
	where $\epsilon_T$ is a Rademacher random variable. Furthermore,
	\begin{align*}
	& \sup_{p_T,q_T} \Eunder{x_T, x'_T\sim p_T}{f_T, f'_T \sim q_T}\sup_{\bphi\in \Phi_T}
			\Compare\Big(\ell_{\phi_1}(f'_1,x'_1)-\ell_{\phi_1}(f_1,x_1), \ldots, \ell_{\phi_T}(f'_T, x'_T)-\ell_{\phi_T}(f_T,x_T) \Big)   \\
			& ~~~~~~~~~~~~~~~ = \sup_{p_T,q_T}\Eunder{x'_T, x_T\sim p_T}{f'_T, f_T \sim q_T} \En_{\epsilon_T} \sup_{\bphi\in \Phi_T}
			\Compare\Big(\ell_{\phi_1}(f'_1,x'_1)-\ell_{\phi_1}(f_1,x_1), \ldots, \epsilon_T (\ell_{\phi_T}(f'_T, x'_T)-\ell_{\phi_T}(f_T,x_T)) \Big) \\
			& ~~~~~~~~~~~~~~~ \le \sup_{\underset{f_T,f'_T \in \F}{x_T,x'_T \in \X}}\En_{\epsilon_T} \sup_{\bphi\in \Phi_T}
			\Compare\Big(\ell_{\phi_1}(f'_1,x'_1)-\ell_{\phi_1}(f_1,x_1), \ldots, \epsilon_T (\ell_{\phi_T}(f'_T, x'_T)-\ell_{\phi_T}(f_T,x_T)) \Big) 
	\end{align*}
		Proceeding similarly notice that since given history $x_{T-1},x'_{T-1}$ and $f_{T-1}, f'_{T-1}$ are distributed independently and identically we have,
	{\small
	\begin{align*}
	&\sup_{p_{T-1},q_{T-1}} \Eunder{x_{T-1}, x'_{T-1}\sim p_{T-1}}{f_{T-1}, f'_{T-1} \sim q_{T-1}}  \sup_{\underset{f_T,f'_T \in \F}{x_T,x'_T \in \X}}\En_{\epsilon_T} \sup_{\bphi\in \Phi_T}\\
	&~~~~~~~~~~\Compare\Big(\ell_{\phi_1}(f'_1,x'_1)-\ell_{\phi_1}(f_1,x_1), \ldots, \ell_{\phi_{T-1}}(f'_{T-1}, x'_{T-1})-\ell_{\phi_{T-1}}(f_{T-1}, x_{T-1}), \epsilon_T (\ell_{\phi_T}(f'_T, x'_T)-\ell_{\phi_T}(f_T,x_T)) \Big) \\
	& = \sup_{p_{T-1}, q_{T-1}} \Eunder{x_{T-1}, x'_{T-1}\sim p_{T-1}}{f_{T-1}, f'_{T-1} \sim q_{T-1}} \En_{\epsilon_{T-1}} \sup_{\underset{f_T,f'_T \in \F}{x_T,x'_T \in \X}}\En_{\epsilon_T} \sup_{\bphi\in \Phi_T}\\
	&~~~~~~~~~~\Compare\Big(\ell_{\phi_1}(f'_1,x'_1)-\ell_{\phi_1}(f_1,x_1), \ldots, \epsilon_{T-1} (\ell_{\phi_T}(f'_{T-1}, x'_{T-1})-\ell_{\phi_{T-1}}(f_{T-1}, x_{T-1})), \epsilon_T (\ell_{\phi_T}(f'_T, x'_T)-\ell_{\phi_T}(f_T,x_T)) \Big) \\
	& \le \sup_{\underset{f_{T-1}, f'_{T-1} \in \F}{x_{T-1}, x'_{T-1} \in \X}} \En_{\epsilon_{T-1}} \sup_{\underset{f_T,f'_T \in \F}{x_T,x'_T \in \X}}\En_{\epsilon_T} \sup_{\bphi\in \Phi_T}\\
	&~~~~~~~~~~\Compare\Big(\ell_{\phi_1}(f'_1,x'_1)-\ell_{\phi_1}(f_1,x_1), \ldots, \epsilon_{T-1} (\ell_{\phi_{T-1}}(f'_{T-1}, x'_{T-1})-\ell_{\phi_{T-1}}(f_{T-1}, x_{T-1})), \epsilon_T (\ell_{\phi_T}(f'_T, x'_T)-\ell_{\phi_T}(f_T,x_T)) \Big) 
	\end{align*}
	}
	Proceeding in similar fashion introducing Rademacher random variables all the way to $\epsilon_1$ we arrive at
	\begin{align*}
		&\sup_{p_1,q_1}\Eunder{x_1,x'_1\sim p_1}{f_1,f'_1 \sim q_1} \ldots  \sup_{p_T,q_T} \Eunder{x_T, x'_T\sim p_T}{f_T, f'_T \sim q_T}\sup_{\bphi\in \Phi_T}
		\Compare\Big(\ell_{\phi_1}(f'_1,x'_1)-\ell_{\phi_1}(f_1,x_1), \ldots, \ell_{\phi_T}(f'_T, x'_T)-\ell_{\phi_T}(f_T,x_T) \Big)\\
		&\le \sup_{\underset{f_{1}, f'_{1} \in \F}{x_{1}, x'_{1} \in \X}} \En_{\epsilon_{1}} \ldots \sup_{\underset{f_T,f'_T \in \F}{x_T,x'_T \in \X}}\En_{\epsilon_T} \sup_{\bphi\in \Phi_T} \Compare\Big(\epsilon_1(\ell_{\phi_1}(f'_1,x'_1)-\ell_{\phi_1}(f_1,x_1)), \ldots,  \epsilon_T (\ell_{\phi_T}(f'_T, x'_T)-\ell_{\phi_T}(f_T,x_T)) \Big) 
	\end{align*}
	Subadditivity of $\Compare$ implies $\Compare(a-b)\leq \Compare(a)+\Compare(-b)$, and thus
	\begin{align*}
		&\Compare\Big(\epsilon_1(\ell_{\phi_1}(f'_1,x'_1)-\ell_{\phi_1}(f_1,x_1)), \ldots,  \epsilon_T (\ell_{\phi_T}(f'_T, x'_T)-\ell_{\phi_T}(f_T,x_T)) \Big) \\
		&\leq \Compare\Big(\epsilon_1\ell_{\phi_1}(f'_1,x'_1), \ldots,  \epsilon_T \ell_{\phi_T}(f'_T, x'_T) \Big) + \Compare\Big(-\epsilon_1\ell_{\phi_1}(f_1,x_1), \ldots,  -\epsilon_T\ell_{\phi_T}(f_T,x_T) \Big)
	\end{align*}
	We, therefore, arrive at
	\begin{align*}
		&\sup_{\underset{f_{1}, f'_{1} \in \F}{x_{1}, x'_{1} \in \X}} \En_{\epsilon_{1}} \ldots \sup_{\underset{f_T,f'_T \in \F}{x_T,x'_T \in \X}}\En_{\epsilon_T} \sup_{\bphi\in \Phi_T} \Compare\Big(\epsilon_1(\ell_{\phi_1}(f'_1,x'_1)-\ell_{\phi_1}(f_1,x_1)), \ldots,  \epsilon_T (\ell_{\phi_T}(f'_T, x'_T)-\ell_{\phi_T}(f_T,x_T)) \Big) \\
	& ~~~~~~\le 2 \sup_{f_1 \in \F, x_1 \in \X } \En_{\epsilon_{1}} \ldots \sup_{f_T \in \F, x_T \in \X} \En_{\epsilon_T} \sup_{\bphi\in \Phi_T} \Compare\Big(\epsilon_1 \ell_{\phi_1}(f_1,x_1), \ldots,  \epsilon_T \ell_{\phi_T}(f_T,x_T) \Big) \\
	& ~~~~~~ = 2 \sup_{(\f, \x)}\ \En_{\epsilon_{1:T}} \sup_{\bphi\in \Phi_T} \Compare\Big(\epsilon_1 \ell_{\phi_1}(\f_1(\epsilon),\x_1(\epsilon)), \ldots,  \epsilon_T \ell_{\phi_T}(\f_T(\epsilon), \x_T(\epsilon)) \Big)
	\end{align*}
	where in the last step we passed to the supremum over $(\F\times\X)$-valued trees. This concludes the proof for the case of $\Compare$ being subadditive. Starting from  Eq.~\eqref{eq:negative_of_cumloss_subadditive}, the proof for the case of $-\Compare$ being subadditive and convex in each of its coordinates leads to the bound of
	\begin{align*}
	2 \sup_{(\f, \x)}\ \En_{\epsilon_{1:T}} \sup_{\bphi\in \Phi_T} -\Compare\Big(\epsilon_1 \ell_{\phi_1}(\f_1(\epsilon),\x_1(\epsilon)), \ldots,  \epsilon_T \ell_{\phi_T}(\f_T(\epsilon), \x_T(\epsilon)) \Big) .
	\end{align*}
	The complete proof can be repeated for the first term in the Triplex Inequality in order to bound it by $2\Rad_T(\ell, \mc{I}, \Compare)$ (or respectively $2\Rad_T(\ell, \mc{I}, - \Compare)$).
\end{proof}

The following Proposition is immediate from the definition of a smooth function via successive expansions of each coordinate around zero.

\begin{proposition}\label{prop:smooth}
Assume function $\Compare : \cH^T \mapsto \reals$ is $(\sigma,p)$-uniformly smooth in each of its arguments and that $\Compare(0,0,\ldots,0) = 0$. Then 
$$
\Compare(z_1,\ldots,z_T) \le \sum_{t=1}^T \inner{\nabla_{t} \Compare(z_1,\ldots,z_{t-1},0,\ldots,0), z_{t}} + \sum_{t=1}^T \frac{\sigma}{p} \|z_t\|^p
$$
\end{proposition}

\begin{lemma}
	\label{lem:smooth_fun_expansion}
	Assume that for some $q\ge 1$, $\Compare^q$ is $(\sigma, p)$-uniformly smooth in each of its arguments and $\Compare(0,\ldots,0) = 0$. Then we have that
	\begin{align*}
	\sup_{p_1}\Eu{z_1,z'_1\sim p_1} \ldots  \sup_{p_T} \Eu{z_T,z'_T\sim p_T} \Compare(z_1 - z'_1, \ldots, z_T -  z'_T ) & \le \left((2\RH)^p \sigma T/p \right)^{1/q}
	\end{align*}
	where the maximization is over distributions $p_t$ with support in the ball $\RH \cdot B_{\|\cdot\|}$ of radius $\RH$.
\end{lemma}
\begin{proof}[\textbf{Proof of Lemma~\ref{lem:smooth_fun_expansion}}]
By Proposition~\ref{prop:smooth} we have that
\begin{align*}
& \sup_{p_1}\Eu{z_1,z'_1\sim p_1} \ldots  \sup_{p_T} \Eu{z_T,z'_T\sim p_T} \Compare^q(z_1 - z'_1, \ldots, z_T - z'_T ) \\
& \le \sup_{p_1}\Eu{z_1,z'_1\sim p_1} \ldots  \sup_{p_T} \Eu{z_T,z'_T\sim p_T} \left\{ \sum_{t=1}^T \inner{\nabla_{t} \Compare^q(z_1- z'_1,\ldots,z_{t-1}- z'_{t-1}, 0, \ldots, 0), z_t- z'_t} 
 + \frac{\sigma}{p} \sum_{t=1}^T \|z_t - z'_t\|^p \right\}\\
& \le \sup_{p_1}\Eu{z_1,z'_1\sim p_1} \ldots  \sup_{p_T} \Eu{z_T,z'_T\sim p_T} \left\{ \sum_{t=1}^T \inner{\nabla_{t} \Compare^q(z_1-z'_1,\ldots,z_{t-1}- z'_{t-1},0,\ldots,0), z_t- z'_t}  \right\}\\
&~~~~~~~+ \sup_{p_1}\Eu{z_1,z'_1\sim p_1} \ldots  \sup_{p_T} \Eu{z_T,z'_T\sim p_T} \left\{ \frac{\sigma}{p} \sum_{t=1}^T \|z_t - z'_t\|^p \right\}\\
&= \sup_{p_1}\Eu{z_1,z'_1\sim p_1} \ldots  \sup_{p_T} \Eu{z_T,z'_T\sim p_T} \left\{ \frac{\sigma}{p} \sum_{t=1}^T \|z_t - z'_t\|^p \right\} \le (2\RH)^p \sigma T/p \\
\end{align*}
Since $q \geq 1$, by Jensen's inequality we conclude that
$$
\sup_{p_1}\Eu{z_1,z'_1\sim p_1} \ldots  \sup_{p_T} \Eu{z_T,z'_T\sim p_T} \Compare(z_1 - z'_1, \ldots, z_T - z'_T ) \le \left((2\RH)^p \sigma T/p \right)^{1/q}
$$
\end{proof}

\begin{proof}[\textbf{Proof of Lemma~\ref{lem:smooth_bound_first_term}}]
	The proof follows immediately from Lemma~\ref{lem:smooth_fun_expansion}. 
\end{proof}

\begin{proof}[\textbf{Proof of Lemma~\ref{lem:smoothrad}}]
By Proposition \ref{prop:smooth} we have:

\begin{align*}
&\Compare^q\Big(\epsilon_1 \ell_{\phi_1}(\f_1(\epsilon),\x_1(\epsilon)), \ldots,  \epsilon_T \ell_{\phi_T}(\f_T(\epsilon), \x_T(\epsilon)) \Big) \\
& \le \sum_{t=1}^T \inner{\nabla_{t} \Compare^q\big(\epsilon_1 \ell_{\phi_1}(\f_1(\epsilon),\x_1(\epsilon)), \ldots,  \epsilon_{t-1} \ell_{\phi_{t-1}}(\f_{t-1}(\epsilon), \x_{t-1}(\epsilon)) ,0,\ldots,0\big) , \epsilon_t \ell_{\phi_t}(\f_{t}(\epsilon), \x_{t}(\epsilon)) } + \frac{\sigma}{p} \sum_{t=1}^T \|\ell_{\phi_t}(f_{t},x_{t})\|^p\\
& \le \sum_{t=1}^T \epsilon_t g_t\big(\ell_{\phi_1}(\f_1(\epsilon),\x_1(\epsilon)),\ldots,\ell_{\phi_t}(\f_t(\epsilon),\x_t(\epsilon))\big) + \sigma \RH^p T/p
\end{align*}
where in the last line we used the definition of $g_t$ as well as an upper bound on the norm.
Now by Jensen's inequality we get
\begin{align*}
& \sup_{\f, \x}\ \En_{\epsilon_{1:T}} \sup_{\bphi\in \Phi_T} \Compare\Big(\epsilon_1 \ell_{\phi_1}(\f_1(\epsilon),\x_1(\epsilon)), \ldots,  \epsilon_T \ell_{\phi_T}(\f_T(\epsilon), \x_T(\epsilon)) \Big) \\
&  \le \left( \sup_{\f, \x}\ \En_{\epsilon_{1:T}} \sup_{\bphi\in \Phi_T}  \sum_{t=1}^T \epsilon_t g_t\big(\ell_{\phi_1}(\f_1(\epsilon),\x_1(\epsilon)),\ldots,\ell_{\phi_t}(\f_t(\epsilon),\x_t(\epsilon))\big) + \sigma \RH^p T/p \right)^{1/q}\\
&\leq \left( \sup_{\f, \x}\ \En_{\epsilon_{1:T}} \sup_{\bphi\in \Phi_T}  \sum_{t=1}^T \epsilon_t g_t\big(\ell_{\phi_1}(\f_1(\epsilon),\x_1(\epsilon)),\ldots,\ell_{\phi_t}(\f_t(\epsilon),\x_t(\epsilon))\big) \right)^{1/q} + (\sigma\RH^p/p)^{1/q} T^{1/q} 
\end{align*}

\end{proof}

\begin{proof}[\textbf{Proof of Proposition~\ref{prop:fin_phi}}]
Fix a $\F\times\X$-valued tree $(\f,\x)$. Note that 
\begin{align*}
&\left|g_t\big(\ell_{\phi_1}(\f_1(\epsilon),\x_1(\epsilon)),\ldots,\ell_{\phi_t}(\f_t(\epsilon),\x_t(\epsilon))\big) \right| \\
& ~~~~~~~\le \left\|\nabla_{t} \Compare^q\big(\epsilon_1 \ell_{\phi_1}(\f_1(\epsilon),\x_1(\epsilon)), \ldots,  \epsilon_{t-1} \ell_{\phi_{t-1}}(\f_{t-1}(\epsilon), \x_{t-1}(\epsilon)) ,0,\ldots,0\big) \right\|_* \left\|\ell_{\phi_t}(\f_{t}(\epsilon), \x_{t}(\epsilon)) \right\|\\
& ~~~~~~~\leq R \cdot \RH
\end{align*}
Using Lemma~\ref{lem:fin},
\begin{align*}
& \En_{\epsilon_{1:T}} \max_{\bphi\in \Phi_T}  \sum_{t=1}^T \epsilon_t g_t\big(\ell_{\phi_1}(\f_1(\epsilon),\x_1(\epsilon)),\ldots,\ell_{\phi_t}(\f_t(\epsilon),\x_t(\epsilon))\big) \\
&~~~~~~~~ \le \sqrt{2 \log(|\Phi_T|) \max_{\bphi \in \Phi_T} \max_{\epsilon \in \{\pm1\}^T} \sum_{t=1}^{T} g_t\big(\ell_{\phi_1}(\f_1(\epsilon),\x_1(\epsilon)),\ldots,\ell_{\phi_t}(\f_t(\epsilon),\x_t(\epsilon))\big)^2} \\
&~~~~~~~~ \le \sqrt{2 \RH^2 R^2 \log(|\Phi_T|) T}
\end{align*}
Now using Lemma \ref{lem:smoothrad} we obtain the desired result.
\end{proof}

\begin{proof}[\textbf{Proof of Corollary~\ref{cor:fin_phi_smooth_sum}}]
	To appeal to Proposition~\ref{prop:fin_phi}, we need to specify smoothness parameters. It can be verified that if $G^q$ is $(\gamma,p)$-smooth in its argument, then $\Compare^q$ is $(\gamma/T^p,p)$-smooth. Furthermore, 
	$$\|\nabla_t \Compare^q(z_1,\ldots,z_T)\|_* \leq \rho/T.$$ The bound of Proposition~\ref{prop:fin_phi} then becomes
$$\Rad_T(\ell,\Phi_T) \le \left(\frac{2 \RH^2 \log(|\Phi_T|)}{T}\right)^{1/2q} + (\gamma\RH^p/p)^{1/q} T^{(1-p)/q} \ .$$
\end{proof}

\begin{proof}[\textbf{Proof of Lemma~\ref{lem:fin_phi_2smooth_sum}}]
	The lemma follows directly from Theorem~\ref{thm:pinelis_union_bound}. To see this, just recall the definition of $\Rad_T(\ell,\Phi_T)$:
$$
\Rad_T(\ell,\Phi_T) = \sup_{\f,\x}
\En_\epsilon \sup_{\phi \in \Phi_T} 
			G\left( 
				\frac{1}{T}\sum_{t=1}^T \epsilon_t \ell_{\phi_t}(\f_t(\epsilon),\x_t(\epsilon)) 
			\right) \ .
$$ 
For any fixed pair $(\f,\x)$ of trees, the argument of $G$ above is the sum of martingale difference sequences coming from a finite family. The step size
bound $B = \RH/T$ and smoothness constant $\sigma = \gamma$. 
\end{proof}

\begin{proof}[\textbf{Proof of Lemma~\ref{lem:fin_phi_p_smooth_sum}}]
		For any $\F$ and $\X$-valued trees $(\f,\x)$,
		\begin{align}\label{eq:fromGtonorm}
		&\Prob_\epsilon\left( \max_{\phi \in \Phi_T} 
			G\left( 
				\frac{1}{T}\sum_{t=1}^T \epsilon_t \ell_{\phi_t}(\f_t(\epsilon),\x_t(\epsilon)) 
			\right) > \theta \right) \le |\Phi_T| \sup_{\z} \Prob_{\epsilon} \left(\left\|\frac{1}{T}\sum_{t=1}^T \epsilon_t \z_t(\epsilon)\right\| > \theta \right)  
		\end{align}
	where supremum is over $\cH$-valued trees such that $\|\z_t(\epsilon)\| \le \RH$. Further, by Lemma~\ref{lem:smoothcon}, for any $$\nu > 8\,c\,\eta\,\gamma^{1/p}\,\log^{3/2}T/T^{1-1/p}\ ,$$
	we have that, 
		\begin{align*}
		&\Prob\left(\left\|\frac{1}{T} \sum_{t=1}^T \epsilon_t \z_t(\epsilon) \right\| >  \frac{c\gamma^{1/p}\RH}{T^{1-1/p}} + \nu \right) \le  2 \exp\left( -\frac{\nu^2 T^{2 - 2/p}}{ 2 c^2 \gamma^{2/p} \RH^2 \log^3 T  } \right)
		\end{align*}
	
		Plugging this into~\eqref{eq:fromGtonorm}, we get,
		\begin{equation*}
		\Prob_\epsilon\left( \max_{\phi \in \Phi_T} 
			G\left( 
				\frac{1}{T}\sum_{t=1}^T \epsilon_t \ell_{\phi_t}(\f_t(\epsilon),\x_t(\epsilon)) 
			\right) >  \frac{c\gamma^{1/p}\RH}{T^{1-1/p}} + \nu \right)
		\le
		2 |\Phi_T|  \exp\left( -\frac{\nu^2 T^{2 - 2/p}}{ 2 c^2 \gamma^{2/p} \RH^2 \log^3 T  } \right)\ .
		\end{equation*}
	
		By a standard argument (e.g. Lemma~\ref{lem:prob_to_exp}) to integrate out the tail, we get
		\begin{align*}
		\En_\epsilon \sup_{\phi \in \Phi_T} 
			G\left( 
		 		\frac{1}{T}\sum_{t=1}^T \epsilon_t \ell_{\phi_t}(\f_t(\epsilon),\x_t(\epsilon)) 
		 	\right) \leq \frac{c\gamma^{1/p}\eta}{T^{1-1/p}}\left(1 +  2\log^{3/2} T \left(\sqrt{\log (2|\Phi_T|)}+1 \right)  \right)\ .
		\end{align*}
		Making trivial over-aaproximations when $T \ge 3 > e$ and $|\Phi_T| > 1$ gives the result.
\end{proof}

\begin{proof}[\textbf{Proof of Theorem~\ref{thm:2smooth_sum_dudley}}]
Define $\beta_0 = 1$ and $\beta_j = 2^{-j}$. For a fixed tree $(\f,\x)$ of depth $T$, let $V_j$ be an  $\ell_\infty$-cover at scale $\beta_j$. For any path $\epsilon \in \{\pm1\}^T$ and any $\bphi \in \Phi_T$, let $\v[\bphi,\epsilon]^j \in V_j$ a $\beta_j$-close element of the cover in the $\ell_\infty$ sense.
Now, for any $\bphi\in\Phi_T$,
\begin{align*}
G\left( \frac{1}{T}\sum_{t=1}^T \epsilon_t \ell_{\phi_t}(\f_t(\epsilon),\x_t(\epsilon)) \right)
&\leq G\left( \frac{1}{T}\sum_{t=1}^T \epsilon_t ( \ell_{\phi_t}(\f_t(\epsilon),\x_t(\epsilon)) - \v[\phi,\epsilon]^{N}_t ) \right) 
+ \sum_{j=1}^{N} G\left( \frac{1}{T} \sum_{t=1}^T \epsilon_t \left( \v[\phi,\epsilon]^{j}_t  - \v[\phi,\epsilon]^{j-1}_t \right) \right)  \\
&\leq \left\| \frac{1}{T}\sum_{t=1}^T \epsilon_t ( \ell_{\phi_t}(\f_t(\epsilon),\x_t(\epsilon)) - \v[\phi,\epsilon]^{N}_t ) \right\| 
+ \sum_{j=1}^{N} G\left( \frac{1}{T} \sum_{t=1}^T \epsilon_t \left( \v[\phi,\epsilon]^{j}_t  - \v[\phi,\epsilon]^{j-1}_t \right) \right)  \\
& \le  \max_{t\in[T]} \left\| \ell_{\phi_t}(\f_t(\epsilon),\x_t(\epsilon))  -  \v[\phi,\epsilon]^N_t \right\| + \sum_{j=1}^{N} G\left(\frac{1}{T} \sum_{t=1}^T \epsilon_t (\v[\phi,\epsilon]^{j}_t - \v[\phi,\epsilon]^{j-1}_t)\right) 
\end{align*}
Thus,
\begin{align*}
\sup_{\bphi\in\Phi_T} G\left( \frac{1}{T}\sum_{t=1}^T \epsilon_t \ell_{\phi_t}(\f_t(\epsilon),\x_t(\epsilon)) \right) & \le \beta_N  + \sup_{\phi \in \Phi_T}\left\{\sum_{j=1}^{N} G\left(\frac{1}{T} \sum_{t=1}^T \epsilon_t (\v[\phi,\epsilon]^{j}_t - \v[\phi,\epsilon]^{j-1}_t)\right) \right\}
\end{align*}
We now proceed to upper bound the second term. Consider all possible pairs of $\v^s\in V_j$ and $\v^r\in V_{j-1}$, for $1\leq s\leq |V_j|$, $1\leq r \leq |V_{j-1}|$, where we assumed an arbitrary enumeration of elements. For each pair $(\v^s,\v^r)$, define a real-valued tree $\w^{(s,r)}$ by
\begin{align*}
\w^{(s,r)}_t(\epsilon) = 
	\begin{cases} 
	\v^s_t(\epsilon)-\v^r_t(\epsilon) & \text{if there exists } \bphi \in \Phi_T \mbox{ s.t. } \v^s = \v[\bphi,\epsilon]^{j}, \v^r = \v[\bphi,\epsilon]^{j-1} \\
	0 &\text{otherwise.}
	\end{cases}
\end{align*}
for all $t\in [T]$ and $\epsilon\in\{\pm1\}^T$. It is crucial that $\w^{(s,r)}$ can be non-zero only on those paths $\epsilon$ for which $\v^s$ and $\v^r$ are indeed the members of the covers (at successive resolutions) close in the $\ell_\infty$ sense {\em to some} $\bphi \in \Phi_T$. It is easy to see that $\w^{(s,r)}$ is well-defined. Let the set of trees $W_j$ be defined as
\begin{align*}
	W_j = \left\{ \w^{(s,r)}: 1\leq s\leq |V_j|, 1\leq r \leq |V_{j-1}| \right\}
\end{align*}
Using the above notations we see that

\begin{align}\label{eq:2smooth_dudsimp1}
	\Es{\epsilon}{\sup_{\phi \in \Phi_T} G\left( \frac{1}{T}\sum_{t=1}^T \epsilon_t \ell_{\phi_t}(\f_t(\epsilon),\x_t(\epsilon)) \right)} 
&\leq  \beta_N  + \Es{\epsilon}{\sup_{\phi \in \Phi_T}\left\{\sum_{j=1}^{N} G\left(\frac{1}{T} \sum_{t=1}^T \epsilon_t (\v[\phi,\epsilon]^{j}_t - \v[\phi,\epsilon]^{j-1}_t)\right) \right\}} \notag\\
&  \le \beta_N  + \Es{\epsilon}{\sum_{j=1}^{N} \sup_{\w^j \in W_j} G\left(\frac{1}{T} \sum_{t=1}^T \epsilon_t \w^j_t(\epsilon)\right)}
\end{align}

From the way the trees in $W_j$ are constructed, it is easy to see that $\max_{t\in[T]} \|\w^j_t(\epsilon)\| \leq 3\beta_j$ for any $\w^j\in W^j$ and any path $\epsilon$. Using Theorem~\ref{thm:pinelis_union_bound}, we get
\begin{align*}
	\Es{\epsilon}{\sup_{\phi \in \Phi_T} G\left( \frac{1}{T}\sum_{t=1}^T \epsilon_t \ell_{\phi_t}(\f_t(\epsilon),\x_t(\epsilon)) \right)} 
&\leq \beta_N + \sum_{j=1}^N 6\beta_j \sqrt{ \frac{\gamma \log(2|W_j|)}{T} } \\ 
&\leq \beta_N + \sum_{j=1}^N 6\beta_j \sqrt{ \frac{\gamma \log(2|V_j|\cdot|V_{j-1}|)}{T} } \\ 
&\leq \beta_N + \frac{12\sqrt{\gamma}}{\sqrt{T}} \sum_{j=1}^N \beta_j \sqrt{ \log(|V_j|) } \\
&\leq \beta_N + \frac{24\sqrt{\gamma}}{\sqrt{T}} \sum_{j=1}^N (\beta_j - \beta_{j+1}) \sqrt{ \log \cN_\infty(\beta_j,\Phi_T,T) } \ . 
\end{align*}
Using standard arguments to move from the discrerized sum to an integral, this gives the bound,
$$
\inf_{\alpha}\ 4\alpha + \frac{24\sqrt{\gamma}}{\sqrt{T}} \int_{\alpha}^1 \sqrt{  \log \cN_\infty(\beta,\Phi_T,T) } d\beta \ .
$$
\end{proof}

\begin{proof}[\textbf{Proof of Corollary~\ref{cor:simple_consequences}}]
	The first statement is trivially verified. In fact, for this to hold we only require that $\Compare$ is subadditive, affine in its arguments, and $\Compare (0,\ldots,0) = 0$. Indeed, the expectations can be sequentially moved inside of $\Compare$, making the coordinates of $\Compare$ zero, and making the suprema over the distributions irrelevant. 
	
	For the second claim, consider the second term in \eqref{eq:three_term_decomposition}, specialized to the case of departure mappings:
	\begin{align}
		\label{eq:upper_bound_on_value_phi_reg}
		\sup_{p_1} \inf_{q_1}  \ldots  \sup_{p_T} \inf_{q_T} \sup_{\bphi\in \Phi_T} \Eunder{x_{1:T} \sim p_{1:T}}{f_{1:T} \sim q_{1:T}} \left\{ \frac{1}{T}\sum_{t=1}^T \ell(f_t,x_t) - \ell(\phi_t(f_t),x_t) \right\}
	\end{align}
	Pick a particular (sub)optimal response $q_t$ which puts all mass on $f_t^* = \arg\min_{f\in\F} \En_{x\sim p_t} \ell(f, x).$ It follows that $\ell(f_t,x_t) - \ell(\phi_t(f_t),x_t) \leq 0$, ensuring that the quantity in \eqref{eq:upper_bound_on_value_phi_reg} is non-positive. 
	
	The third claim is a straightforward consequence of Theorem~\ref{thm:2smooth_sum_dudley}. Indeed, $\cH \subset [-1,1]$ and $G(x) = |x|$ which is non-negative, $0$ at $0$,
Lipschitz and $G^2$ is $(2,2)$-smooth.
\end{proof}

\begin{proof}[\textbf{Proof of Lemma~\ref{lem:accum_pts}}]
Fix an $(\F\times\X)$-valued tree $(\f,\x)$ of depth $T$. Let $(i_0,\ldots,i_k)$ be the sequence which defines intervals of time-invariant mappings for the sequence $(\phi_1,\ldots,\phi_T)$. Fix $\epsilon\in\{\pm 1\}^T$. Let $\v^{i_0},\ldots,\v^{i_k} \in V$ be the elements of the $L_\infty$ cover closest to $\phi_{i_0},\ldots,\phi_{i_k}$, respectively, on the path $\epsilon$. That is, for any $a\in \{i_0,\ldots,i_k\}$,
$$\max_t \|\ell_{\phi_{a}}(\f_t(\epsilon),\x_t(\epsilon)) - \v^{a}_t(\epsilon) \| \leq \alpha.$$
By our assumption, on any interval $I$, defined by the endpoints $a=i_j$ and $b=i_{j+1}$,
$$\max_{t\in \{a, \ldots, b-1\}} \|\ell_{\phi_{a}}(\f_t(\epsilon),\x_t(\epsilon)) - \ell_{\phi_{t}}(\f_t(\epsilon),\x_t(\epsilon)) \| \leq \alpha, $$
Hence,
$$\max_{t\in \{a, \ldots, b-1\}} \|\ell_{\phi_{t}}(\f_t(\epsilon),\x_t(\epsilon)) - \v^{a}_t(\epsilon) \| \leq  2\alpha $$
Denoting by $a(t)\in\{i_0,\ldots,i_k\}$ the left endpoint of an interval to which $t$ belongs,
$$\max_{t\in\{1,\ldots,T\}} \|\ell_{\phi_{t}}(\f_t(\epsilon),\x_t(\epsilon)) - \v^{a(t)}_t(\epsilon) \| \leq  2\alpha $$
It is then clear that to construct a $2\alpha$-cover for $\Phi^{k,\alpha}_T$ in $L_\infty$ norm, it is enough to concatenate trees in $V$. More precisely, this is done as follows. Construct a set $V^{k}$ of $\cH$-valued trees as
$$ V^{k} = \{ \v'=\v' \left(\v^{0},\ldots,\v^{k}, i_0,\ldots,i_k\right): 1 = i_0 \leq i_1\leq \ldots \leq i_k \leq T,~~ \v^{0},\ldots,\v^{k}\in V \}$$ 
and $\v'=\v' \left(\v^{0},\ldots,\v^{k}, i_0,\ldots,i_k \right)$ is defined as a sequence of $T$ mappings 
$$ \v'_t(\epsilon) = \v^{a(t)}_t (\epsilon) \ \ \ \ \  t\in I_{a(t)}$$
for any $\epsilon\in\{\pm 1\}^T$. Here $I_a = \{i_j,\ldots,i_{j+1}-1\}$ and $a(t)$ is the index of the  interval to which $t$ belongs. In plain words, we consider all ways of partitioning $\{1,\ldots,T\}$ into $k+1$ intervals and defining a new set of trees out of $V$ in such a way that within the interval, the values are given by a fixed tree from $V$. As before, it is clear that 
$$\N_\infty(2\alpha, \Phi^{k,\alpha}_T, T) =  |V^{k}| \leq {T \choose k} \cdot \N_\infty(\alpha, \Phi, T)^{k+1},$$
providing a control on the complexity of $\Phi^{k,\alpha}_T$.
\end{proof}

\begin{lemma}
\label{lem:concave}
Let $\F$ be the probability simplex in any dimension. Let $\|\cdot\|$ be any norm. The function
\[
	x \mapsto \inf_{f \in \F} \| f \odot x \|\ ,
\]
defined on the positive orthant, is concave.
\end{lemma}
\begin{proof}
Since the function above is absolutely homogeneous and continuous, all we need to prove is
\[
	\inf_{f \in \F} \| f \odot (x+y) \| \ge \inf_{f\in\F} \| f \odot x \| + \inf_{f \in \F} \| f \odot y \|\ .
\]
for arbitrary $x,y$. That is, for arbitrary $f, x, y$,
\[
	\| f \odot (x+y) \| \ge \inf_{f\in\F} \| f \odot x \| + \inf_{f \in \F} \| f \odot y \|\ .
\]
Define $h,g \in \F$ as follows:
\begin{align*}
g_i &= \frac{f_i (1 + y_i/x_i) }{ Z_g }
&
h_i &= \frac{ f_i (1 + x_i/y_i) }{ Z_h }\ ,
\end{align*}
where
\begin{align*}
Z_g &= \sum_i f_i (1 + y_i/x_i) 
&
Z_h &= \sum_i f_i (1 + x_i/y_i) \ .
\end{align*}
Now, as we show below, $1/Z_g + 1/Z_h \le 1$. Thus,
\begin{align*}
\| f \odot (x+y) \| &\ge \frac{1}{Z_g} \| f \odot (x+y) \| + \frac{1}{Z_h} \| f \odot (x + y) \| \\
&= \| g \odot x \| + \| h \odot y \| \\
&\ge \inf_{f \in \F} \| f \odot x \| + \inf_{f \in \F} \| f \odot y \| \ .
\end{align*}
To finish the proof, note that, by Cauchy-Schwarz,
\[
	\left( \sum_i f_i (1+ y_i/x_i) \right) \cdot \left( \sum_i f_i \frac{x_i}{x_i+y_i} \right)
	\ge \left( \sum_i f_i \right)^2 = 1 \ .
\]
This shows,
\[
	\frac{1}{Z_g} \le \sum_i f_i \frac{x_i}{x_i+y_i} \ .
\]
Similarly, we get
\[
	\frac{1}{Z_h} \le \sum_i f_i \frac{y_i}{x_i + y_i} \ .
\]
Adding them, we get
\[
	\frac{1}{Z_g} + \frac{1}{Z_h} \le \sum_i f_i = 1
\]
as claimed. This completes the proof. 
\end{proof}

\begin{proof}[\textbf{Proof of Proposition~\ref{prop:equalizer}}]
Consider any equalizer strategy $\{p^*_t\}$ for the adversary. Note that
\begin{align*}
	\Val_T(\ell,\Phi_T) &= \inf_{q_1}\sup_{p_1} \Eunder{x_1\sim p_1}{f_1\sim q_1} \ldots  \inf_{q_T}\sup_{p_T} \Eunder{x_T\sim p_T}{f_T\sim q_T}  
	\sup_{\bphi\in \Phi_T} \left\{ \Compare(\ell(f_1,x_1), \ldots, \ell(f_T, x_T)) - \Compare(\ell_{\phi_1 }( f_1,x_1), \ldots, \ell_{\phi_T}(f_T, x_T))\right\}\\
	& \ge \inf_{q_1} \Eunder{f_1 \sim q_1}{x_1 \sim p^*_1} \inf_{q_2} \Eunder{f_2 \sim q_2}{x_2 \sim p^*_2} \ldots \inf_{q_T} \Eunder{f_T \sim q_T}{x_T \sim p^*_T} \left\{ \Compare\left(\ell(f_1,x_1), \ldots, \ell(f_T,x_T)\right) - \inf_{\phi \in \Phi_T} \Compare\left(\ell_{\phi_1}(f_1,x_1), \ldots, \ell_{\phi_T}(f_T,x_T) \right) \right\}\\
	& = \Eu{x_1 \sim p_1} \ldots  \Eu{x_T \sim p_T} \left\{\Compare\left(\ell(f,x_1), \ldots, \ell(f,x_T)\right) - \inf_{\phi \in \Phi_T}  \Compare\left(\ell_{\phi_1}(f,x_1), \ldots, \ell_{\phi_T}(f,x_T)\right) \right\}
\end{align*}
where $f \in \F$ is any arbitrary choice fixed before starting the game and $p_t = p^*_t\left(\{f_s = f, x_s\right\}_{s=1}^{t-1})$ is defined by the equalizer strategy.
\end{proof}

\begin{lemma}\label{lem:lincvx}
For any departure mapping $\Phi_T$ and any $L > 0$ we have that 
$$
\Val_T(\cC_\F,\F,\Phi_T) = \Val_T(\cL_\F,\F,\Phi_T)
$$
\end{lemma}
\begin{proof}
Note that for any convex $x_1,\ldots,x_T$ we have that
\begin{align}\label{eq:cvxlin}
\sum_{t=1}^T x_t(f_t) - \inf_{\phi \in \Phi} \sum_{t=1}^T x_t(\phi \circ f_t) & = \sup_{\phi \in \Phi} \sum_{t=1}^T \left( x_t(f_t) - x_t(\phi \circ f_t) \right) \nonumber \\
& \le \sup_{\phi \in \Phi} \sum_{t=1}^T \inner{\nabla x_t(f_t), f_t - \phi \circ f_t } \nonumber \\
& =\sum_{t=1}^T  \inner{\nabla x_t(f_t), f_t} - \inf_{\phi \in \Phi}  \sum_{t=1}^T \inner{\nabla x_t(f_t), \phi \circ f_t}
\end{align}

For any adversary strategy $x^* = (x^*_1,\ldots,x^*_T)$ where each $x^*_t : \F^t \mapsto \X$ and any player strategy  $f^* = (f^*_1,\ldots,f^*_T)$ where each $f^*_t : \X^{t-1} \mapsto \F$, by Equation \eqref{eq:cvxlin}
we have that
\begin{align*}
\sum_{t=1}^T \inner{\nabla x_t(f_t), f_t} - \inf_{\phi \in \Phi} \sum_{t=1}^T \inner{\nabla x_t(f_t), \phi \circ f_t} \ge  \sum_{t=1}^T x_t(f_t) - \inf_{\phi \in \Phi}  \sum_{t=1}^T x_t(\phi \circ f_t) 
\end{align*}
where in the above, $f_t = f^*_t(\inner{\nabla x_1(f_1), \cdot},\ldots,\inner{\nabla x_{t-1}(f_{t-1}),\cdot})$ and $x_t = x^*_t(f_1,\ldots,f_t)$. Now if we take $f^*$ and $x^*$ to be the minimax optimal strategies then we see that 
\begin{align*}
 \Val_T(\cL_\F,\F,\Phi_T) & \ge \sum_{t=1}^T \inner{\nabla x_t(f_t), f_t} - \inf_{\phi \in \Phi} \sum_{t=1}^T \inner{\nabla x_t(f_t), \phi \circ f_t} \\
& \ge  \sum_{t=1}^T x_t(f_t) - \inf_{\phi \in \Phi}  \sum_{t=1}^T x_t(\phi \circ f_t) \\
& \ge \Val_T(\cC_\F,\F,\Phi_T)
\end{align*}

Thus we see that the value of the linear game upper bounds the value of the Lipschitz convex game. In fact the above argument shows that any strategy that provides vanishing regret guarantee against linear adversary provides  vanishing regret gaurantee (with same rate) against convex Lipschitz adversary. This means that all that one needs to do to solve convex Lipschitz optimization optimally is to be able to solve online linear optimization optimally and also be able to calculate sub-gradient of a given function at any desired point.

Further since the set of linear functions is a subset of the set of convex Lipschitz functions we can conclude that 
$$
\Val_T(\cL_\F,\F,\Phi_T) \le \Val_T(\cC_\F,\F,\Phi_T) 
$$
Hence we conclude the required statement that the value of the linear game is equal to the value of the convex Lipschitz game.
\end{proof}

\begin{lemma}
Consider a game where player plays from set $\F$ adversary from set $\X$ and we are give a linear $\Compare$, loss $\ell$ and transformation set $\Phi_T$. Assume that there exists a set $\X'$, loss function $\ell'$ and transformation set $\Phi'_T$ such that for any $\bphi \in \Phi_T$ there exists $\phi' \in \Phi'_T$ such that for $x \in \X$ and $f \in \F$ there exists an $x' \in \X'$ such that for any $t \in [T]$, 
\begin{align*}
\ell(f,x) - \ell_{\phi_t}(f,x) \le \ell'(f,x') - \ell_{\phi'_t}(f,x')
\end{align*}
In that case we can conclude that value of the first game is bounded by value of the second game played with $\F$, $\X'$, $\Compare$, $\ell'$, $\Phi'_T$, that is
$$
 \Val_{T}(\ell,\Phi_T,\F,\X)  \le  \Val_{T}(\ell',\Phi'_T,\F,\X')
$$
\end{lemma}
\begin{proof}
By assumption that for any $\bphi \in \Phi_T$ there exists $\phi' \in \Phi'_T$ such that for $x \in \X$ and $f \in \F$ there exists an $x' \in \X'$ such that for any $t \in [T]$, 
\begin{align*}
\ell(f,x) - \ell_{\phi_t}(f,x) \le \ell'(f,x') - \ell_{\phi'_t}(f,x')
\end{align*}
We can conclude that since $\Compare$ is linear, for any $\bphi \in \Phi_T$ there exists $\phi' \in \Phi'_T$ such that for any $f_1,\ldots,f_T$ and $x_1,\ldots,x_T$ we have that for the corresponding $x'_1,\ldots,x'_T$ given by our assumption, we have that 
\begin{align*}
& \Compare(\ell(f_1,x_1),\ldots,\ell(f_T,x_T)) - \Compare(\ell_{\phi_1}(f_1,x_1),\ldots,\ell_{\phi_T}(f_T,x_T)) \\
& ~~~~~~\le \Compare(\ell'(f_1,x'_1),\ldots,\ell'(f_T,x'_T)) - \Compare(\ell_{\phi'_1}(f_1,x'_1),\ldots,\ell_{\phi'_T}(f_T,x'_T))
\end{align*}
Hence we can conclude that 
\begin{align*}
\sup_{\bphi \in \Phi_T}&\left\{ \Compare(\ell(f_1,x_1),\ldots,\ell(f_T,x_T)) - \Compare(\ell_{\phi_1}(f_1,x_1),\ldots,\ell_{\phi_T}(f_T,x_T))\right\} \\
& \le \sup_{\phi' \in \Phi'_T} \left\{ \Compare(\ell'(f_1,x'_1),\ldots,\ell'(f_T,x'_T)) - \Compare(\ell_{\phi'_1}(f_1,x'_1),\ldots,\ell_{\phi'_T}(f_T,x'_T)) \right\}
\end{align*}
Now say $q^* = (q^*_1,\ldots,q^*_T)$ where each $q^*_t : \left(\F \times \X'\right)^{t-1} \mapsto \Delta(\F)$ is the minimax optimal strategy for the player while playing the second game. Also let $p^* = (p^*_1,\ldots,p^*_T)$ where each $p^*_t : \left(\F \times \X\right)^{t} \mapsto \Delta(\X')$ be the minimax optimal strategy for the player while playing the first game. In this case we see that
\begin{align*}
 \Val_{T}(\ell,\Phi_T,\F,\X) & =  \Eunder{x_1 \sim p^*_1}{f_1 \sim q^*_1} \ldots \Eunder{x_T \sim p^*_T}{f_T \sim q^*_T} \sup_{\bphi \in \Phi_T}\left\{ \Compare(\ell(f_1,x_1),\ldots,\ell(f_T,x_T)) - \Compare(\ell_{\phi_1}(f_1,x_1),\ldots,\ell_{\phi_T}(f_T,x_T))\right\}\\
 & \le \Eunder{x'_1 \sim p^*_1}{f_1 \sim q^*_1} \ldots \Eunder{x'_T \sim p^*_T}{f_T \sim q^*_T} \sup_{\phi' \in \Phi'_T} \left\{ \Compare(\ell'(f_1,x'_1),\ldots,\ell'(f_T,x'_T)) - \Compare(\ell_{\phi'_1}(f_1,x'_1),\ldots,\ell_{\phi'_T}(f_T,x'_T)) \right\}\\
 & \le  \Val_{T}(\ell',\Phi'_T,\F,\X')
\end{align*}

\end{proof}

\begin{proof}[\textbf{Proof of Theorem~\ref{thm:adaptive}}]
We start by applying the Triplex inequality in Theorem \ref{thm:main} along with Theorem \ref{thm:rad} we get that :
\begin{align*}
\Val_T  \le &\ 2\Rad_T(\ell,\mathcal{I},B) + \sup_{p_1} \inf_{q_1} \ldots \sup_{p_T} \inf_{q_T} \sup_{\phi \in \Phi_T} \left\{ - \Eunder{x_{1:T} \sim p_{1:T}}{f_{1:T} \sim q_{t:T}} \Compare(\ell_{\phi_1}(f_1,x_1), \ldots, \ell_{\phi_T}(f_T,x_T)) \right\}  + 2\Rad_T(\ell,\Phi_T,B)\\
& =  0 + \sup_{p_1} \inf_{q_1} \ldots \sup_{p_T} \inf_{q_T} \sup_{\phi \in \Phi_T} \left\{ \Eunder{x_{1:T} \sim p_{1:T}}{f_{1:T} \sim q_{t:T}} \frac{1}{T}\sum_{t=1}^T \left( \loss(f_t,x_t) - \loss(\psi_t \circ f_t,x_t) \right) \ind{t \in I_t}
\right\} + 2\Rad_T(\ell,\Phi_T,B)
\end{align*}
where the last inequality above is because the first term of the triplex inequality is $0$ as $\Compare$ is linear (see Corollary~\ref{cor:simple_consequences}). If we use $q_t$ to be point mass on $f_t = \argmin{f \in \F} \Es{x_t \sim p_t}{\loss(f,x_t)}$ we see that the second term of the triplex inequality above is bounded above by $0$. Hence we can conclude that
\begin{align*}
\Val_T & \le 2\Rad_T(\ell,\Phi_T,B) = 2 \sup_{\f , \x} \Es{\epsilon}{\sup_{\psi \in \Psi, [r,s] \subseteq [T]} \frac{1}{T} \sum_{t=1}^T \epsilon_t\left( \loss(\f_t(\epsilon),\x_t(\epsilon)) - \loss(\psi \circ \f_t(\epsilon),\x_t(\epsilon)) \right) \ind{t \in [r,s]}  }
\end{align*}
To bound the above we use Corollary \ref{cor:simple_consequences} (noting that $\ell_{\phi_t}(f,x) \in [-2,2]$) to get
\begin{align*}
\Val_T & \le 8 \inf_{\alpha > 0 }\left\{\alpha + 6 \sqrt{2} \int_{\alpha}^2 \sqrt{\frac{\log\ \mc{N}_\infty(\delta,\Phi_T,T)}{T}} d\delta \right\}\\
& \le 8 \inf_{\alpha > 0 }\left\{\alpha + 6 \sqrt{2} \int_{\alpha}^2 \sqrt{\frac{\log\ \mc{N}_\infty(\delta,\Psi,T) + \log(|\mc{I}_T|)}{T}} d\delta \right\} \ .
\end{align*}
Now note that $|\mc{I}_T| \le T^2$ and so we get that 
$$
\Val_T  \le 8 \inf_{\alpha > 0 }\left\{\alpha + 6 \sqrt{2} \int_{\alpha}^2 \sqrt{\frac{\log\ \mc{N}_\infty(\delta,\Psi,T)}{T}} d\delta \right\} + 96 \sqrt{\frac{\log\ T}{T}} \ .
$$

We conclude that whenever covering number of $\Psi$ can be bounded appropriately, adaptive regret can be bounded at the expense of an extra $O\left(\sqrt{\frac{\log\ T}{T}}\right)$ term. 
\end{proof}


\begin{proof}[\textbf{Proof of Theorem~\ref{thm:main-prob}}]
	For any $\theta \ge 0$, the value of the game $\Val^\theta_T(\ell,\Phi_T)$, defined in \eqref{def:theta-value}, is
	{\small
	\begin{align*}
		&\Val^\theta_T(\ell,\Phi_T) \\
		&= \inf_{q_1} \sup_{p_1} \Eunder{x_1\sim p_1}{f_1\sim q_1} \ldots \inf_{q_T} \sup_{p_T} \Eunder{x_T\sim p_T}{f_T\sim q_T}  \left[
		\ind{\sup_{\bphi\in \Phi_T}\left\{ \Compare(\ell(f_1,x_1), \ldots, \ell(f_T, x_T)) -  \Compare(\ell_{\phi_1}(  f_1,x_1), \ldots, \ell_{\phi_T}(f_T, x_T))\right\} > \theta} \right]\\
		&= \sup_{p_1}\inf_{q_1} \Eunder{x_1\sim p_1}{f_1\sim q_1} \ldots  \sup_{p_T}\inf_{q_T} \Eunder{x_T\sim p_T}{f_T\sim q_T}  \left[
		\ind{\sup_{\bphi\in \Phi_T} \left\{ \Compare(\ell(f_1,x_1), \ldots, \ell(f_T, x_T)) - \Compare(\ell_{\phi_1 }( f_1,x_1), \ldots, \ell_{\phi_T}(f_T, x_T))\right\}> \theta}\right]
	\end{align*} 
	}
	via an application of the minimax theorem. Adding and subtracting terms to the expression above leads to
	{\small
	\begin{align*}
		\Val^\theta_T(\ell, \Phi_T) &= \sup_{p_1}\inf_{q_1}\Eunder{x_1\sim p_1}{f_1 \sim q_1} \ldots  \sup_{p_T}\inf_{q_T} \Eunder{x_T\sim p_T}{f_T \sim q_T}
		\hspace{0.4cm} \left[ \mathbf{1}\left\{ \Compare(\ell(f_1,x_1), \ldots, \ell(f_T, x_T)) -  \Compare(\ell(q_1,p_1), \ldots, \ell(q_T, p_T)) \right. \right.\\
		&\left. \left.\hspace{0.4in}+\sup_{\bphi\in \Phi_T} \left\{  \Compare(\ell(q_1,p_1), \ldots, \ell(q_T, p_T)) - \Compare(\ell_{\phi_1}(f_1,x_1), \ldots, \ell_{\phi_T}(f_T,x_T))\right\} > \theta \right\} \right]\\
		&\leq \sup_{p_1}\inf_{q_1}\Eunder{x_1\sim p_1}{f_1 \sim q_1} \ldots  \sup_{p_T}\inf_{q_T} \Eunder{x_T\sim p_T}{f_T \sim q_T}
		\hspace{0.4cm} \left[ \mathbf{1}\left\{\Compare(\ell(f_1,x_1), \ldots, \ell(f_T, x_T)) -  \Compare(\ell(q_1,p_1), \ldots, \ell(q_T, p_T)) \right. \right.\\
		&\left. \left.\hspace{0.4in}+\sup_{\bphi\in \Phi_T}  \Big\{ \Compare(\ell(q_1,p_1), \ldots, \ell(q_T, p_T)) -   \Compare(\ell_{\phi_1}(q_1,p_1), \ldots, \ell_{\phi_T}(q_T, p_T))\Big\} \right. \right.\\
		&\left. \left.\hspace{0.4in} + \sup_{\bphi\in \Phi_T} \left\{  \Compare(\ell_{\phi_1}(q_1,p_1), \ldots, \ell_{\phi_T}(q_T, p_T)) -  \Compare(\ell_{\phi_1}(f_1,x_1), \ldots, \ell_{\phi_T}(f_T,x_T))\right\} > \theta \right\} \right]\\
		&\leq \sup_{p_1}\inf_{q_1}\Eunder{x_1\sim p_1}{f_1 \sim q_1} \ldots  \sup_{p_T}\inf_{q_T} \Eunder{x_T\sim p_T}{f_T \sim q_T}
		\hspace{0.4cm} \left[ \mathbf{1}\left\{\Compare(\ell(f_1,x_1), \ldots, \ell(f_T, x_T)) - \Compare(\ell(q_1,p_1), \ldots, \ell(q_T, p_T)) > \theta/3\right\} \right.\\
		&\left. \hspace{0.4in}+ \mathbf{1}\left\{ \sup_{\bphi\in \Phi_T}  \Big\{ \Compare(\ell(q_1,p_1), \ldots, \ell(q_T, p_T)) -   \Compare(\ell_{\phi_1}(q_1,p_1), \ldots, \ell_{\phi_T}(q_T, p_T))\Big\} > \theta/3 \right\} \right.\\
		&\left. \hspace{0.4in} + \mathbf{1}\left\{ \sup_{\bphi\in \Phi_T}  \left\{  \Compare(\ell_{\phi_1}(q_1,p_1), \ldots, \ell_{\phi_T}(q_T, p_T)) -  \Compare(\ell_{\phi_1}(f_1,x_1), \ldots, \ell_{\phi_T}(f_T,x_T))\right\} > \theta/3 \right\} \right]
		\end{align*} }
	At this point, we would like to break up the expression into three terms. To do so, notice that expectation is linear and $\sup$ is a convex function, while for the infimum,
	$$\inf_a \left[C_1(a)+C_2(a)+C_3(a)\right] \leq \left[\sup_a C_1(a)\right] + \left[\inf_a C_2(a)\right] + \left[\sup_a C_3(a)\right] $$
	for functions $C_1,C_2,C_3$. We use these properties of $\inf$, $\sup$, and expectation, starting from the inside of the nested expression and splitting the expression in three parts. We arrive at
	{\small
	\begin{align*}
		&\Val^\theta_T(\ell, \Phi_T)\\ 
		&\leq \sup_{p_1}\sup_{q_1}\Eunder{x_1\sim p_1}{f_1 \sim q_1} \ldots  \sup_{p_T}\sup_{q_T} \Eunder{x_T\sim p_T}{f_T \sim q_T}
		 \left[ \ind{\Compare(\ell(f_1,x_1), \ldots, \ell(f_T, x_T)) -  \Compare(\ell(q_1,p_1), \ldots, \ell(q_T, p_T)) > \theta/3} \right]\\
		&+\sup_{p_1}\inf_{q_1}\Eunder{x_1\sim p_1}{f_1 \sim q_1} \ldots  \sup_{p_T}\inf_{q_T} \Eunder{x_T\sim p_T}{f_T \sim q_T}
	 \left[ \ind{\sup_{\bphi\in \Phi_T}   \left\{ \Compare(\ell(q_1,p_1), \ldots, \ell(q_T, p_T)) - \Compare(\ell_{\phi_1}(q_1,p_1), \ldots, \ell_{\phi_T}(q_T, p_T))\right\}> \theta/3} \right]\\
		&+\sup_{p_1}\sup_{q_1}\Eunder{x_1\sim p_1}{f_1 \sim q_1} \ldots  \sup_{p_T}\sup_{q_T} \Eunder{x_T\sim p_T}{f_T \sim q_T} \left[ \ind{\sup_{\bphi\in \Phi_T} \left\{   \Compare(\ell_{\phi_1}(q_1,p_1), \ldots, \ell_{\phi_T}(q_T, p_T)) -  \Compare(\ell_{\phi_1}(f_1,x_1), \ldots, \ell_{\phi_T}(f_T,x_T))\right\}> \theta/3} \right]
	\end{align*}
	}
	As mentioned in the corresponding proof of Theorem~\ref{thm:main}, the replacement of infima by suprema in the first and third terms appears to be a loose step and, indeed, one can pick a particular response strategy $\{q^*_t\}$ instead of passing to the supremum. 

	Consider the second term in the above decomposition. Clearly,
	\begin{align*}
		&\sup_{p_1}\inf_{q_1}\Eunder{x_1\sim p_1}{f_1 \sim q_1} \ldots  \sup_{p_T}\inf_{q_T} \Eunder{x_T\sim p_T}{f_T \sim q_T}
		 \left[ \ind{\sup_{\bphi\in \Phi_T}   \Compare(\ell(q_1,p_1), \ldots, \ell(q_T, p_T)) - \Compare(\ell_{\phi_1}(q_1,p_1), \ldots, \ell_{\phi_T}(q_T, p_T))> \theta/3  } \right]\\
		&=\sup_{p_1}\inf_{q_1} \ldots  \sup_{p_T}\inf_{q_T}\ 
		 \ind{\sup_{\bphi\in \Phi_T}    \Compare(\ell(q_1,p_1), \ldots, \ell(q_T, p_T)) -  \Compare(\ell_{\phi_1}(q_1,p_1), \ldots, \ell_{\phi_T}(q_T, p_T)) > \theta/3 }
	\end{align*}
	because the objective does not depend on the random draws. 
\end{proof}

\begin{proof}[\textbf{Proof of Theorem~\ref{thm:symmetrization_probability}}]
	Assume that $\Compare$ is sub-additive (the other case is identical).
	\begin{align*}
		&\Compare(\ell_{\phi_1}(q_1,p_1), \ldots, \ell_{\phi_T}(q_T, p_T)) ~-~  \Compare(\ell_{\phi_1}(f_1,x_1), \ldots, \ell_{\phi_T}(f_T, x_T)) \\
		&\leq \Compare(\ell_{\phi_1}(q_1,p_1) - \ell_{\phi_1}(f_1,x_1), \ldots, \ell_{\phi_T}(q_T, p_T) - \ell_{\phi_T}(f_T,x_T))
	\end{align*}
	
By our assumption we have that for any distribution $\mbf{D}$ and any fixed $\phi \in \Phi_T$,
\begin{align}\label{eq:fixphi}
\mathbb{P}_\mbf{D}\left(    \Compare\left(\ell_{\phi_1}(q_1,p_1) - \ell_{\phi_1}(f'_1,x'_1), \ldots, \ell_{\phi_T}(q_T, p_T) - \ell_{\phi_T}(f'_T,x'_T)\right) \le \theta/6\ \middle| \ (f_1,x_1),\ldots, (f_T,x_T)  \right) \ge \frac{1}{2}
\end{align}
For a given $(f_1,x_1),\ldots,(f_T,x_T)$, let $\phi^* \in \Phi$ be the transformation defined as $$ \phi^* = \argmax{\phi \in \Phi_T}\Compare\left(\ell_{\phi_1}(q_1,p_1) - \ell_{\phi_1}(f_1,x_1), \ldots, \ell_{\phi_T}(q_T, p_T) - \ell_{\phi_T}(f_T,x_T)\right)$$
(We are assuming for simplicity that the supremum is achieved; otherwise, we can easily modify arguments to take care of it). Since $\phi^*$ is fixed given $(f_1,x_1),\ldots,(f_T,x_T)$, using Equation \eqref{eq:fixphi} we get
$$
\frac{1}{2} \le \mathbb{P}_\mbf{D}\left(    \Compare\left(\ell_{\phi^*_1}(q_1,p_1) - \ell_{\phi^*_1}(f'_1,x'_1), \ldots, \ell_{\phi^*_T}(q_T, p_T) - \ell_{\phi^*_T}(f'_T,x'_T)\right) \le \theta/6\ \middle| \ (f_1,x_1),\ldots, (f_T,x_T)  \right)  
$$
Define set 
$$A = \left\{\left((f_1,x_1),\ldots,(f_T,x_T)\right)\ \middle|\ \sup_{\phi \in \Phi_T} \Compare\left(\ell_{\phi_1}(q_1,p_1) - \ell_{\phi_1}(f_1,x_1), \ldots, \ell_{\phi_T}(q_T, p_T) - \ell_{\phi_T}(f_T,x_T)\right) > \theta/3 \right\}.$$ 
Since the above inequality holds for any $(f_1,x_1),\ldots,(f_T,x_T)$, we assert that
$$
\frac{1}{2} \le \mathbb{P}_\mbf{D}\left(    \Compare\left(\ell_{\phi^*_1}(q_1,p_1) - \ell_{\phi^*_1}(f'_1,x'_1), \ldots, \ell_{\phi^*_T}(q_T, p_T) - \ell_{\phi^*_T}(f'_T,x'_T)\right) \le \theta/6\ \middle| \ \left((f_1,x_1),\ldots, (f_T,x_T)\right) \in A \right)  
$$
It then follows that
{\small
\begin{align*}
& \frac{1}{2}  \mathbb{P} \left(\sup_{\phi \in \Phi_T} \Compare(\ell_{\phi_1}(q_1,p_1) - \ell_{\phi_1}(f_1,x_1), \ldots, \ell_{\phi_T}(q_T, p_T) - \ell_{\phi_T}(f_T,x_T)) > \frac{ }{ } \theta/3 \right)  \\
& ~ \le  \mathbb{P} \left( \sup_{\phi \in \Phi_T} \Compare(\ell_{\phi_1}(q_1,p_1) - \ell_{\phi_1}(f_1,x_1), \ldots, \ell_{\phi_T}(q_T, p_T) - \ell_{\phi_T}(f_T,x_T)) > \frac{ }{ } \theta/3 \right)  \\
& ~ ~~~ \times \mathbb{P}\left(\Compare\left(\ell_{\phi^*_1}(q_1,p_1) - \ell_{\phi^*_1}(f'_1,x'_1), \ldots, \ell_{\phi^*_T}(q_T, p_T) - \ell_{\phi^*_T}(f'_T,x'_T)\right) \le \theta/6\ \middle| \ \left((f_1,x_1),\ldots, (f_T,x_T)\right) \in A \right)  \\
& ~ \le \mathbb{P} \left( \Compare(\ell_{\phi^*_1}(q_1,p_1) - \ell_{\phi^*_1}(f_1,x_1), \ldots, \ell_{\phi^*_T}(q_T, p_T) - \ell_{\phi^*_T}(f_T,x_T)) \right.\\
& ~~~~~~~~~~~~~~~~~~~~~~~~~~~\left. - \Compare(\ell_{\phi^*_1}(q_1,p_1) - \ell_{\phi^*_1}(f'_1,x'_1), \ldots, \ell_{\phi^*_T}(q_T, p_T) - \ell_{\phi^*_T}(f'_T,x'_T)) > \theta/6 \right) .
\end{align*}
By subadditivity of $\Compare$, the above expression is upper-bounded by
\begin{align*}
&\mathbb{P} \left( \Compare(\ell_{\phi^*_1}(f'_1,x'_1) - \ell_{\phi^*_1}(f_1,x_1), \ldots, \ell_{\phi^*_T}(f'_T,x'_T) - \ell_{\phi^*_T}(f_T,x_T))> \theta/6 \right) \\
&\le \mathbb{P} \left(\sup_{\phi \in \Phi_T} \Compare(\ell_{\phi_1}(f'_1,x'_1) - \ell_{\phi_1}(f_1,x_1), \ldots, \ell_{\phi_T}(f'_T,x'_T) - \ell_{\phi_T}(f_T,x_T))> \theta/6 \right)
\end{align*}
}
Hence,
{\small
\begin{align*}
&\sup_{\mbf{D}} \mathbb{P}_{\mbf{D}} \left(\sup_{\phi \in \Phi_T} \Compare(\ell_{\phi_1}(q_1,p_1) - \ell_{\phi_1}(f_1,x_1), \ldots, \ell_{\phi_T}(q_T, p_T) - \ell_{\phi_T}(f_T,x_T)) > \frac{ }{ } \theta/3 \right) \\
& ~~ \le 2 \sup_{\mbf{D}} \mathbb{P}_\mbf{D} \left(\sup_{\phi \in \Phi_T} \Compare(\ell_{\phi_1}(f'_1,x'_1) - \ell_{\phi_1}(f_1,x_1), \ldots, \ell_{\phi_T}(f'_T,x'_T) - \ell_{\phi_T}(f_T,x_T)) > \theta/6 \right)\\
&  = 2  \sup_{q_1,p_1} \Eunder{f_1 , f'_1 \sim q_1}{x_1, x'_1 \sim p_1} \ldots \sup_{q_T,p_T} \Eunder{f_T, f'_T \sim q_T}{x_T, x'_T \sim p_T} \ind{ \left(\sup_{\phi \in \Phi_T} \Compare(\ell_{\phi_1}(f'_1,x'_1) - \ell_{\phi_1}(f_1,x_1), \ldots, \ell_{\phi_T}(f'_T,x'_T) - \ell_{\phi_T}(f_T,x_T)) > \theta/6 \right)}.
\end{align*}
Next, introducing a Rademacher random variable $\epsilon_T$, the above quantity is equal to
\begin{align*}
& 2  \sup_{q_1,p_1} \Eunder{f_1 , f'_1 \sim q_1}{x_1, x'_1 \sim p_1} \ldots \sup_{q_T,p_T} \Eunder{f_T, f'_T \sim q_T}{x_T, x'_T \sim p_T}\Es{\epsilon_T}{ \ind{ \left(\sup_{\phi \in \Phi_T} \Compare(\ell_{\phi_1}(f'_1,x'_1) - \ell_{\phi_1}(f_1,x_1), \ldots, \epsilon_T (\ell_{\phi_T}(f'_T,x'_T) - \ell_{\phi_T}(f_T,x_T))) > \theta/6 \right)}}.
\end{align*}
We pass to an upper bound by taking supremum over $(f_T,x_T),(f'_T,x'_T)$:
\begin{align*}
& 2  \sup_{q_1,p_1} \Eunder{f_1 , f'_1 \sim q_1}{x_1, x'_1 \sim p_1} \ldots \sup_{q_{T-1},p_{T-1}} \Eunder{f_{T-1} , f'_{T-1} \sim q_{T -1}}{x_{T-1}, x'_{T-1} \sim p_{T-1}} \sup_{(f_T,x_T),(f'_T,x'_T)} \En_{\epsilon_T} \\
&~~~~~~~~~~~~~~~~~~~~~~~~~~~~~~\mathbf{1}\left\{ \sup_{\phi \in \Phi_T} \Compare(\ell_{\phi_1}(f'_1,x'_1) - \ell_{\phi_1}(f_1,x_1), \ldots, \epsilon_T (\ell_{\phi_T}(f'_T,x'_T) - \ell_{\phi_T}(f_T,x_T))) > \theta/6 \right\}.
\end{align*}
Repeating the process from inside out, we arrive at the upper bound 
\begin{align*}
& 2  \sup_{(f_1,x_1),(f'_1,x'_1)} \En_{\epsilon_1} \ldots  \sup_{(f_T,x_T),(f'_T,x'_T)} \En_{\epsilon_T} \\
&~~~~~~~~~~~~~~~~~~~~~~~~~~~~~ \mathbf{1}\left\{ \sup_{\phi \in \Phi_T} \Compare(\epsilon_1 (\ell_{\phi_1}(f'_1,x'_1) - \ell_{\phi_1}(f_1,x_1)), \ldots, \epsilon_T (\ell_{\phi_T}(f'_T,x'_T) - \ell_{\phi_T}(f_T,x_T))) > \theta/6 \right\}
\end{align*}
which can be written using the tree notation as
\begin{align*}
&  2  \sup_{\f , \f' ,\x, \x'} \Es{\epsilon}{ \mathbf{1}\left\{ \sup_{\phi \in \Phi_T} \Compare(\epsilon_1 (\ell_{\phi_1}(\f'_1(\epsilon),\x'_1(\epsilon)) - \ell_{\phi_1}(\f_1(\epsilon),\x_1(\epsilon))), \ldots, \epsilon_T (\ell_{\phi_T}(\f'_T(\epsilon),\x'_T(\epsilon)) - \ell_{\phi_T}(\f_T(\epsilon),\x_T(\epsilon)))) > \theta/6 \right\}}\\
&  = 2  \sup_{\f , \f' ,\x, \x'} \mathbb{P}_{\epsilon}\left( \sup_{\phi \in \Phi_T} \Compare(\epsilon_1 (\ell_{\phi_1}(\f'_1(\epsilon),\x'_1(\epsilon)) - \ell_{\phi_1}(\f_1(\epsilon),\x_1(\epsilon))), \ldots, \epsilon_T (\ell_{\phi_T}(\f'_T(\epsilon),\x'_T(\epsilon)) - \ell_{\phi_T}(\f_T(\epsilon),\x_T(\epsilon)))) > \theta/6 \right)
\end{align*}
Next, using subadditivity of $\Compare$, the last quantity can be upper bounded by
\begin{align*}
&2  \sup_{\f , \f' ,\x, \x'} \mathbb{P}_{\epsilon}\left( \sup_{\phi \in \Phi_T} \left\{ \Compare(\epsilon_1 \ell_{\phi_1}(\f'_1(\epsilon),\x'_1(\epsilon)) , \ldots, \epsilon_T \ell_{\phi_T}(\f'_T(\epsilon),\x'_T(\epsilon)) ) 
+  \Compare(-\epsilon_1 \ell_{\phi_1}(\f_1(\epsilon),\x_1(\epsilon)) , \ldots, - \epsilon_T \ell_{\phi_T}(\f_T(\epsilon),\x_T(\epsilon)) ) \right\} > \theta/6 \right)\\
& \le 2  \sup_{\f , \f' ,\x, \x'} \left\{ \mathbb{P}_{\epsilon}\left( \sup_{\phi \in \Phi_T}  \Compare(\epsilon_1 \ell_{\phi_1}(\f'_1(\epsilon),\x'_1(\epsilon)) , \ldots, \epsilon_T \ell_{\phi_T}(\f'_T(\epsilon),\x'_T(\epsilon)) ) > \theta/12 \right)  \right.\\
&~~~~~~~~~~~~~~~~~~~~~~~~~~~\left.+ \mathbb{P}_{\epsilon}\left( \sup_{\phi \in \Phi_T} \Compare(-\epsilon_1 \ell_{\phi_1}(\f_1(\epsilon),\x_1(\epsilon)) , \ldots, - \epsilon_T \ell_{\phi_T}(\f_T(\epsilon),\x_T(\epsilon)) )  > \theta/12 \right)\right\}\\
& = 4 \ \sup_{\f  ,\x}\ \mathbb{P}_{\epsilon}\left( \sup_{\phi \in \Phi_T}  \Compare(\epsilon_1 \ell_{\phi_1}(\f_1(\epsilon),\x_1(\epsilon)) , \ldots, \epsilon_T \ell_{\phi_T}(\f_T(\epsilon),\x_T(\epsilon)) ) > \theta/12 \right)  ,
\end{align*}
}
concluding the proof.
\end{proof}

\begin{proof}[\textbf{Proof of Lemma~\ref{lem:gensmoothcon}}]
	By Proposition~\ref{prop:smooth} and the Azuma-Hoeffding inequality for real-valued martingales,
\begin{align*}
P\left(\Compare(z_1,\ldots,z_T) > \theta \right) & = P\left(\Compare^q(z_1,\ldots,z_T) > \theta ^q\right) \\
& \le P\left(\sum_{t=1}^T \inner{\nabla_t \Compare^q(z_1,\ldots,z_{t-1},0,\ldots,0)  ,  z_t} > \theta ^q - \sigma T \RH^p/p \right)\\
& \le \mathrm{exp}\left( - \frac{\left(\theta ^q - \sigma T \RH^p/p\right)^2}{2\RH^2 R^2 T}\right) \ .
\end{align*}
\end{proof}

\begin{proof}[\textbf{Proof of Lemma~\ref{lem:pollard}}]
Fix $(\f,\x)$ and let $V=\{\v^1,\ldots,\v^N\}$ be a minimal $\ell_1$-cover of $\Phi_T$ on $(\f,\x)$ of size $N\leq {\mathcal N}_1(\theta/2,\Phi_T,T)$. Let $\v[\phi,\epsilon]\in V$ denote a member of the cover which is close to $\phi\in\Phi_T$ on the path $\epsilon$. By sub-additivity of $G$,
\begin{align*}
	&\Prob_\epsilon\left( \sup_{\phi \in \Phi_T} G\left( 
		\frac{1}{T}\sum_{t=1}^T \epsilon_t \ell_{\phi_t}(\f_t(\epsilon),\x_t(\epsilon)) 
	\right) > \theta \right) \\
	&\leq \Prob_\epsilon\left( \sup_{\phi \in \Phi_T} \left\{ G\left(\frac{1}{T} \sum_{t=1}^T \epsilon_t (\ell_{\phi_t}(\f_t(\epsilon),\x_t(\epsilon)) - \v[\phi,\epsilon]_t))\right) + G\left(\frac{1}{T} \sum_{t=1}^T \epsilon_t \v[\phi,\epsilon]_t\right) \right\} > \theta \right) 
\end{align*}

Using the Lipschitz property of $G$ along with $G(0)=0$ and triangle inequality, we can upper bound the last quantity by
\begin{align*}
&\Prob_\epsilon\left( \sup_{\phi \in \Phi_T} \left\{ 
	\frac{1}{T} \sum_{t=1}^T \left\|\ell_{\phi_t}(\f_t(\epsilon),\x_t(\epsilon)) - \v[\phi,\epsilon]_t\right\| + G\left(\frac{1}{T} \sum_{t=1}^T \epsilon_t \v[\phi,\epsilon]_t\right) 
	\right\} 
	> \theta \right) \\
&\le \Prob_\epsilon\left( \sup_{\phi \in \Phi_T}  G\left( \frac{1}{T} \sum_{t=1}^T \epsilon_t \v[\phi,\epsilon]_t\right) > \theta/2 \right), 
\end{align*}
where the last step follows by the definition of the cover. The last quantity can be upper bounded by
\begin{align*}
\Prob_\epsilon\left( \max_{\v \in V}  G\left(\frac{1}{T} \sum_{t=1}^T \epsilon_t \v_t(\epsilon)\right) > \theta/2 \right) 
& \le \sum_{\v \in V} \Prob_\epsilon\left(   G\left(\frac{1}{T} \sum_{t=1}^T \epsilon_t \v_t(\epsilon)\right) > \theta/2 \right) \\
& \le |V| \sup_{\z} \Prob_{\epsilon} \left( G\left(\frac{1}{T}\sum_{t=1}^T \epsilon_t \z_t(\epsilon)\right) > \theta/2 \right) , 
\end{align*}
where the supremum is over all $\cH$-valued binary trees $\z$ of depth $T$. \end{proof}

\begin{proof}[\textbf{Proof of Corollary~\ref{cor:2smooth}}]
Follows directly by combining Lemma~\ref{lem:pollard} with Corollary~\ref{cor:pinelis_concentration}.
\end{proof}

\begin{proof}[\textbf{Proof of Proposition~\ref{prop:dudreal}}]
Define $\beta_0 = 1$ and $\beta_j = 2^{-j}$. For a fixed tree $(\f,\x)$ of depth $T$, let $V_j$ be an  $\ell_\infty$-cover at scale $\beta_j$. For any path $\epsilon \in \{\pm1\}^T$ and any $\bphi \in \Phi_T$, let $\v[\bphi,\epsilon]^j \in V_j$ a $\beta_j$-close element of the cover in the $\ell_\infty$ sense.
Now, for any $\bphi\in\Phi_T$,
\begin{align*}
\left| \frac{1}{T}\sum_{t=1}^T \epsilon_t \ell_{\phi_t}(\f_t(\epsilon),\x_t(\epsilon)) \right|
&\leq \left| \frac{1}{T}\sum_{t=1}^T \epsilon_t ( \ell_{\phi_t}(\f_t(\epsilon),\x_t(\epsilon)) - \v[\phi,\epsilon]^{N}_t ) \right| 
+ \sum_{j=1}^{N} \left| \frac{1}{T} \sum_{t=1}^T \epsilon_t \left( \v[\phi,\epsilon]^{j}_t  - \v[\phi,\epsilon]^{j-1}_t \right) \right|  \\
& \leq \max_{t\in[T]} \left| \ell_{\phi_t}(\f_t(\epsilon),\x_t(\epsilon))  -  \v[\phi,\epsilon]^N_t \right|  + \sum_{j=1}^{N} \left|\frac{1}{T} \sum_{t=1}^T \epsilon_t (\v[\phi,\epsilon]^{j}_t - \v[\phi,\epsilon]^{j-1}_t)\right| 
\end{align*}
Thus,
\begin{align*}
\sup_{\bphi\in\Phi_T} \left| \frac{1}{T}\sum_{t=1}^T \epsilon_t \ell_{\phi_t}(\f_t(\epsilon),\x_t(\epsilon)) \right| & \le \beta_N  + \sup_{\phi \in \Phi_T}\left\{\sum_{j=1}^{N} \left|\frac{1}{T} \sum_{t=1}^T \epsilon_t (\v[\phi,\epsilon]^{j}_t - \v[\phi,\epsilon]^{j-1}_t)\right| \right\}
\end{align*}
We now proceed to upper bound the second term. Consider all possible pairs of $\v^s\in V_j$ and $\v^r\in V_{j-1}$, for $1\leq s\leq |V_j|$, $1\leq r \leq |V_{j-1}|$, where we assumed an arbitrary enumeration of elements. For each pair $(\v^s,\v^r)$, define a real-valued tree $\w^{(s,r)}$ by
\begin{align*}
\w^{(s,r)}_t(\epsilon) = 
	\begin{cases} 
	\v^s_t(\epsilon)-\v^r_t(\epsilon) & \text{if there exists } \bphi \in \Phi_T \mbox{ s.t. } \v^s = \v[\bphi,\epsilon]^{j}, \v^r = \v[\bphi,\epsilon]^{j-1} \\
	0 &\text{otherwise.}
	\end{cases}
\end{align*}
for all $t\in [T]$ and $\epsilon\in\{\pm1\}^T$. It is crucial that $\w^{(s,r)}$ can be non-zero only on those paths $\epsilon$ for which $\v^s$ and $\v^r$ are indeed the members of the covers (at successive resolutions) close in the $\ell_\infty$ sense {\em to some} $\bphi \in \Phi_T$. It is easy to see that $\w^{(s,r)}$ is well-defined. Let the set of trees $W_j$ be defined as
\begin{align*}
	W_j = \left\{ \w^{(s,r)}: 1\leq s\leq |V_j|, 1\leq r \leq |V_{j-1}| \right\}
\end{align*}
Using the above notations we see that

\begin{align}\label{eq:dudsimp1}
	\sup_{\phi \in \Phi_T} \left| \frac{1}{T}\sum_{t=1}^T \epsilon_t \ell_{\phi_t}(\f_t(\epsilon),\x_t(\epsilon)) \right| 
&\leq  \beta_N  + \sup_{\phi \in \Phi_T}\left\{\sum_{j=1}^{N} \left|\frac{1}{T} \sum_{t=1}^T \epsilon_t (\v[\phi,\epsilon]^{j}_t - \v[\phi,\epsilon]^{j-1}_t)\right| \right\} \notag\\
&  \le \beta_N  + \sum_{j=1}^{N} \sup_{\w^j \in W_j} \left|\frac{1}{T} \sum_{t=1}^T \epsilon_t \w^j_t(\epsilon)\right|
\end{align}

It is easy to show that $\max_{t\in[T]} |\w^j_t(\epsilon)| \leq 3\beta_j$ for any $\w^j\in W_j$ and any path $\epsilon$.

In the remainder of the proof we will use the shorthand $\N_\infty (\beta) = \N_\infty(\beta, \Phi_T, T)$. By Azuma-Hoeffding inequality for real-valued martingales,
$$
\Prob_\epsilon\left(\left|\frac{1}{T} \sum_{t=1}^T \epsilon_t \w^j_t(\epsilon)\right| > \theta \beta_j \sqrt{\log \N_\infty(\beta_j)}\right) \le 2\exp\left\{- \frac{ T \theta^2 \log \N_\infty(\beta_j)}{2} \right\}
$$
Hence by union bound we have,
$$
\Prob_\epsilon\left(\sup_{\w^j \in W_j} \left|\frac{1}{T} \sum_{t=1}^T \epsilon_t \w^j_t(\epsilon)\right| > \theta \beta_j \sqrt{\log \N_\infty(\beta_j)}\right) \le 2 \N_\infty(\beta_j)^2\ \exp\left\{- \frac{T \theta^2 \log \N_\infty(\beta_j)}{2} \right\}
$$
and so
$$
\Prob_\epsilon\left(\exists j \in [N],\ \ \sup_{\w^j \in W_j} \left|\frac{1}{T} \sum_{t=1}^T \epsilon_t \w^j_t(\epsilon)\right| > \theta \beta_j \sqrt{\log \N_\infty(\beta_j)}\right) \le 2\sum_{j=1}^N \N_\infty(\beta_j)^2\ \exp\left\{- \frac{T \theta^2 \log \N_\infty(\beta_j)}{2} \right\}
$$
Hence clearly
$$
\Prob_\epsilon\left(\sum_{j=1}^N \sup_{\w^j \in W_j} \left|\frac{1}{T} \sum_{t=1}^T \epsilon_t \w^j_t(\epsilon)\right| > \theta \sum_{j=1}^N \beta_j \sqrt{\log \N_\infty(\beta_j)}\right) \le 2\sum_{j=1}^N \N_\infty(\beta_j)^2\ \exp\left\{- \frac{T \theta^2 \log \N_\infty(\beta_j)}{2} \right\}
$$
Using the above with Equation \eqref{eq:dudsimp1} gives us that
\begin{align*}
\Prob_\epsilon\left(\sup_{\phi \in \Phi_T} \left|\frac{1}{T}\sum_{t=1}^T \epsilon_t \ell_{\phi_t}(\f_t(\epsilon),\x_t(\epsilon))\right| > \beta_N + \theta \sum_{j=1}^N \beta_j \sqrt{\log \N_\infty(\beta_j)} \right) 
&\le 2\sum_{j=1}^N \N_\infty(\beta_j)^2\ \exp\left\{- \frac{T \theta^2 \log \N_\infty(\beta_j)}{2} \right\}\\
&\le 2\sum_{j=1}^N \exp\left\{\log \N_\infty(\beta_j) \left(2 - \frac{T \theta^2}{2}\right) \right\}
\end{align*}
Since we assume that $2 < \frac{T \theta^2}{4}$, the right-hand side of the last inequality is bounded above by
\begin{align*}
2\sum_{j=1}^N \exp\left\{ - \frac{T \theta^2 \log \N_\infty(\beta_j) }{4} \right\} \le 2\sum_{j=1}^N \exp\left\{ - \frac{T \theta^2 }{4} - \log \N_\infty(\beta_j)  \right\}
\le 2e^{- \frac{T \theta^2  }{4}} \sum_{j=1}^N \N_\infty(\beta_j)^{-1} \ .
\end{align*}
By our assumption that
$
\sum_{j=1}^N \N_\infty(\beta_j)^{-1} \le L
$ for some appropriate constant $L$, we see that
\begin{align*}
& \Prob_\epsilon\left(\sup_{\phi \in \Phi_T} \left|\frac{1}{T}\sum_{t=1}^T \epsilon_t \ell_{\phi_t}(\f_t(\epsilon),\x_t(\epsilon))\right| > \beta_N + \theta \sum_{j=1}^N \beta_j \sqrt{\log \N_\infty(\beta_j)} \right) \le L e^{- \frac{T \theta^2  }{4}}
\end{align*}
Now picking $N$ appropriately and bounding sum by integral we have that 
$$
\beta_N + \theta \sum_{j=1}^N \beta_j \sqrt{\log \N_\infty(\beta_j)} \le \inf_{\alpha > 0}\left\{ 4 \alpha + 12 \theta \int_{\alpha}^{1} \sqrt{\log \N_\infty(\delta)} d \delta \right\}
$$
Hence we conclude that 
\begin{align*}
& \Prob_\epsilon\left(\sup_{\phi \in \Phi_T} \left|\frac{1}{T}\sum_{t=1}^T \epsilon_t \ell_{\phi_t}(\f_t(\epsilon),\x_t(\epsilon))\right| > \inf_{\alpha > 0}\left\{ 4 \alpha + 12 \theta \int_{\alpha}^{1} \sqrt{\log \N_\infty(\delta,\Phi_T,T)} d \delta \right\} \right) \le L e^{- \frac{T \theta^2  }{4}}
\end{align*}

The last statement the Proposition follows from the fact that the Dudley-type integral
\[
\inf_{\alpha > 0}\left\{ 4 \alpha + 12 \theta \int_{\alpha}^{1} \sqrt{\log \N_\infty(\delta,\Phi_T,T)} d \delta \right\} 
\]
can be upper bounded by
\[
8\left(1 + 4\sqrt{2}\theta \sqrt{ T \log^3(eT^2)}\right) \le 128\left(1 + \theta \sqrt{T \log^3(2T)} \right) 
\]
times the sequential Rademacher complexity. The proof can be found in \cite{RakSriTew10}.
\end{proof}

\begin{proof}[\textbf{Proof of Lemma~\ref{lem:smoothcon}}]
Let $\|\cdot\|_*$ be the norm dual to $\|\cdot\|$. First note that
\begin{align*}
& \Prob\left(\left\|\frac{1}{T} \sum_{t=1}^T \epsilon_t \x_t(\epsilon) \right\| > c\ \sup_\x \E{\left\|\frac{1}{T} \sum_{t=1}^T \epsilon_t \x_t(\epsilon) \right\|}\left(1 + \theta \sqrt{T \log^{3}T} \right) \right) \\
& ~~~~~~~~~~ = \Prob\left(\sup_{w : \|w\|_* \le 1 } \frac{1}{T} \sum_{t=1}^T \epsilon_t \inner{ w , \x_t(\epsilon)} > c\ \frac{1}{T} \sup_\x \E{\sup_{w : \|w\|_* \le 1} \sum_{t=1}^T \epsilon_t \inner{w , \x_t(\epsilon)}}\left(1 +   \theta \sqrt{T \log^{3}T} \right)  \right) \ .
\end{align*}
Now, by Proposition \ref{prop:dudreal} for payoff functions $\ell(f,x) = f(x)= \inner{f,x}$ and class $\Phi_T$ being the time-invariant constant departure mapping class, by noting that $\sup_\x \E{\left\|\frac{1}{T} \sum_{t=1}^T \epsilon_t \x_t(\epsilon) \right\|} = \Rad_T(\F)$ we get that
\begin{align*}
& \Prob\left(\left\|\frac{1}{T} \sum_{t=1}^T \epsilon_t \x_t(\epsilon) \right\| > c\ \sup_\x \E{\left\|\frac{1}{T} \sum_{t=1}^T \epsilon_t \x_t(\epsilon) \right\|}\left(1 + \theta \sqrt{T \log^{3}T} \right) \right) \le L \exp\left( -T \theta^2/ 2\right)
\end{align*}
where $c=128$.
Now note that for a $(\sigma,p)$-smooth space we have that
$$
\sup_\x \E{\left\|\frac{1}{T} \sum_{t=1}^T \epsilon_t \x_t(\epsilon) \right\| } \le \sigma^{1/p} \left(\frac{1}{T^p}\sup_\x \sum_{t=1}^T \E{\|\x_t(\epsilon)\|^p} \right)^{1/p} \le  \frac{\sigma^{1/p}R}{T^{1-1/p}}
$$
Moreover, the linear class $\F$ has covering numbers satisfying $\mc{N}_\infty(\beta) \ge 1/\beta$ and hence $L < 2$. Thus,
\begin{align*}
& \Prob\left(\left\|\frac{1}{T} \sum_{t=1}^T \epsilon_t \x_t(\epsilon) \right\| > c\ \frac{\sigma^{1/p}R}{T^{1-1/p}}
\left( 1+ \theta \sqrt{T \log^{3}T} \right) \right) \le 2 \exp\left( -T \theta^2/ 2\right)
\end{align*}
Now setting $\nu = \theta \sigma^{1/p} \sqrt{T \log^3 T}/T^{1-1/p}$ gives the required bound as,
\begin{align*}
& \Prob\left(\left\|\frac{1}{T} \sum_{t=1}^T \epsilon_t \x_t(\epsilon) \right\| >
	c\ \frac{\sigma^{1/p}R}{T^{1-1/p}} + c\nu R \right)
	\le
	2 \exp\left( -\frac{\nu^2T^{2-2/p}}{2\sigma^{2/p}\log^3 T} \right)
\end{align*}
The condition $\theta > \sqrt{8/T}$ on $\theta$ (from Proposition \ref{prop:dudreal}) implies that the above is valid only for
$$\nu > \frac{8\sigma^{1/p}\log^{3/2} T}{T^{1-1/p}}\ .$$
\end{proof}

\begin{proof}[\textbf{Proof of Theorem~\ref{thm:rad_upper_p_smooth}}]
Define $\beta_0 = 1$ and $\beta_j = 2^{-j}$. For a fixed tree $(\f,\x)$ of depth $T$, let $V_j$ be an  $\ell_\infty$-cover at scale $\beta_j$. For any path $\epsilon \in \{\pm1\}^T$ and any $\bphi \in \Phi_T$, let $\v[\bphi,\epsilon]^j \in V_j$ a $\beta_j$-close element of the cover in the $\ell_\infty$ sense.
Now, for any $\bphi\in\Phi_T$,
\begin{align*}
&\sup_{\phi \in \Phi_T} G\left(\frac{1}{T}\sum_{t=1}^T \epsilon_t \ell_{\phi_t}(\f_t(\epsilon),\x_t(\epsilon))\right) \\
& =  \sup_{\phi \in \Phi_T} \left\{ G\left(\frac{1}{T}\sum_{t=1}^T \epsilon_t \ell_{\phi_t}(\f_t(\epsilon),\x_t(\epsilon))\right) -  G\left(\frac{1}{T} \sum_{t=1}^T \epsilon_t \v[\phi,\epsilon]^{N}_t\right) \right.\\
&~~~~~~~~~~~\left. + \sum_{j=1}^{N}\left(G\left(\frac{1}{T} \sum_{t=1}^T \epsilon_t \v[\phi,\epsilon]^{j}_t\right) - G\left(\frac{1}{T} \sum_{t=1}^T \epsilon_t \v[\phi,\epsilon]^{j-1}_t \right) \right) \right\} \\
& \le \sup_{\phi \in \Phi_T} \left\{ \left\|\frac{1}{T} \sum_{t=1}^T \epsilon_t (\ell_{\phi_t}(\f_t(\epsilon),\x_t(\epsilon))  -  \v[\phi,\epsilon]^N_t) \right\| + \sum_{j=1}^{N} \left\|\frac{1}{T} \sum_{t=1}^T \epsilon_t (\v[\phi,\epsilon]^{j}_t - \v[\phi,\epsilon]^{j-1}_t)\right\|\right\}\\
& \le \sup_{\phi \in \Phi_T}\left\{\max_{t \in [T]} \left\| \ell_{\phi_t}(\f_t(\epsilon),\x_t(\epsilon))  -  \v[\phi,\epsilon]^N_t \right\|  + \sum_{j=1}^{N} \left\|\frac{1}{T} \sum_{t=1}^T \epsilon_t (\v[\phi,\epsilon]^{j}_t - \v[\phi,\epsilon]^{j-1}_t)\right\| \right\}\\
& \le \beta_N  + \sup_{\phi \in \Phi_T}\left\{\sum_{j=1}^{N} \left\|\frac{1}{T} \sum_{t=1}^T \epsilon_t (\v[\phi,\epsilon]^{j}_t - \v[\phi,\epsilon]^{j-1}_t)\right\| \right\}
\end{align*}
Consider all possible pairs of $\v^s\in V_j$ and $\v^r\in V_{j-1}$, for $1\leq s\leq |V_j|$, $1\leq r \leq |V_{j-1}|$, where we assumed an arbitrary enumeration of elements. For each pair $(\v^s,\v^r)$, define an $\cH$-valued tree $\w^{(s,r)}$ by
\begin{align*}
\w^{(s,r)}_t(\epsilon) = 
	\begin{cases} 
	\v^s_t(\epsilon)-\v^r_t(\epsilon) & \text{if there exists } \bphi \in \Phi_T \mbox{ s.t. } \v^s = \v[\bphi,\epsilon]^{j}, \v^r = \v[\bphi,\epsilon]^{j-1} \\
	0 &\text{otherwise.}
	\end{cases}
\end{align*}
for all $t\in [T]$ and $\epsilon\in\{\pm1\}^T$. It is crucial that $\w^{(s,r)}$ can be non-zero only on those paths $\epsilon$ for which $\v^s$ and $\v^r$ are indeed the members of the covers (at successive resolutions) close in the $\ell_\infty$ sense {\em to some} $\bphi \in \Phi_T$. It is easy to see that $\w^{(s,r)}$ is well-defined. Let the set of trees $W_j$ be defined as
\begin{align*}
	W_j = \left\{ \w^{(s,r)}: 1\leq s\leq |V_j|, 1\leq r \leq |V_{j-1}| \right\}
\end{align*}
Using the above notations we see that

\begin{align*}
\sup_{\phi \in \Phi_T} G\left(\frac{1}{T}\sum_{t=1}^T \epsilon_t \ell_{\phi_t}(\f_t(\epsilon),\x_t(\epsilon))\right) &  \le \beta_N  + \sup_{\phi \in \Phi_T}\left\{\sum_{j=1}^{N} \left\|\frac{1}{T} \sum_{t=1}^T \epsilon_t (\v[\phi,\epsilon]^{j}_t - \v[\phi,\epsilon]^{j-1}_t)\right\| \right\} \notag \\
&  \le \beta_N  + \sum_{j=1}^{N} \sup_{\w^j \in \W^j} \left\|\frac{1}{T} \sum_{t=1}^T \epsilon_t \w^j_t(\epsilon)\right\| 
\end{align*}

Now before we proceed note that any $\w^j \in W_j$ is such that for any $t \in [T]$ and any $\epsilon \in \{\pm1\}^T$, $\|\w_t^j(\epsilon)\| \le 3\beta_j$. Hence we see that $W_j$ consists of $Y_j$-valued trees, where $Y_j = \{x : \|x\| \le 3\beta_j\}$. Hence  
\begin{align}\label{eq:dudsimp}
\sup_{\phi \in \Phi_T} G\left(\frac{1}{T}\sum_{t=1}^T \epsilon_t \ell_{\phi_t}(\f_t(\epsilon),\x_t(\epsilon))\right) &  \le \beta_N  + \sum_{j=1}^{N} \sup_{\w^j \in \W^j} \left\|\frac{1}{T} \sum_{t=1}^T \epsilon_t \w^j_t(\epsilon)\right\| \notag \\
&  \le \beta_N  + \sum_{j=1}^{N} \sup_{\y^j} \left\|\frac{1}{T} \sum_{t=1}^T \epsilon_t \y^j_t(\epsilon)\right\| 
\end{align}
where the supremum is over $Y_j$-valued trees. 

In the remainder of the proof we will use the shorthand $\N_\infty(\beta) = \N_\infty(\beta, \Phi_T, T)$ and will use the constant $c=128$. By Lemma \ref{lem:smoothcon}, for any
$\theta \ge 8\,c\,\sigma^{1/p}\,\log^{3/2}T/T^{1-1/p}$, we have
$$
\Prob_\epsilon\left(\left\|\frac{1}{T} \sum_{t=1}^T \epsilon_t \y^j_t(\epsilon)\right\| > \frac{3c\sigma^{1/p} \beta_j}{T^{1 - 1/p}} + 3\theta \beta_j \sqrt{\log \N_\infty(\beta_j)}\right) \le 2\ \exp\left\{- \frac{ T^{2 - 2/p}\ \theta^2\ \log \N_\infty(\beta_j)}{2 c^2 \sigma^{2/p} \log^3 T} \right\} \ .
$$
By the union bound,
$$
\Prob_\epsilon\left( \sup_{\y^j} \left\|\frac{1}{T} \sum_{t=1}^T \epsilon_t \y^j_t(\epsilon)\right\| > \frac{3c\sigma^{1/p} \beta_j}{T^{1 - 1/p}} + 3\theta \beta_j \sqrt{\log \N_\infty(\beta_j)}\right) \le 2\ \N_\infty(\beta_j)\  \exp\left\{- \frac{ T^{2 - 2/p}\ \theta^2\ \log \N_\infty(\beta_j)}{2 c^2 \sigma^{2/p} \log^3 T} \right\}
$$
and so
{\small
$$
\Prob_\epsilon\left(\exists j \in [N],  ~~\sup_{\y^j} \left\|\frac{1}{T} \sum_{t=1}^T \epsilon_t \y^j_t(\epsilon)\right\| > \frac{3c \sigma^{1/p} \beta_j}{T^{\frac{p - 1}{p}}} + 3\theta \beta_j \sqrt{\log \N_\infty(\beta_j)}\right) \le 2 \sum_{j=1}^N \N_\infty(\beta_j)\  \exp\left\{- \frac{ T^{\frac{2(p-1)}{p}} \theta^2 \log \N_\infty(\beta_j)}{2 c^2 \sigma^{2/p} \log^3 T} \right\}
$$}
Hence, 
{\small 
$$
\Prob_\epsilon\left(\sum_{j=1}^N \sup_{\y^j} \left\|\frac{1}{T} \sum_{t=1}^T \epsilon_t \y^j_t(\epsilon)\right\| > \frac{6\sigma^{1/p} c}{T^{\frac{p- 1}{p}}} + 3\theta \sum_{j=1}^N \beta_j \sqrt{\log \N_\infty(\beta_j)}\right) \le 2\ \sum_{j=1}^N \N_\infty(\beta_j)\  \exp\left\{- \frac{ T^{\frac{2(p - 1)}{p}} \theta^2 \log \N_\infty(\beta_j)}{2 c^2 \sigma^{2/p} \log^3 T} \right\} \ .
$$}

Using the above with Equation \eqref{eq:dudsimp} gives us that
\begin{align*}
& \Prob_\epsilon\left(\sup_{\phi \in \Phi_T} G\left(\frac{1}{T}\sum_{t=1}^T \epsilon_t \ell_{\phi_t}(\f_t(\epsilon),\x_t(\epsilon))\right) > \frac{6\sigma^{1/p} c}{T^{\frac{p-1}{p}}} + \beta_N + 3\theta \sum_{j=1}^N \beta_j \sqrt{\log \N_\infty(\beta_j)} \right) \\
& ~~~~~~~~~~\le 2\ \sum_{j=1}^N \N_\infty(\beta_j)\  \exp\left\{- \frac{ T^{\frac{2(p - 1)}{p}} \theta^2 \log \N_\infty(\beta_j)}{2 c^2 \sigma^{2/p} \log^3 T} \right\}\\
& ~~~~~~~~~~\le 2\ \sum_{j=1}^N  \exp\left\{\log \N_\infty(\beta_j) \left(1 - \frac{ T^{\frac{2(p - 1)}{p}} \theta^2 }{2 c^2 \sigma^{2/p} \log^3 T}\right) \right\}
\end{align*}
Our assumption on $\theta$ implies that $\frac{ T^{\frac{2(p - 1)}{p}} \theta^2 }{4 c^2 \sigma^{2/p} \log^3 T} \ge 2$, so that
\begin{align*}
& \Prob_\epsilon\left(\sup_{\phi \in \Phi_T} G\left(\frac{1}{T}\sum_{t=1}^T \epsilon_t \ell_{\phi_t}(\f_t(\epsilon),\x_t(\epsilon))\right) > \frac{6\sigma^{1/p} c}{T^{\frac{p-1}{p}}} + \beta_N + 3\theta \sum_{j=1}^N \beta_j \sqrt{\log \N_\infty(\beta_j)} \right) \\
& ~~~~~~~~~~\le 2\ \sum_{j=1}^N  \exp\left\{-   \frac{ T^{\frac{2(p - 1)}{p}} \theta^2 \log \N_\infty(\beta_j) }{4 c^2 \sigma^{2/p} \log^3 T} \right\}\\
& ~~~~~~~~~~\le 2\ \sum_{j=1}^N \exp\left\{ - \frac{ T^{\frac{2(p - 1)}{p}} \theta^2 }{4 c^2 \sigma^{2/p} \log^3 T} - \log \N_\infty(\beta_j)  \right\}\\
& ~~~~~~~~~~\le2\  \exp\left\{- \frac{ T^{\frac{2(p - 1)}{p}} \theta^2 }{4 c^2 \sigma^{2/p} \log^3 T}\right\} \sum_{j=1}^N  \N_\infty(\beta_j)^{-1}
\end{align*}
Since we have assumed that
$
2\ \sum_{j=1}^N \N_\infty(\beta_j)^{-1} \le L
$, we see that
\begin{align*}
& \Prob_\epsilon\left(\sup_{\phi \in \Phi_T} G\left(\frac{1}{T}\sum_{t=1}^T \epsilon_t \ell_{\phi_t}(\f_t(\epsilon),\x_t(\epsilon))\right) > \frac{6 \sigma^{1/p} c}{T^{\frac{p-1}{p}}} + \beta_N + 3\theta \sum_{j=1}^N \beta_j \sqrt{\log \N_\infty(\beta_j)} \right) \le L \exp\left\{- \frac{ T^{\frac{2(p - 1)}{p}} \theta^2 }{4 c^2 \sigma^{2/p} \log^3 T}\right\}
\end{align*}
Using the arguments employed previously, picking $N$ appropriately and bounding sum by integral we have that 
$$
\beta_N + 3\theta \sum_{j=1}^N \beta_j \sqrt{\log \N_\infty(\beta_j)} \le \inf_{\alpha > 0}\left\{ 4 \alpha + 36 \theta \int_{\alpha}^{1} \sqrt{\log \N_\infty(\delta)} d \delta \right\} \ .
$$
Hence we conclude that 
\begin{align*}
& \Prob_\epsilon\left(\sup_{\phi \in \Phi_T} G\left(\frac{1}{T}\sum_{t=1}^T \epsilon_t \ell_{\phi_t}(\f_t(\epsilon),\x_t(\epsilon))\right) > \frac{6 \sigma^{1/p} c}{T^{\frac{p-1}{p}}} + \inf_{\alpha > 0}\left\{ 4 \alpha + 36 \theta \int_{\alpha}^{1} \sqrt{\log \N_\infty(\delta)} d \delta \right\} \right) \le L \exp\left\{- \frac{ T^{\frac{2(p - 1)}{p}} \theta^2 }{4 c^2 \sigma^{2/p} \log^3 T}\right\}
\end{align*}
\end{proof}

\begin{proof}[\textbf{Proof of Theorem~\ref{thm:calibration_almost_sure}}]
Let $\alpha > 0$ be a constant that we will fix later. Consider a ``subgaussian game" whose value is defined as:
{\small
\begin{align}
\label{def:subgaussianvalue}
	\Val^{SG}_T(\ell,\Phi_T) &= \inf_{q_1} \sup_{x_1} \Eu{f_1\sim q_1} \ldots \inf_{q_T} \sup_{x_T} \Eu{f_T\sim q_T}
	\Gamma\left(
	\sup_{\bphi\in \Phi_T}\left\{ \Compare(\ell(f_1,x_1), \ldots, \ell(f_T, x_T)) -  \Compare(\ell_{\phi_1}(  f_1,x_1), \ldots, \ell_{\phi_T}(f_T, x_T))\right\}
	\right)
\end{align}
}
where
$$
\Gamma(x) := \sup_{\theta}\ \exp(\alpha T \theta^2/k) \ind{ x > \theta} = \exp(\alpha T x^2/k) \ . 
$$
Here, we are using the intuition that we expect to find a player strategy using which the regret will have subgaussian tails.
As before, we consider the calibration setting described in Example~\ref{eg:calibration2} augmented with the restriction that the
player's choice belongs to $C_\delta$, a $2\delta$-maximal packing of $\Delta(k)$, instead of $\Delta(k)$. The choice of $\delta$
will be fixed later. We now apply the
general triplex inequality in Appendix~\ref{sec:gentriplex} with
\[
	\Lambda(x) := \sup_{\theta}\ \exp(\alpha T\theta^2/k) \ind{ x > \theta/3 } = \exp(9 \alpha T x^2/k) \ .
\]
Observe that the first term in the General Triplex Inequality is simply equal to $1$.
The second term is upper bounded by a particular (sub)optimal response $q_t$ being the point mass on $p^\delta_t$, the element of $C_\delta$ closest to $p_t$. Note that any $2 \delta$ packing is also a $2 \delta$ cover. Thus, the second term becomes
\begin{align*}
	&\sup_{p_1}\inf_{q_1} \ldots  \sup_{p_T}\inf_{q_T} 
  \Lambda\left( \sup_{\bphi\in \Phi_T}   \left\{ - \Compare(\ell_{\phi_1}(q_1,p_1), \ldots, \ell_{\phi_T}(q_T, p_T))\right\} \right) \\
	&\leq \sup_{p_1}\ldots  \sup_{p_T} \Lambda\left( \sup_{\lambda> 0}\sup_{p\in\Delta(k)} \left\| \frac{1}{T}\sum_{t=1}^T \En_{x_t \sim p_t} \ell_{\phi_{p,\lambda}}(p^\delta_t,x_t) \right\| \right)\\
	&= \sup_{p_1}\ldots  \sup_{p_T} \Lambda\left( \sup_{\lambda> 0}\sup_{p\in\Delta(k)}\left\| \frac{1}{T}\sum_{t=1}^T \ind{\|p^\delta_t-p\|\leq \lambda}\cdot (p^\delta_t-p_t) \right\| \right) \\
	&\leq \Lambda\left(\delta\right) = \exp(9\alpha\delta^2/k) \ .
\end{align*}
By the same reasoning as used in the previous proof, the third term
\[
\sup_{\mbf{D}} \Es{\mbf{D}}{
	\Lambda\left( \sup_{p,\lambda} \left\|
				\frac{1}{T} \sum_{t=1}^T \left(
								\ind{ \|f_t - p\| \le \lambda }(f_t - x_t)
								- \Es{t-1}{ \ind{ \|f_t - p\| \le \lambda }(f_t - x_t) }
							\right)
			\right\| \right)
			} 
\]
can be bounded by
\[
\sup_{\mbf{D}} \Es{\mbf{D}}{
	\Lambda\left( \max_{(p,\lambda)\in S} \left\|
				\frac{1}{T} \sum_{t=1}^T \left(
								\ind{ \|f_t - p\| \le \lambda }(f_t - x_t)
								- \Es{t-1}{ \ind{ \|f_t - p\| \le \lambda }(f_t - x_t) }
							\right)
			\right\| \right)
			} 
\]
where $S$ is a finite set of cardinality $|S| \le |C_\delta|^{ck^2}$. Since $\Lambda$ is non-decreasing and maximum of positive quantities is bounded by their sum,
we have the upper bound
\begin{align*}
& \sup_{\mbf{D}} \sum_{(\lambda,p)\in S} \Es{\mbf{D}}{
	\Lambda\left( \left\|
				\frac{1}{T} \sum_{t=1}^T \left(
								\ind{ \|f_t - p\| \le \lambda }(f_t - x_t)
								- \Es{t-1}{ \ind{ \|f_t - p\| \le \lambda }(f_t - x_t) }
							\right)
			\right\| \right)						
	} \\
&\le |S| \cdot M_\Lambda
\end{align*}
where $M_\Lambda$ is defined as
\[
	M_\Lambda := \sup_{MDS}\ \E{ \Lambda\left( \left\| \sum_{t=1}^T X_t \right\| \right) }\ .
\]
Here the supremum is over all martingale difference sequences $X_1,\ldots,X_T$ with $\|X_t\|_1 \le 2/T$ almost surely. 
Since we are considering the case when $\|\cdot\| = \|\cdot\|_1$, we have
\begin{align*}
M_\Lambda
&= \sup_{MDS} \E{ \exp\left( 9 \alpha T \left\| \sum_{t=1}^T X_t \right\|_1^2 / k \right) }\\
&\le \sup_{MDS} \E{ \exp\left( 9 \alpha T \left\| \sum_{t=1}^T X_t \right\|_2^2 \right) }
\end{align*}
Using Corollary~\ref{cor:pinelis_concentration}, we have
\begin{align*}
\E{ \exp\left( 9 \alpha T \left\| \sum_{t=1}^T X_t \right\|_2^2 \right) } 
&\le e + \int_{\theta \ge e} \Prob\left( 9 \alpha T \left\| \sum_{t=1}^T X_t \right\|_2^2 \ge \theta\right) d\theta \\
&\le e + \int_{\theta \ge e} 2\exp\left( - \frac{\log(\theta)}{288 \alpha} \right) d\theta \\
&\le e + \int_{\theta \ge e} \frac{2}{\theta^2}d\theta \le e+2 \le 5
\end{align*}
where we chose $\alpha = 1/576$ to make $288 \alpha = 1/2$. This shows that $M_\Lambda \le 5$ and hence the third term is bounded by
$5 |S|$.

Now putting the upper bounds on the three triplex inequality terms together, we get that
\[
	\Val^{SG}_T(\ell,\Phi_T) \le 1 + \exp\left( \frac{T\delta^2}{64k} \right) + 5\left(\frac{1}{\delta}\right)^{ck^3}\ .
\]
Choose $\delta = \sqrt{k/T}$ to get
\[
	\Val^{SG}_T(\ell,\Phi_T) \le 3 + 5 \left( \sqrt{\frac{T}{k}} \right)^{ck^3} \le 8 \, T^{ck^3/2}\ .
\]
Using Markov's inequality now shows that there is a player strategy such that against any adversary and any $\theta > 0$, we have
\[
	\Prob( \Reg_T > \theta) \le 8\, T^{ck^3/2} \exp\left( - \frac{T \theta^2}{576 k} \right)\ .
\]
Equivalently, for the same player strategy, against any adversary and any $\eta \in (0,1)$, we have with probability at least $1-\eta$,
\begin{equation}\label{eq:fixedhorizon}
	\Reg_T \le \frac{24}{\sqrt{T}} \cdot \sqrt{k \log\left(\frac{8}{\eta}\right) + \frac{ck^4}{2} \log(T) } \ .
\end{equation}

Finally to show almost sure convergence we need to use a ``doubling trick" similar to the one used in~\cite{ManSto09}.
We divide time into episodes $r=1,2,\ldots$ with episode $r$ of length $2^r$. In episode $r$, the player plays the optimal strategy
for the subgaussian game of length $2^r$. Thus, episode $r$ lasts during the time steps $E_r = \{2^r-1,\ldots,2^{r+1}-2\}$.
Now fix any adversary for the infinite round game and let us focus on the regret incurred
at some time $T$. We have,
\begin{align*}
\Reg_T
&=  \sup_{\lambda > 0}\sup_{p\in\Delta(k)} \left\|\frac{1}{T}\sum_{t=1}^T \ind{\|f_t-p\|\leq \lambda} \cdot (f_t-x_t) \right\| \\
&\le \frac{1}{T} \sum_{r = 1}^{\lceil \log_2(T) \rceil} \sup_{\lambda >0} \sup_{p\in\Delta(k)} \left\| \sum_{t \in E_r} \ind{\|f_t-p\|\leq \lambda} \cdot (f_t-x_t) \right\| \\
&\le \frac{1}{T} \sum_{r=1}^{\lceil \log_2(T) \rceil} 2^r \cdot \frac{24}{\sqrt{2^r}} \cdot \sqrt{ k \log\left( \frac{8}{\eta_{T,r}} \right) + \frac{ck^4}{2} \log(2^r) }
\end{align*}
with probability at least $1-\sum_{r<\log_2(T)} \eta_{T,r}$. In the last step we used~\eqref{eq:fixedhorizon} along with a union bound over episodes. Choosing
$\eta_{T,r} = 1/T^22^r$ ensures that with probability at least $1-1/T^2$, we have
\[
	\Reg_T \le 24(1+\sqrt{2}) \cdot \frac{ \sqrt{ k \log\left( 8 T^3 \right) + \frac{ck^4}{2} \log(T) } }   {\sqrt{T}} \ . 
\]
Since $24(1+\sqrt{2}) \le 60$, using Borel-Cantelli, this shows that
\[
	\Prob\left( \frac{\sqrt{T}}{\sqrt{ 3 k\log(2T) + \tfrac{ck^4}{2} \log(T)}} \cdot \Reg_T  > 60 \qquad \text{infinitely often} \right) = 0 \ .
\]
This proves the theorem.
\end{proof}

%% file: pinelis.tex
\section{Concentration of 2-Smooth Functions of Martingale-Difference Sums in Banach Spaces}

In this section we prove an extension of some of the results of Pinelis \cite{Pinelis94}. Let $(\cH,\|\cdot\|)$ be a separable Banach space such that there is a function $G:\cH \to \reals$ with the following properties:
\begin{align*}
G(\mathbf{0}) &= 0 &&\\
|G(\v + \w) - G(\v)| &\le \|\w\| && \text{(Lipschitz)} \\
(G^2)''(\v)[\w,\w] &\le \sigma \| \w \|^2 && \text{($G^2$ is $(\sigma,2)$-smooth)}
\end{align*}

Suppose we have an $\cH$-valued MDS $\{ X_t \}_{t=1}^T$. Define the partial sums $S_0 = \mathbf{0}$, $S_t = \sum_{s \le t} X_t$ for $t > 0$.
Define, for $t \ge 0$,
\[
	Z_t = \cosh(\lambda G(S_t))
\]
The following lemma is embedded in proof of Theorem 3.2 in Pinelis. Assume $\sigma \ge 1$ for simplicity. Otherwise, everything below
works by replacing $\sigma$ with $\max\{\sigma,1\}$.

\begin{lemma}
Suppose $\|X_t\| \le B$ a.s. and fix $\lambda > 0$. Then $Z_t/c^t$ is a supermartingale where
\[
	c = 1 + \sigma (\exp(\lambda B) - 1 - \lambda B) \ .
\]
In particular, we have
\[
	\E{ Z_T } \le c^T\ .
\]
\end{lemma}
\begin{proof}
The key step is to define a scalar function $\phi:[0,1] \to \reals$:
\[
	\phi(\alpha) := \Es{t-1}{ \cosh( \lambda G(S_{t-1} + \alpha X_t) ) }\ .
\]
Note that $\phi(1) = \Es{t-1}{ Z_t }$ and $\phi(0) = Z_{t-1}$, so our goal is to prove $\phi(1) \le c \cdot \phi(0)$.
We compute the first two derivatives of $\phi$,
\begin{align}
\notag
\phi'(\alpha) &= \Es{t-1}{ \sinh( \lambda g_{S_{t-1},X_t}(\alpha) ) \cdot \lambda g_{S_{t-1},X_t}'(\alpha) } \ , \\
\phi''(\alpha) &= \Es{t-1}{ \cosh( \lambda g_{S_{t-1},X_t}(\alpha) ) \cdot (\lambda g_{S_{t-1},X_t}'(\alpha))^2 } \\
&\quad +  \Es{t-1}{ \sinh( \lambda g_{S_{t-1},X_t}(\alpha) ) \cdot \lambda g''_{S_{t-1},X_t}(\alpha) } \ ,
\label{eq:secondder}
\end{align}
where, for any $S,X \in \cH$, we define $g_{S,X}(\alpha) = G(S + \alpha X)$. Note that
\begin{align*}
g_{S,X}'(\alpha) &= G'(S + \alpha X)(X) \ ,\\
g_{S,X}''(\alpha) &= G''(S + \alpha X)(X,X) \ .
\end{align*}
Now, consider two cases.

Case 1: $\sign(\lambda g_{S_{t-1},X_t}(\alpha)) = \sign(g''_{S_{t-1},X_t}(\alpha))$. In this case, we use the fact that
$\sign(\sinh(x)) = \sign(x \cosh(x)$ and that $|\sinh(x)| \le |x \cosh(x)|$, to obtain the upper bound
\begin{align*}
&\quad \cosh( \lambda g_{S_{t-1},X_t}(\alpha) ) \cdot (\lambda g_{S_{t-1},X_t}'(\alpha))^2 
+ \sinh( \lambda g_{S_{t-1},X_t}(\alpha) ) \cdot \lambda g''_{S_{t-1},X_t}(\alpha) \\
&\le  \cosh( \lambda g_{S_{t-1},X_t}(\alpha) ) \cdot (\lambda g_{S_{t-1},X_t}'(\alpha))^2 
+  \cosh( \lambda g_{S_{t-1},X_t}(\alpha) ) \cdot \lambda g_{S_{t-1},X_t}(\alpha) \cdot \lambda g''_{S_{t-1},X_t}(\alpha) \\
&= \lambda^2 \cdot \cosh( \lambda g_{S_{t-1},X_t}(\alpha) ) \cdot (g_{S_{t-1},X_T}^2)''(\alpha)  \\
&\le \sigma \lambda^2 B^2 \cdot \cosh( \lambda g_{S_{t-1},X_t}(\alpha) ) \ ,
\end{align*}
because $(g_{S_{t-1},X_T}^2)''(\alpha) = G''(S_{t-1} + \alpha X_t)(X_t,X_t) \le \sigma \|X_t\|^2 \le \sigma B^2$.

Case 2: $\sign(\lambda g_{S_{t-1},X_t}(\alpha)) \neq \sign(g''_{S_{t-1},X_t}(\alpha))$. In this case, we simply have,
\begin{align*}
&\quad \cosh( \lambda g_{S_{t-1},X_t}(\alpha) ) \cdot (\lambda g_{S_{t-1},X_t}'(\alpha))^2 
+ \sinh( \lambda g_{S_{t-1},X_t}(\alpha) ) \cdot \lambda g''_{S_{t-1},X_t}(\alpha) \\
&\le \cosh( \lambda g_{S_{t-1},X_t}(\alpha) ) \cdot (\lambda g_{S_{t-1},X_t}'(\alpha))^2 \\
&\le \lambda^2 B^2 \cdot \cosh( \lambda g_{S_{t-1},X_t}(\alpha) ) \ ,
\end{align*}
because, by Lipschitz property of $G$, we have
\[
|g_{S_{t-1},X_t}'(\alpha)| = |G'(S_{t-1} + \alpha X_t)(X_t)| \le \|G'(S_{t-1} + \alpha X_t)\|_\star \cdot \|X_t\| \le 1 \cdot B \ .
\]

Thus, we always have,
\[
	\cosh( \lambda g_{S_{t-1},X_t}(\alpha) ) \cdot (\lambda g_{S_{t-1},X_t}'(\alpha))^2 
	+ \sinh( \lambda g_{S_{t-1},X_t}(\alpha) ) \cdot \lambda g''_{S_{t-1},X_t}(\alpha)
	\le \sigma \lambda^2 B^2 \cdot \cosh( \lambda g_{S_{t-1},X_t}(\alpha) ) \ .
\]
Plugging this into~\eqref{eq:secondder}, we get
\begin{align*}
\phi''(\alpha) &\le \sigma \lambda^2 B^2 \Es{t-1}{ \cosh( \lambda G(S_{t-1} + \alpha X_t) ) } \\
&\le \sigma \lambda^2 B^2 \Es{t-1}{ \cosh( \lambda G(S_{t-1}) + \lambda \alpha \|X_t\| ) } \\
&\le \sigma \lambda^2 B^2 \Es{t-1}{ \cosh( \lambda G(S_{t-1}) ) \cdot \exp( \lambda \alpha \|X_t\| ) } \\
&\le \sigma \lambda^2 B^2 \cdot \cosh( \lambda G(S_{t-1}) ) \cdot \exp( \lambda \alpha B )  \\
&= \sigma \lambda^2 B^2 \cdot Z_{t-1} \cdot \exp( \lambda \alpha B ) \ .
\end{align*}

Note that $\phi'(0) = \Es{t-1}{G'(S_{t-1})(X_t)} = G'(S_{t-1})(\Es{t-1}{X_t}) = 0$ by the MDS property. Thus,
\[
	\phi'(\beta) = \int_{y=0}^\beta \phi''(y) dy
\]
and therefore
\begin{align*}
Z_{t} = \phi(1) &= \phi(0) + \int_{\beta = 0}^1 \phi'(\beta) d\beta \\
&= Z_{t-1} + \int_{\beta = 0}^1 \int_{y=0}^\beta \phi''(y) dy d\beta \\
&= Z_{t-1} + \int_{y = 0}^1 \int_{\beta=y}^1 \phi''(y) d\beta dy \\
&= Z_{t-1} + \int_{y = 0}^1 \phi''(y) (1-y) dy \\
&\le Z_{t-1} \cdot \left(1 + \sigma \lambda^2 B^2 \int_{y=0}^1 \exp(\lambda B y) (1-y) dy \right)\\
&= Z_{T-1} \cdot (1 + \sigma (\exp(\lambda B) - 1 - \lambda B)) 
\end{align*}
\end{proof}

Now that we have control over $\E{\cosh(\lambda G(S_T))}$, the following control on m.g.f. is immediate.

\begin{corollary}
	\label{cor:pinelis_concentration}
Under the same conditions as previous lemma,
\[
	\E{\exp(\lambda G(S_T))} \le 2\,c^T\ .
\]
Moreover, 
\begin{align*}
	P(G(S_T) > \epsilon) \leq 2 \exp\left( -\frac{\epsilon^2}{4T\sigma B^2}\right)
\end{align*}
whenever $T> \epsilon/(2\sigma B)$.
\end{corollary}
\begin{proof}
The first inequality follows by noting that $\cosh(x) = (\exp(x) + \exp(-x))/2 \ge \exp(x)/2$. 

For the second inequality, 
\begin{align*}
	P\left(G(S_T) > \epsilon\right) &= P \left(\exp(\lambda G(S_T)) > \exp(\lambda\epsilon)\right) \\
							&\leq \exp(-\lambda\epsilon) \E{\exp(\lambda G(S_T))} \\
							&\leq 2\exp(-\lambda\epsilon)(1+\sigma(\exp(\lambda B)-1-\lambda B))^T\\
							&\leq 2\exp\left\{-\lambda\epsilon + T\log(1+\sigma(\exp(\lambda B)-1-\lambda B))\right\} \\
							&\leq 2\exp\left\{-\lambda\epsilon + T\sigma(\exp(\lambda B)-1-\lambda B)\right\} \\
							&\leq 2\exp\left\{-\lambda\epsilon + T\sigma\lambda^2 B^2 \right\} 
\end{align*}
where the last inequality is valid for any $\lambda \le 1/B$. Optimizing over $\lambda$, we let 
$$\lambda = \frac{\epsilon}{2T\sigma B^2},$$
which yields the desired upper bound. The condition $\lambda \le 1/B$ is satisfied whenever $T> \epsilon/(2\sigma B)$.
\end{proof}

With control on the m.g.f., a Massart style union bound argument at the level of expectations is immediate.

\begin{theorem}
	\label{thm:pinelis_union_bound}
Suppose $\{X_t^\gamma\}_{t=0}^T$ is a family of MDS indexed by $\gamma$ in some finite set $\Gamma$. Suppose for each $\gamma, t$, $\|X_t^\gamma\|\le B$ a.s.
Then, we have, for any $T \ge \log(2 |\Gamma|)/\sigma$,
\[
	\E{ \max_{\gamma \in \Gamma} G( S^\gamma_T ) } \le 2B \, \sqrt{ \sigma\, \log(2|\Gamma|) \, T}  \ ,
\]
where $S^\gamma_T = \sum_{t=1}^T X^\gamma_t$.
\end{theorem}
\begin{proof}
Fix $\lambda > 0$. Then,
\begin{align*}
\exp\left( \lambda \E{ \max_{\gamma \in \Gamma} G( S^\gamma_T ) } \right)
&\le \E{ \exp( \lambda  \max_{\gamma \in \Gamma} G( S^\gamma_T ) ) } \\
&= \E{ \max_{\gamma \in \Gamma} \exp( \lambda G( S^\gamma_T ) ) } \\
&\le \E{ \sum_{\gamma \in \Gamma} \exp( \lambda G( S^\gamma_T ) ) } \\
&\le 2|\Gamma| \cdot (1 + \sigma (\exp(\lambda B) - 1 - \lambda B))^T \ .
\end{align*}
Taking logs and dividing by $\lambda$ gives,
\begin{align*}
\E{ \max_{\gamma \in \Gamma} G( S^\gamma_T ) } &\le \frac{\log(2|\Gamma|) + T \log (1 + \sigma (\exp(\lambda B) - 1 - \lambda B))}{\lambda} \\
&\le \frac{\log(2|\Gamma|) + T \sigma (\exp(\lambda B) - 1 - \lambda B)}{\lambda} \\
&\le \frac{\log(2|\Gamma|) + T \sigma \lambda^2 B^2}{\lambda} \ ,
\end{align*}
where the last inequality is valid for any $\lambda \le 1/B$. Optimizing over $\lambda$, we choose $\lambda = \sqrt{ \log(2|\Gamma|)/T\sigma B^2 }$ which is less than
$1/B$ under the condition $T \ge \log(2|\Gamma|/\sigma)$. Plugging this in gives,
\[
	\E{ \max_{\gamma \in \Gamma} G( S^\gamma_T ) } \le 2B\, \sqrt{\sigma\, \log(2|\Gamma|) \, T } \ .
\]
\end{proof}

\begin{lemma}
	\label{lem:prob_to_exp}
	If $F$ is a non-negative real-valued random variable and $\Prob(F>\epsilon)\leq 2\exp\left\{-\frac{T\epsilon^2}{2c}\right\}$, then 
	$$\En F \leq \sqrt{2\pi c/T}.$$
	More generally, if $\Prob(F>a+\epsilon)\leq 2N\exp\left\{-\frac{\epsilon^2 b}{2}\right\}$ for $\epsilon > \sqrt{\frac{4\log(2N)}{b}}$, then 
	$$\En F \leq a+\left(\sqrt{\log(2N)}+1\right)\sqrt{\frac{4}{b}}.$$
\end{lemma}
\begin{proof}
	$$\En F = \int_{0}^\infty \Prob(F>\epsilon) d\epsilon \leq 2\int_{0}^\infty \exp\left\{-\frac{T\epsilon^2}{2c}\right\} d\epsilon = 2\sqrt{\frac{2\pi c}{T}}\frac{1}{\sqrt{2\pi}}\int_{0}^\infty \exp\{-u^2/2\}du = \sqrt{\frac{2\pi c}{T}}.$$
	For the second statement,
	 $$\En F = \int_{0}^\infty \Prob(F>a+\epsilon) d\epsilon \leq  a+x + \int_{x}^\infty \Prob(F>a+\epsilon) d\epsilon.$$
	Choose $x=\sqrt{\frac{4\log(2N) }{b}}$. For $\epsilon>x$, it holds that $-\frac{b\epsilon^2}{2}+\log(2N)\leq -\frac{b\epsilon^2}{4}$. Thus,
	 $$\En F \leq  a+\sqrt{\frac{4\log(2N)}{b}} + \int_{0}^\infty \exp\left\{-\frac{b\epsilon^2}{4}\right\} d\epsilon =  \sqrt{\frac{4\log(2N)}{b}} +  \sqrt{\frac{4\pi}{b}} \frac{1}{\sqrt{2\pi}}\int_{0}^\infty \exp\{-u^2/2\}du .$$
\end{proof}

%% file: generaltriplex.tex
\section{A General Triplex Inequality}
\label{sec:gentriplex}

Here we make the observation that the two versions of the triplex inequality, namely the expected (Theorem~\ref{thm:main}) and high probability (Theorem~\ref{thm:main-prob})
versions, are special cases of a general triplex inequality which bounds the value of a ``$\Gamma$-game" defined as:

{\small
\begin{align}
\label{def:gengamevalue}
	\Val^{\Gamma}_T(\ell,\Phi_T) &= \inf_{q_1} \sup_{x_1} \Eu{f_1\sim q_1} \ldots \inf_{q_T} \sup_{x_T} \Eu{f_T\sim q_T}
	\Gamma\left(
	\sup_{\bphi\in \Phi_T}\left\{ \Compare(\ell(f_1,x_1), \ldots, \ell(f_T, x_T)) -  \Compare(\ell_{\phi_1}(  f_1,x_1), \ldots, \ell_{\phi_T}(f_T, x_T))\right\}
	\right)
\end{align}
}

The expectation and high-probability games are recovered by choosing $\Gamma(x) = x$ and $\Gamma(x) = \ind{ x > \theta }$ respectively. We now state and prove the general triplex inequality\footnote{To be precise, the expectation version of the Triplex inequality  presented in Theorem~\ref{thm:main} is slightly different, as the expectation is taken outside of $\Compare$. Modulo this difference, the proofs are identical.}.

\begin{theorem}[\textbf{General Triplex Inequality}]
\label{thm:main-gen}
If $\Gamma$ satisfies
\[
	\Gamma(x+y+z) \le \Lambda(x) + \Lambda(y) + \Lambda(z)
\]
for some $\Lambda:\reals\to\reals$, then we have,
\begin{align*}
	\Val^\Gamma_T(\ell,\Phi_T)
	& \le \sup_{\mbf{D}} \Es{\mbf{D}}{ \Lambda\left(  \Compare(\ell(f_1,x_1), \ldots, \ell(f_T, x_T)) - \Compare(\ell(q_1,p_1), \ldots, \ell(q_T, p_T))  \right) }\\
	&+\sup_{p_1}\inf_{q_1} \ldots  \sup_{p_T}\inf_{q_T} 
		  \Lambda\left( \sup_{\bphi\in \Phi_T}   \left\{ \Compare(\ell(q_1,p_1), \ldots, \ell(q_T, p_T)) - \Compare(\ell_{\phi_1}(q_1,p_1), \ldots, \ell_{\phi_T}(q_T, p_T))\right\} \right)\\
	&+\sup_{\mbf{D}}\Es{\mbf{D}}{ \Lambda\left( \sup_{\bphi\in \Phi_T} \left\{\Compare(\ell_{\phi_1}(q_1,p_1), \ldots, \ell_{\phi_T}(q_T, p_T) ) - \Compare(\ell_{\phi_1}(f_1,x_1),\ldots, \ell_{\phi_T}(f_T,x_T))\right\}  \right) }
\end{align*}
where $\mbf{D}$ ranges over distributions over sequences $(x_1,f_1),\ldots, (x_T,f_T)$. 
\end{theorem}
\begin{proof}
	The value of the game $\Val^\Gamma_T(\ell,\Phi_T)$, defined in \eqref{def:gengamevalue}, is
	{\small
	\begin{align*}
		&\Val^\Gamma_T(\ell,\Phi_T) \\
		&= \inf_{q_1} \sup_{p_1} \Eunder{x_1\sim p_1}{f_1\sim q_1} \ldots \inf_{q_T} \sup_{p_T} \Eunder{x_T\sim p_T}{f_T\sim q_T}  \left[
		\Gamma\left( \sup_{\bphi\in \Phi_T}\left\{ \Compare(\ell(f_1,x_1), \ldots, \ell(f_T, x_T)) -  \Compare(\ell_{\phi_1}(  f_1,x_1), \ldots, \ell_{\phi_T}(f_T, x_T))\right\} \right) \right]\\
		&= \sup_{p_1}\inf_{q_1} \Eunder{x_1\sim p_1}{f_1\sim q_1} \ldots  \sup_{p_T}\inf_{q_T} \Eunder{x_T\sim p_T}{f_T\sim q_T}  \left[
		\Gamma\left( \sup_{\bphi\in \Phi_T} \left\{ \Compare(\ell(f_1,x_1), \ldots, \ell(f_T, x_T)) - \Compare(\ell_{\phi_1 }( f_1,x_1), \ldots, \ell_{\phi_T}(f_T, x_T))\right\} \right) \right]
	\end{align*} 
	}
	via an application of the minimax theorem. Adding and subtracting terms to the expression above leads to
	{\small
	\begin{align*}
		\Val^\Gamma_T(\ell, \Phi_T) &= \sup_{p_1}\inf_{q_1}\Eunder{x_1\sim p_1}{f_1 \sim q_1} \ldots  \sup_{p_T}\inf_{q_T} \Eunder{x_T\sim p_T}{f_T \sim q_T}
		\hspace{0.4cm} \left[ \Gamma\left( \Compare(\ell(f_1,x_1), \ldots, \ell(f_T, x_T)) -  \Compare(\ell(q_1,p_1), \ldots, \ell(q_T, p_T)) \right. \right.\\
		&\left. \left.\hspace{0.4in}+\sup_{\bphi\in \Phi_T} \left\{  \Compare(\ell(q_1,p_1), \ldots, \ell(q_T, p_T)) - \Compare(\ell_{\phi_1}(f_1,x_1), \ldots, \ell_{\phi_T}(f_T,x_T))\right\} \right) \right]\\
		&\leq \sup_{p_1}\inf_{q_1}\Eunder{x_1\sim p_1}{f_1 \sim q_1} \ldots  \sup_{p_T}\inf_{q_T} \Eunder{x_T\sim p_T}{f_T \sim q_T}
		\hspace{0.4cm} \left[ \Gamma\left(\Compare(\ell(f_1,x_1), \ldots, \ell(f_T, x_T)) -  \Compare(\ell(q_1,p_1), \ldots, \ell(q_T, p_T)) \right. \right.\\
		&\left. \left.\hspace{0.4in}+\sup_{\bphi\in \Phi_T}  \Big\{ \Compare(\ell(q_1,p_1), \ldots, \ell(q_T, p_T)) -   \Compare(\ell_{\phi_1}(q_1,p_1), \ldots, \ell_{\phi_T}(q_T, p_T))\Big\} \right. \right.\\
		&\left. \left.\hspace{0.4in} + \sup_{\bphi\in \Phi_T} \left\{  \Compare(\ell_{\phi_1}(q_1,p_1), \ldots, \ell_{\phi_T}(q_T, p_T)) -  \Compare(\ell_{\phi_1}(f_1,x_1), \ldots, \ell_{\phi_T}(f_T,x_T))\right\} \right) \right]\\
		&\leq \sup_{p_1}\inf_{q_1}\Eunder{x_1\sim p_1}{f_1 \sim q_1} \ldots  \sup_{p_T}\inf_{q_T} \Eunder{x_T\sim p_T}{f_T \sim q_T}
		\hspace{0.4cm} \left[ \Lambda\left( \Compare(\ell(f_1,x_1), \ldots, \ell(f_T, x_T)) - \Compare(\ell(q_1,p_1), \ldots, \ell(q_T, p_T)) \right) \right.\\
		&\left. \hspace{0.4in}+ \Lambda\left( \sup_{\bphi\in \Phi_T}  \Big\{ \Compare(\ell(q_1,p_1), \ldots, \ell(q_T, p_T)) -   \Compare(\ell_{\phi_1}(q_1,p_1), \ldots, \ell_{\phi_T}(q_T, p_T))\Big\} \right) \right.\\
		&\left. \hspace{0.4in} + \Lambda\left( \sup_{\bphi\in \Phi_T}  \left\{  \Compare(\ell_{\phi_1}(q_1,p_1), \ldots, \ell_{\phi_T}(q_T, p_T)) -  \Compare(\ell_{\phi_1}(f_1,x_1), \ldots, \ell_{\phi_T}(f_T,x_T))\right\} \right) \right]
		\end{align*} }
	At this point, we would like to break up the expression into three terms. To do so, notice that expectation is linear and $\sup$ is a convex function, while for the infimum,
	$$\inf_a \left[C_1(a)+C_2(a)+C_3(a)\right] \leq \left[\sup_a C_1(a)\right] + \left[\inf_a C_2(a)\right] + \left[\sup_a C_3(a)\right] $$
	for functions $C_1,C_2,C_3$. We use these properties of $\inf$, $\sup$, and expectation, starting from the inside of the nested expression and splitting the expression in three parts. We arrive at
	{\small
	\begin{align*}
		&\Val^\Gamma_T(\ell, \Phi_T)\\ 
		&\leq \sup_{p_1}\sup_{q_1}\Eunder{x_1\sim p_1}{f_1 \sim q_1} \ldots  \sup_{p_T}\sup_{q_T} \Eunder{x_T\sim p_T}{f_T \sim q_T}
		 \left[ \Lambda\left( \Compare(\ell(f_1,x_1), \ldots, \ell(f_T, x_T)) -  \Compare(\ell(q_1,p_1), \ldots, \ell(q_T, p_T)) \right) \right]\\
		&+\sup_{p_1}\inf_{q_1}\Eunder{x_1\sim p_1}{f_1 \sim q_1} \ldots  \sup_{p_T}\inf_{q_T} \Eunder{x_T\sim p_T}{f_T \sim q_T}
	 \left[ \Lambda\left(\sup_{\bphi\in \Phi_T}   \left\{ \Compare(\ell(q_1,p_1), \ldots, \ell(q_T, p_T)) - \Compare(\ell_{\phi_1}(q_1,p_1), \ldots, \ell_{\phi_T}(q_T, p_T))\right\} \right) \right]\\
		&+\sup_{p_1}\sup_{q_1}\Eunder{x_1\sim p_1}{f_1 \sim q_1} \ldots  \sup_{p_T}\sup_{q_T} \Eunder{x_T\sim p_T}{f_T \sim q_T} \left[ \Lambda\left(\sup_{\bphi\in \Phi_T} \left\{   \Compare(\ell_{\phi_1}(q_1,p_1), \ldots, \ell_{\phi_T}(q_T, p_T)) -  \Compare(\ell_{\phi_1}(f_1,x_1), \ldots, \ell_{\phi_T}(f_T,x_T))\right\} \right) \right]
	\end{align*}
	}
	As mentioned in the corresponding proof of Theorem~\ref{thm:main}, the replacement of infima by suprema in the first and third terms appears to be a loose step and, indeed, one can pick a particular response strategy $\{q^*_t\}$ instead of passing to the supremum. 

	Consider the second term in the above decomposition. Clearly,
	\begin{align*}
		&\sup_{p_1}\inf_{q_1}\Eunder{x_1\sim p_1}{f_1 \sim q_1} \ldots  \sup_{p_T}\inf_{q_T} \Eunder{x_T\sim p_T}{f_T \sim q_T}
		 \left[ \Lambda\left(\sup_{\bphi\in \Phi_T}   \Compare(\ell(q_1,p_1), \ldots, \ell(q_T, p_T)) - \Compare(\ell_{\phi_1}(q_1,p_1), \ldots, \ell_{\phi_T}(q_T, p_T))\right) \right]\\
		&=\sup_{p_1}\inf_{q_1} \ldots  \sup_{p_T}\inf_{q_T}\ 
		 \Lambda\left( \sup_{\bphi\in \Phi_T}    \Compare(\ell(q_1,p_1), \ldots, \ell(q_T, p_T)) -  \Compare(\ell_{\phi_1}(q_1,p_1), \ldots, \ell_{\phi_T}(q_T, p_T)) \right)
	\end{align*}
	because the objective does not depend on the random draws.
\end{proof}